\documentclass{article}

\PassOptionsToPackage{numbers, compress}{natbib}
\usepackage[final]{neurips_2022}

\usepackage[utf8]{inputenc} % allow utf-8 input
\usepackage[T1]{fontenc}    % use 8-bit T1 fonts
\usepackage{hyperref}       % hyperlinks
\usepackage{url}            % simple URL typesetting
\usepackage{booktabs}       % professional-quality tables
\usepackage{amsfonts}       % blackboard math symbols
\usepackage{nicefrac}       % compact symbols for 1/2, etc.
\usepackage{microtype}      % microtypography
\usepackage{xcolor}         % colors

\usepackage{algorithmic}
\usepackage{algorithm}

\renewcommand{\algorithmiccomment}[1]{{\tiny\hfill$\triangleright$ #1}}

\usepackage{subfigmat}
\title{Escaping Saddle Points with Bias-Variance Reduced Local Perturbed SGD for Communication Efficient Nonconvex Distributed Learning}

\author{%
  Tomoya Murata\thanks{murata@msi.co.jp} \\
  NTT DATA Mathematical Systems Inc., Tokyo, Japan \\
  Graduate School of Information Science and Technology, The University of Tokyo \\
  \And
  Taiji Suzuki\thanks{taiji@mist.i.u-tokyo.ac.jp} \\
  Graduate School of Information Science and Technology, The University of Tokyo \\ 
  Center for Advanced Intelligence Project, RIKEN, Tokyo, Japan \\
}

\usepackage{amsmath}
\usepackage{amssymb}
\usepackage{mathtools}
\usepackage{amsthm}
\usepackage{dsfont}
\usepackage{bm}
\usepackage{comment}
\usepackage[symbol]{footmisc}

\usepackage[capitalize,noabbrev]{cleveref}

\theoremstyle{theorem}
\newtheorem{theorem}{Theorem}[section]
\newtheorem{lemma}[theorem]{Lemma}
\newtheorem{proposition}[theorem]{Proposition}
\newtheorem{corollary}[theorem]{Corollary}

\theoremstyle{definition}
\newtheorem{definition}{Definition}[section]
\newtheorem{assumption}{Assumption}
\theoremstyle{remark}
\newtheorem*{remark}{Remark}

\begin{document}
\maketitle

\begin{abstract}
%Distributed learning is one of the promising approaches to tackle modern large scale optimization problems in machine learning.  Although minibatch Stochastic Gradient Descent (SGD) is widely used in machine learning, minibatch SGD typically suffers from its communication cost in distributed learning.
In recent centralized nonconvex distributed learning and federated learning, local methods are one of the promising approaches to reduce communication time. However, existing work has mainly focused on studying first-order optimality guarantees. 
On the other side, second-order optimality guaranteed algorithms, i.e., algorithms escaping saddle points, have been extensively studied in the non-distributed optimization literature. 
In this paper, we study a new local algorithm called Bias-Variance Reduced Local Perturbed SGD (BVR-L-PSGD), that combines the existing bias-variance reduced gradient estimator with parameter perturbation to find second-order optimal points in centralized nonconvex distributed optimization. 
BVR-L-PSGD enjoys second-order optimality with nearly the same communication complexity as the best known one of BVR-L-SGD to find first-order optimality. Particularly, the communication complexity is better than non-local methods when the local datasets heterogeneity is smaller than the smoothness of the local loss. In an extreme case, the communication complexity approaches to $\widetilde \Theta(1)$ when the local datasets heterogeneity goes to zero. Numerical results validate our theoretical findings. 
\end{abstract}

\section{Introduction}\label{sec: intro}
%In this paper, we consider centralized nonconvex distributed learning. In centralized distributed learning, each worker parallelly executes local computations on his own local dataset and the server periodically aggregates the workers outputs through communications. The simplest instance of this scheme is parallel minibatch Stochastic Gradient Descent (SGD). In a single round of minibatch SGD, each worker first receives the current global model from the server and locally computes a stochastic gradient at the global model on his own local dataset and send it to the server. The server aggregates the received local gradients and updates the global model based on them. \par 
Distributed learning is an attractive approach to reduce the total execution time by utilizing the parallel computations. However, the communication time in distributed learning can be a main bottleneck in the entire process due to huge parameter size typical in deep learning or low bandwidth communication environments. \par

To reduce communication time, one of the promising approaches is the usage of local methods such as local SGD (also called as Parallel Restart SGD or FedAvg). In local SGD, each worker independently executes multiple updates of the local model based on his own local dataset, and the server periodically communicates and aggregates the local models. Many paper have studied local SGD \cite{stich2018local, yu2019parallel, haddadpour2019convergence, haddadpour2019local, koloskova2020unified, khaled2020tighter, woodworth2020minibatch, woodworth2020local}. Particularly, for convex objectives, it has been shown in \cite{woodworth2020local}, for the first time, the communication complexity (that is the necessary number of communication rounds to achieve given desired optimization error) of local SGD can be smaller than the one of minibatch SGD when the \emph{heterogeneity} of the local datasets is extremely small. In traditional distributed learning, the local datasets are typically random subsets of the global dataset and in this case the heterogeneity of the local datasets may become quite small. However, in the recent federated learning regimes \cite{konevcny2015federated, shokri2015privacy, mcmahan2017communication}, it is often the case that the heterogeneity of the local datasets is not too small. Also, the analysis in \cite{woodworth2020local} has only focused on convex cases. Hence, the superiority of local SGD to minibatch SGD is still quite limited. \par

Recently, more communication efficient local methods than local SGD have been proposed for possibly nonconvex objectives to guarantee first-order optimality  \cite{karimireddy2020scaffold, reddi2016aide, murata2021bias}. SCAFFOLD \cite{karimireddy2020scaffold} is a new local algorithm based on the idea of reducing their called client-drift by using a similar formulation to the variance reduction technique \cite{johnson2013accelerating}. They have shown that the communication complexity of SCAFFOLD can be smaller than the one of minibatch SGD for not too heterogenous local datasets under the \emph{quadraticity} of the (possibly nonconvex) local objectives, which is quite limited. For general nonconvex objectives, the communication complexity of SCAFFOLD is same as minibatch SGD. More recently, Murate and Suzuki \cite{murata2021bias} have proposed Bias-Variance Reduced Local SGD (BVR-L-SGD). BVR-L-SGD utilizes their proposed {\emph{bias-variance reduced estimator}} that simultaneously reduces the bias caused by local gradient steps and the variance caused by stochastization of the gradients in local optimization based on the formulation of SARAH like variance reduction \cite{nguyen2017sarah}. They have shown that the communication complexity of BVR-L-SGD is smaller than minibatch SGD for not too heterogeneous local datasets {\emph{for general nonconvex objectives}}. Specifically, BVR-L-SGD is superior to minibatch SGD when the Hessian heterogeneity of the local datasets is small relative to the smoothness of the local loss in order sense.  \par 

On the other side, there are vast work that has studied second-order optimality guarantees, which is much more challenging to ensure but desirable than first-order one, in non-distributed nonconvex optimization. Several approaches are known and one of the simplest approaches is parameter perturbation \cite{ge2015escaping, jin2017escape, jin2021nonconvex}. However, almost all existing analysis of local methods have only focused on achieving first-order optimality. As an exception, Vlaski et al. \cite{vlaski2020second} have analysed second-order guarantees of local SGD with parameter perturbation, but the obtained communication complexity is much worse than the one of minibatch SGD and no benefit of localization has been shown. \par

{\bf{Open question.}} For local methods, it is not well-studied how to find second-order optimal points with low communication cost and thus we have the following research questions: \\
{\emph{
Is there a first-order distributed optimization algorithm with second-order optimality guarantees which satisfies that (i) the communication complexity is smaller than non-local methods for not too heterogeneous local datasets; and (ii) the communication complexity approaches to $\Theta(1)$ when the heterogeneity of local datasets goes to zero?}}

Note that the both properties are desirable in distributed optimization. We expect that local methods are superior to non-local methods for not highly heterogeneous local datasets. Furthermore,  when the local datasets are nearly identical, it is expected that a few communications are sufficient to optimize the global objective. {\emph{Only in the case of first-order optimality}}, Murata and Suzuki \cite{murata2021bias} have shown that their proposed BVR-L-SGD satisfies (i) and (ii), and the question has been positively answered. However, the question is still open in the case of second-order optimality\footnotemark. 
\footnotetext{Since there are communication efficient distributed optimization algorithms that find first-order stationary points like BVR-L-SGD, we can apply generic algorithms which guarantee second-order optimality to them \cite{xu2017first, allen2017neon2}. However, this naive approach does not possess the aforementioned property (ii) because the generic framework requires at least $\Omega(1/\varepsilon^{3/2})$ communication rounds to guarantee second-order optimality due to the multiple negative curvature exploitation steps in their framework for any communication efficient algorithms with first-order optimality guarantees. Also, this approach requires explicit negative curvature exploitation, that is complicated and makes the whole algorithm less practical.}
%{\bf{Technical challenges. }}Since there is a communication efficient distributed optimization algorithm that finds first-order stationary points like BVR-L-SGD, we can apply generic algorithms which guarantee second-order optimality to them \cite{xu2017first, allen2017neon2}. However, this naive approach does not possess the aforementioned property (ii) because the generic framework requires at least $\Omega(1/\varepsilon^{3/2})$ communication rounds to guarantee second-order optimality due to the multiple negative curvature exploitation  steps in their framework for any communication efficient algorithms with first-order optimality guarantees. Also, this approach requires explicit negative curvature exploitation, that complicates and makes the whole algorithm less practical. For these reasons, we reject this naive attempt. Another natural approach is the usage of parameter perturbation. In this approach, there are some difficulties unique in distributed optimization. Specifically, it is unclear whether simple local perturbation \cite{jin2017escape, jin2021nonconvex} is sufficient to escape the saddle points of the global objective, although simple noise addition surely helps each local model escape the saddle points of the local objective. %Most local methods rely on the average aggregation of the local models when the workers are synchronized and thus it is possible that the averaged model still gets stuck around the saddle points of the global objective even if each local model escapes the saddle points of his own local objective. 

\subsection*{Main Contributions}
We propose a new local algorithm called Bias-Variance Reduced Local \emph{Perturbed} SGD (BVR-L-PSGD) for nonconvex distributed learning to efficiently find second-order optimal points, which positively answered the above research questions.
%The main features of our algorithm and theoretical results are summarized as below. \par
\par
%{\bf{(Algorithm)}} 
The algorithm is based on a simple combination of the existing bias and variance reduced gradient estimator and parameter perturbation. In our algorithm, parameter perturbation is carried out \emph{at every local update} and it is not necessary to determine whether or not to add noise by checking the norm of the global gradients, which is often required in several previous non-distributed algorithms 
%with second-order optimality guarantees
\cite{ge2019stabilized, li2019ssrgd}. 
%Also, parameter averaging is avoided by using \emph{a randomly picked local model} as a synchronized global model. This approach is justified thanks to the asymptotic consistency of the bias-variance reduced estimator to the global gradient.
\par
%{\bf{(Theoretical analysis)}}
We analyse BVR-L-PSGD for general nonconvex smooth objectives. 
The most challenging part of our analysis is to ensure that our algorithm efficiently \emph{escapes global saddle points} even \emph{in local optimization}. To realize this, it is necessary to analyse the behavior of the bias-variance reduced estimator around the saddle points by carefully evaluating the degree of some kind of asymptotic consistency of the estimator around the saddle points. This point has never been pursued in previous work and has a unique difficulty of our analysis.  

%{\bf{Theoretical results}} 
The comparison of the communication complexities of our method with the most relevant existing results is given in Table \ref{tab: theoretical_comparison}. Our proposed method enjoys second-order optimality with nearly the same communication complexity as the one of BVR-L-SGD, which achieves the best known communication complexity to achieve first-order optimality. This means that our method finds second-order optimal points without hurting the communication efficiency of the state-of-the-art first-order optimality guaranteed method. Particularly, the communication complexity is better than minibatch SGD when Hessian heterogeneity $\zeta$ is small relative to smoothness $L$. Also, the communication complexity approaches to $\widetilde O(1)$ when heterogeneity $\zeta$ goes to zero and the local computation budget $\mathcal B$ (see Section \ref{sec: problem_setting}) goes to infinity. 
%In other words, the necessary number of communication rounds to achieve second-order optimality is nearly $\Theta(1)$ in this setting.
Hence, our method enjoys the aforementioned two desired properties. \par 

\begin{table*}[t]
    \centering
    \scalebox{0.95}{
    \begin{tabular}{c c c c}\hline
        Algorithm & Communication Rounds & Assumptions & Guarantee\\ \hline \hline
        %GD&$\frac{1}{\varepsilon}$ & None \\\hline
        \begin{tabular}{c}
        Minibatch SGD
        \end{tabular} & $\frac{1}{\varepsilon^2} + \frac{1}{\mathcal B P\varepsilon^4}$ & 2-3, BSGV & 1st-order\\ \hline 
        \begin{tabular}{c}
        Noisy Minibatch SGD \cite{jin2021nonconvex}
        \end{tabular}& $\frac{1}{\varepsilon^2} + \frac{1}{\mathcal B P\varepsilon^4}$ & 2-4 & {\bf{2nd-order}}\\ \hline
        \begin{tabular}{c}
        Minibatch SARAH \cite{nguyen2019finite}
        \end{tabular} & $\frac{1}{\varepsilon^2} + \frac{\sqrt{n}}{\mathcal B P\varepsilon^2} + \frac{n}{\mathcal B P}$ & 2-3 & 1st-order\\ \hline
        SSRGD \cite{li2019ssrgd} & $\frac{1}{\varepsilon^2} + \frac{\sqrt{n}}{\varepsilon^\frac{3}{2}}$ & 2-5, $\mathcal B \geq \frac{\sqrt{n}}{P}$& \bf{2nd-order} \\ \hline
        \begin{tabular}{c}Local SGD \cite{yu2019parallel}\end{tabular}& $\frac{1}{\mathcal B\varepsilon^2} + \frac{1}{\mathcal B P\varepsilon^4}+ \frac{1}{\varepsilon^3}$ & 2-3, 5 & 1st-order\\ \hline
        \begin{tabular}{c}SCAFFOLD \cite{karimireddy2020scaffold}\end{tabular}& $\frac{1}{\varepsilon^2} + \frac{1}{\mathcal B P\varepsilon^4}$ & 2-3, BSGV & 1st-order\\ \hline
        \begin{tabular}{c}SCAFFOLD \cite{karimireddy2020scaffold}\end{tabular}& $\frac{1}{\mathcal B \varepsilon^2} + \frac{1}{\mathcal B P\varepsilon^4} + \frac{\zeta}{\varepsilon^2}$ &  \begin{tabular}{c}1-3, BSGV, \\quadraticity \end{tabular} & 1st-order \\ \hline
        \begin{tabular}{c}BVR-L-SGD \cite{murata2021bias}\end{tabular} & $\frac{1}{\sqrt{\mathcal B}\varepsilon^2} + \frac{\sqrt{n}}{\mathcal B P\varepsilon^2} +  \frac{\zeta}{\varepsilon^2} $ & \begin{tabular}{c}1-4 \end{tabular} & 1st-order \\ \hline
        \begin{tabular}{c} {\color{red}{\bf{BVR-L-PSGD}}} {\color{red}{\bf{(this paper)}}}  \end{tabular} & $\frac{1}{\sqrt{\mathcal B}\varepsilon^2} + \frac{\sqrt{n}}{\mathcal B P\varepsilon^2} +  \frac{\zeta}{\varepsilon^2}$ & 1-5  & {\bf{2nd-order}} \\ \hline
\end{tabular}
    }
    \caption{Comparison of the order of the necessary number of communication rounds to achieve desired optimization error $\varepsilon$ in terms of given optimization criteria (described in the column of "Guarantee'') in nonconvex optimization. "Assumptions'' indicates the necessary assumptions to derive the results (the numbers correspond to Assumptions \ref{assump: local_loss_grad_lipschitzness}, \ref{assump: optimal_sol}, \ref{assump: local_loss_hessian_lipschitzness}, \ref{assump: bounded_loss_gradient} in Section \ref{sec: problem_setting} respectively). BSGV means the bounded stochastic gradient variance assumption, that is $\mathbb{E}_{z \sim D_p}\|\nabla \ell(x, z) - \nabla f_p(x)\|^2 \leq \sigma^2$. $\mathcal B$ is the local computation budget 
    %that is the allowed number of single stochastic gradient computations per communication round for each worker
    , which is defined in Section \ref{sec: problem_setting}. $P$ is the number of workers. $n$ is the total number of samples. The gradient Lipschitzness $L$, Hessian Lipschitzness $\rho$, the gradient boundedness $G$ are regarded as $\Theta(1)$ for ease of presentation. In this notation, \emph{Hessian heterogeneity $\zeta$ always satisfies $\zeta \leq \Theta(L) = \Theta(1)$}.}% We hide order symbols for simplicity. }
    \label{tab: theoretical_comparison}
\end{table*}

\subsection*{Related Work}
Here, we briefly review the related studies to our paper. \par
%{\bf{Local methods.}} Inexact DANE \cite{reddi2016aide} is another communication efficient generic framework where each worker uses a general convex optimization algorithm as a subsolver. The superiority to non-local methods has been only shown for quadratic objectives.
{\bf{Local methods.}} Several recent papers have studied local algorithms combined with variance reduction technique \cite{sharma2019parallel, das2020faster, khanduriachieving,  karimireddy2020mime}. Sharma et al. \cite{sharma2019parallel} have proposed a local variant of SPIDER \cite{fang2018spider} and shown that the proposed algorithm achieves the optimal total computational complexity. However, the communication complexity essentially matches the ones of non-local SARAH and no advantage of localization has been shown. Khanduri et al. \cite{khanduriachieving} have proposed STEM and its variants based on their called two-sided momentum, but again the communication complexity does not improve non-local methods. Also, Das et al. \cite{das2020faster} have considered a SPIDER like local algorithm called FedGLOMO but the derived communication complexity is even worse than minibatch SARAH. Karimireddy et al. \cite{karimireddy2020mime} have proposed Mime, which is a general framework to mitigate client-drift. Particularly, under $\delta$-Bounded Hessian Dissimilarity (BHD)\footnotemark, their MimeMVR achieves communication complexity of $1/(P\varepsilon^2) + \delta/(\sqrt{P}\varepsilon^3) + \delta/\varepsilon^2$ when $\mathcal B \to \infty$, that is better than the one of minibatch SGD $1/\varepsilon^2$ when $\delta \leq \sqrt{P}\varepsilon$. However, the asymptotic rate is still worse than the one of BVR-L-SGD $\zeta/\varepsilon^2$ because $\zeta \leq \delta$ always holds. \par
{\bf{Second-order guarantee.}} Neon \cite{xu2017first} and Neon2 \cite{allen2017neon2} are generic first-order methods with second-order guarantees, that  repeatedly run a first-order guaranteed algorithm and negative curvature descent. Another approach is a parameter perturbation for SGD. For the first time, Ge et al. \cite{ge2015escaping} have shown that SGD with a simple parameter perturbation escapes saddle points efficiently. Later, the analysis has been refined by \cite{jin2017escape, jin2021nonconvex}. Recently, applying variance reduction technique to second-order guaranteed methods has been also studied \cite{ge2019stabilized, li2019ssrgd} and particularly Li et al. \cite{li2019ssrgd} have proposed SSRGD that combines SARAH \cite{nguyen2017sarah} with parameter perturbation and shown that SSRGD nearly achieves the optimal computational complexity with second-order optimality guarantees. 
\footnotetext{$\delta$-BHD condition in \cite{karimireddy2020mime} requires $\|\nabla^2 \ell(x, z) - \nabla^2 f(x)\| \leq \delta$ for every $x \in \mathbb{R}^d$, $z \sim D_p$ and $p \in [P]$. Note that $\delta$-BHD condition requires both intra Hessian dissimilarity boundedness $\|\nabla^2 f_i(x) - \nabla^2 f(x)\|$, which is bounded by  $\zeta$ under Assumption \ref{assump: heterogeneous}, and additionally inner Hessian dissimilarity $\|\nabla^2 \ell(x, z) - \nabla^2 f_i(x)\|$. Hence, $\delta$-BHD is much stronger than Assumption \ref{assump: heterogeneous} and it is possible that $\delta \gg \zeta$. }

\section{Problem Definition and Assumptions}\label{sec: problem_setting}
In this section, we first introduce several notations and definitions used in this paper. Then, the problem settings are described and theoretical assumptions used in our analysis are given. \par

{\bf{Notation.}}
$\| \cdot \|$ denotes the Euclidean $L_2$ norm $\| \cdot \|_2$: $\|x\| = \sqrt{\sum_{i}x_i^2}$ for vector $x$. 
For a matrix $X$, $\|X\|$ denotes the induced norm by the Euclidean $L_2$ norm. For a natural number $m$, $[m]$ means the set $\{1, 2, \ldots, m\}$.
For a set $A$, $\# A$ means the number of elements, which is possibly $\infty$. For any number $a, b$, $a \vee b$ and $a \wedge b$ denote $\mathrm{max}\{a, b\}$ and $\mathrm{min}\{a, b\}$ respectively. We denote the uniform distribution over $A$ by $\mathrm{Unif}(A)$. Given $K, T, S \in \mathbb{N}$, let $I(k, t, s)$ be integer $k + Kt + KTs$ for $k \in [K]\cup\{0\}$, $t \in [T-1]\cup\{0\}$ and $s \in [S-1]\cup\{0\}$. Note that $I(K, t, s) = I(0, t+1, s)$ and $I(k, T, s) = I(k, 0, s+1)$ for $k \in [K]\cup\{0\}$, $t \in [T-1]\cup\{0\}$ and $s \in [S-1]\cup\{0\}$. $\bm{B}_r^d$ denotes the set $\{x \in \mathbb{R}^d| \|x\| \leq r\}$, which is the Euclidean ball  in $\mathbb{R}^d$ with radius $r$. 

\begin{definition}[Gradient Lipschitzness]
A differentiable function $f: \mathbb{R}^d \to \mathbb{R}$ is $L$-gradient Lipschitz if $\left\|\nabla f(x) - \nabla f(y)\right\| \leq L\|x - y\|, \forall x, y \in \mathbb{R}^d$.
%\begin{align*}
%    \left\|\nabla f(x) - \nabla f(y)\right\| \leq L\|x - y\|, \forall x, y \in \mathbb{R}^d.
%\end{align*}
\end{definition}
\begin{definition}[Hessian Lipschitzness]
A twice differentiable function $f: \mathbb{R}^d \to \mathbb{R}$ is $\rho$-Hessian Lipschitz if $\left\|\nabla^2 f(x) - \nabla^2 f(y)\right\| \leq \rho\|x - y\|, \forall x, y \in \mathbb{R}^d$.
%\begin{align*}
%    \left\|\nabla^2 f(x) - \nabla^2 f(y)\right\| \leq \rho\|x - y\|, \forall x, y \in \mathbb{R}^d.
%\end{align*}
\end{definition}

\begin{definition}[Second-order optimality]
For a $\rho$-Hessian Lipschitz function $f$, $x \in \mathbb{R}^d$ is an $\varepsilon$-second-order optimal point of $f$ if $\|\nabla f(x)\|\leq \varepsilon \text{ and } \nabla^2 f(x) \succeq - \sqrt{\rho \varepsilon}  I$.
%\begin{align*}
%    \|\nabla f(x)\|\leq \varepsilon \text{ and } \nabla^2 f(x) \succeq - \sqrt{\rho \varepsilon}  I. 
%\end{align*}
\end{definition}

\subsection{Problem Settings}
{\bf{Objective function.}} We want to minimize nonconvex smooth objective 
$f(x) := \frac{1}{P}\sum_{p=1}^P f_p(x)$, where $ f_p(x) := \mathbb{E}_{z \sim  D_p}[\ell(x, z)]$
for $x \in \mathbb{R}^d$, where $D_p$ is the data distribution associated with worker $p$. In this paper, we focus on offline settings (i.e., $\# \mathrm{supp}(D_p) < \infty$ for every $p \in [P]$) for simple presentation. It is easy to extend our results to online settings. Also, just for simplicity, it is assumed that each local dataset has an equal number of samples, i.e., $\# \mathrm{supp}(D_p) = n/P$ for every $p, p' \in [P]$, where $n$ is the total number of samples. \\
{\bf{Optimization criteria.}} Since objective function $f$ is nonconvex, it is generally difficult to find a global minima of $f$. Previous work in distributed learning has mainly focused on finding first-order stationary points of $f$. In this study, we aim to find $\varepsilon$-second-order stationary points of $f$ in distributed learning settings. \\
{\bf{Data access constraints and communication settings.}} It is assumed that each worker $p$ can only access the own data distribution $D_p$ without communication. Aggregation (e.g., summation) of all the worker's $d$-dimensional parameters or broadcast of a $d$-dimensional parameter from one worker to the other workers can be realized by single communication. \footnotemark \\
{\bf{Evaluation criteria: communication complexity.}} In this paper, we compare \emph{communication complexities} of optimization algorithms to satisfy the aforementioned optimization criteria. In typical situations, single communication is more time-consuming than single stochastic gradient computation. Let $\mathcal C$ be the single communication cost and $\mathcal G$ be the single stochastic gradient computation cost. Using these notations, $\mathcal C \geq \mathcal G$ is assumed. We expect that increasing the number of available stochastic gradients in a single communication round leads to faster convergence. Hence, it is natural to increase the number of stochastic gradient computations in a single communication round unless the total stochastic gradient computation time exceeds $\mathcal C$ to reduce the total running time. This motivates the concept of {\bf{\emph{local computation budget $\mathcal B$}}} ($\leq \mathcal C / \mathcal G$): given a communication and computational environment, it is assumed that \emph{each worker can only computes at most $\mathcal B$ single stochastic gradients per communication round on average}. Then, we compare the communication complexity, that is \emph{the total number of communication rounds of a distributed optimization algorithm to achieve the desired optimization accuracy}. From the definition, we can see that the communication complexity on a fixed local computation budget $\mathcal B := \mathcal C/\mathcal G$ captures the best achievable total running time of an algorithm.
\footnotetext{In this work, it is assumed that all the workers can participate in a single communication. It is not so hard to extend our algorithm and analysis to  worker sampling settings, which is more realistic in cross-device federated learning. }
\subsection{Theoretical Assumptions}In this paper, we assume the following five assumptions. 
The first one has already been adopted in several previous work \cite{karimireddy2020scaffold, murata2021bias}. The other ones are standard in the nonconvex optimization literature to guarantees second-order optimality. 
\begin{assumption}[Hessian heterogeneity \cite{karimireddy2020scaffold, murata2021bias}]
\label{assump: heterogeneous}
$\{f_p\}_{p=1}^P$ is second-order  $\zeta$-heterogeneous, i.e., for any $p, p' \in [P]$, $\left\|\nabla^2 f_p(x) - \nabla^2 f_{p'}(x)\right\| \leq \zeta, \forall x \in \mathbb{R}^d$.
%\begin{align*}
%    \left\|\nabla^2 f_p(x) - \nabla^2 f_{p'}(x)\right\| \leq \zeta, \forall x \in \mathbb{R}^d.
%\end{align*}
\end{assumption}
Assumption \ref{assump: heterogeneous} characterizes the heterogeneity of local objectives $\{f_p\}_{p=1}^P$ in terms of Hessians and has a important role in our analysis. Intuitively, we expect that relatively small heterogeneity parameter $\zeta$ to the smoothness parameter $L$ (defined in Assumption \ref{assump: local_loss_grad_lipschitzness}) reduces the necessary number of communication rounds to optimize the global objective. Especially when the local objectives are identical, i.e., $D_p = D_{p'}$ for every $p, p' \in [P]$, $\zeta$ becomes zero. 
When each $D_p$ is the empirical distribution of $n/P$ IID samples from common data distribution $D$, we have $\|\nabla^2 f_p(x) - \nabla^2 f_{p'}(x)\| \leq \widetilde \Theta(\sqrt{P/n}L)$ with high probability by matrix Hoeffding's inequality under Assumption \ref{assump: local_loss_grad_lipschitzness} for fixed $x$.
%\footnotemark.
%\footnotetext{Although to show the high probability bound for every $x \in \mathbb{R}^d$ is generally difficult, we can use the high probability bounds on the discrete optimization path rather than the entire space $\mathbb{R}^d$ and then the same bound still holds. For only simplicity, we assume the heterogeneity condition on entire space $\mathbb{R}^d$ in this paper. }
Hence, in traditional distributed learning regimes, Assumption \ref{assump: heterogeneous} naturally holds. An important remark is that {\emph{Assumption \ref{assump: local_loss_grad_lipschitzness} implies $\zeta \leq 2 L$, i.e., the heterogeneity is bounded by the smoothness}}. Even in federated learning regimes, we expect $\zeta \ll 2L$ for some problems practically. 

\begin{assumption}[Gradient Lipschitzness]
\label{assump: local_loss_grad_lipschitzness}
$\forall p \in [P], z \in \mathrm{supp}(D_p)$, $\ell(\cdot, z)$ is $L$-gradient Lipschitz.
\end{assumption}

\begin{assumption}[Existence of global optimum]
\label{assump: optimal_sol}
$f$ has a global minimizer $x_* \in \mathbb{R}^d$.
\end{assumption}

\begin{assumption}[Hessian Lipschitzness]
\label{assump: local_loss_hessian_lipschitzness}
$\forall p \in [P], z \in \mathrm{supp}(D_p)$, $\ell(\cdot, z)$ is $\rho$-Hessian Lipschitz.
\end{assumption}

%Assuming Lipschitzness of the gradients and Hessians of loss $\ell$ rather than $f_p$ are a bit strong, but are typically necessary in the analysis of variance reduced gradient estimators.  

\begin{assumption}[Bounded stochastic gradient]\label{assump: bounded_loss_gradient}
$\forall p \in [P], z \in \mathrm{supp}(D_p)$, $\nabla \ell(\cdot, z)$ is $G$-bounded, i.e., $\|\nabla \ell(x, z)\| \leq G, \forall x \in \mathbb{R}^d$. 
%\begin{align*}
%    \|\nabla \ell(x, z)\| \leq G, \forall x \in \mathbb{R}^d.
%\end{align*}
\end{assumption}
In our analysis, $G$ has no significant impact because $G$ only depends on our theoretical communication complexity in logarithmic order.

\section{Main Ideas and Proposed Algorithm}

\begin{algorithm}[t]
\caption{BVR-L-PSGD($\widetilde x_0$, $\eta$, $b$, $K$, $T$, $S$, $r$)}
\label{alg: bvr_l_psgd}
\begin{algorithmic}[1]
\STATE Add noise $x_0 = \widetilde x_0 + \eta \xi_{-1}$, where $\xi_{-1} \sim \mathrm{Unif}(\bm{B}_r^d)$.
\FOR {$s=0$ to $S-1$}
    \FOR {$p=1$ to $P$ in parallel}
        \STATE $v_{I(0, 0, s)}^{(p)} = \nabla f_p(x_{I(0, 0, s)})$.
    \ENDFOR
    \STATE Communicate $\{ v_{I(0, 0, s)}^{(p)}\}_{p=1}^P$. Set $v_{I(0, 0, s)} = \frac{1}{P}\sum_{p=1}^P v_{I(0, 0, s)}^{(p)}$.
    \FOR {$t=0$ to $T-1$}
        \FOR {$p=1$ to $P$ in parallel}
        \STATE $g_{I(0, t, s)}^{(p)} = \frac{1}{Kb}\sum_{l=1}^{Kb} \nabla \ell(x_{I(0, t, s)}, z_{l, I(0, t, s)})$, 
        \STATE $g_{I(0, t, s)}^{(p), \mathrm{ref}} =  \frac{1}{Kb}\sum_{l=1}^{Kb} \nabla \ell(x_{I(0, t-1, s)}, z_{l, I(0, t, s}))$ ($z_{l, I(0, t, s)} \overset{i.i.d.}{\sim} D_p$).
        \STATE $v_{I(0, t, s)}^{(p)} = \mathds{1}_{t \geq 1} (g_{I(0, t, s)}^{(p)} - g_{I(0, t, s)}^{(p), \mathrm{ref}} + v_{I(0, t-1, s)}^{(p)}) + \mathds{1}_{t = 0}v_{I(0, 0, s)}^{(p)}$.
        \ENDFOR
        \STATE Communicate $\{v_{I(0, t, s)}^{(p)}\}_{p=1}^P$. Set $v_{I(0, t, s)} = \frac{1}{P}\sum_{p=1}^P v_{I(0, t, s)}^{(p)}$.
        \STATE Randomly select $p_{t, s} \sim \mathrm{Unif}[P]$.  \# Only worker $p_{t, s}$ runs local optimization.
        \FOR {$k=0$ to $K-1$}
        \STATE $b_k = \mathds{1}_{k \equiv 0 \ (\mathrm{mod} \lceil \sqrt{K} \rceil)}\lceil \sqrt{K} \rceil b + \mathds{1}_{k \not\equiv 0 \ (\mathrm{mod} \lceil \sqrt{K} \rceil)}b$.
        \STATE $g_{I(k, t, s)} = \frac{1}{b_k}\sum_{l=1}^{b_k} \nabla \ell(x_{I(k, t, s)}, z_{l, I(k, t, s)})$, 
        \STATE $g_{I(k, t, s)}^{\mathrm{ref}} =  \frac{1}{b_k}\sum_{l=1}^{b_k} \nabla \ell(x_{I(k-1, t, s)}), z_{l, I(k, t, s)})$ ($z_{l, I(k, t, s)} \overset{i.i.d.}{\sim} D_{p_{t, s}}$).
        \STATE $v_{I(k, t, s)} = \mathds{1}_{k \geq 1}(g_{I(k, t, s)} - g_{I(k, t, s)}^{\mathrm{ref}}  + v_{I(k-1, t, s)}) + \mathds{1}_{k = 0}v_{I(0, t, s)}$.
        \STATE Update $\widetilde x_{I(k+1, t, s)} = x_{I(k, t, s)} - \eta v_{I(k, t, s)}$.
        \STATE Add noise $x_{I(k+1, t, s)} = \widetilde x_{I(k+1, t, s)}+ \eta \xi_{I(k, t, s)}$, where $\xi_{I(k, t, s)} \sim \mathrm{Unif}(\bm{B}_r^d)$.
\ENDFOR
\STATE Communicate $x_{I(0, t+1, s)}$.
\ENDFOR
\ENDFOR
\end{algorithmic}
\end{algorithm}

Our proposed algorithm is based on a natural combination of (i) \emph{Bias-Variance Reduced (BVR) estimator}; and (ii) \emph{parameter perturbation at each local update}. The first idea has been proposed by \cite{murata2021bias} to find first-order stationary points with small communication complexity. The second one is a well-known approach to find second-order stationary points in non-distributed nonconvex optimization \cite{ge2015escaping, jin2017escape, jin2021nonconvex}. In this section, we illustrate these two ideas and provide its concrete procedures.
\subsection{Review of BVR Estimator \cite{murata2021bias}}
The bias-variance reduced estimator aims to efficiently find first-order stationary points by simultaneously reducing the bias caused by local gradient descent steps and the variance caused by stochastization of the used gradients. 
\par
%Here, we illustrate the bias-variance reduced estimator. To simply convey the core ideas, in this subsection we will describe the estimator using SVRG like technique, although our algorithm actually relies on SARAH like one. \par
First we consider why the standard local SGD is not sufficient to achieve fast convergence and sometimes slower than minibatch SGD. Recall that in local SGD each worker takes the update rules of $x_{k+1}^{(p)}
= x_k^{(p)} - \eta g_k^{(p)}$ for $k \in [\mathcal B/b]$ in each communication round, where $g_k^{(p)}$ is a stochastic gradient with minibatch size $b$ at $x_k^{(p)}$ on local dataset $D_p$ and $\mathcal B$ is given local computation budget. In typical convergence analysis, we need to bound the expected deviation of $g_k^{(p)}$ from ideal global gradient $\nabla f(x_k)$, that is $\mathbb{E}\|g_k^{(p)} - \nabla f(x_k^{(p)})\|^2 = \|\nabla f_p(x_k^{(p)}) - \nabla f(x_k^{(p)})\|^2 + \mathbb{E}\|g_k^{(p)} - \nabla f_p(x_k^{(p)})\|^2$. The former term is called \emph{bias} and the latter one is called \emph{variance}. A typical assumption to bound the first term is \emph{bounded gradient heterogeneity assumption}, that requires $\|\nabla f_p(x) - \nabla f(x)\| \leq \zeta_1$ for every $x \in \mathbb{R}^d$ and $p \in [P]$. Under this assumption, the first term is only bounded by $\zeta_1$, that is a constant. The second term is typically bounded by $\sigma^2/b$ for $g_k^{(p)}$ with minibatch size $b$, when the variance of a single stochastic gradient is bounded by $\sigma^2$. 
%bounded stochastic gradient variance assumption: $\mathbb{E}_{z\sim D_p}\|\nabla \ell(x, z) - \nabla f_p(x)\|^2 \leq \sigma^2$ for every $x \in \mathbb{R}^d$ and $p \in [P]$.
These facts show that the bias is still a constant and does not vanish even if minibatch size $b$ is enhanced and the variance vanishes. This is why local SGD can be worse than minibatch SGD when $\zeta_1$ is not too small. Also, we can see that the variance is still a constant for fixed minibatch size $b$ and this is a common reason why minibatch SGD and local SGD only show slow convergences. These observations give critical motivations of the simultaneous reduction of the bias and variance. 

The bias-variance reduced estimator $v_k^{(p)}$ is defined as  $v_k^{(p)} := (1/b)\sum_{l=1}^b (\nabla \ell(x_k^{(p)}, z_l) - \nabla \ell(x_0, z_l)) + \nabla f(x_0)$ (SVRG version). It is known that the bias caused by localization can be bounded by $\zeta\|x_k^{(p)} - x_0\|$ and the variance caused by stochastization can be bounded by $(L^2/b)\|x_k^{(p)} - x_0\|^2$, where $\zeta$ is the Hessian heterogeneity of $\{f_p\}$ and $L$ is the smoothness of $\ell$. This implies that both the bias and variance of $v_k^{(p)}$ converges to zero as $x_k^{(p)}$ and $x_0$ go to $x_*$. In other words, \emph{bias-variance reduced estimator $v_k^{(p)}$ is asymptotically consistent to the global gradient $\nabla f(x_k^{(p)})$} by using periodically computed global full gradients $\nabla f(x_0)$. We actually adopt SARAH version of BVR estimator as in \cite{murata2021bias} rather than SVRG one due to its theoretical advantages.

\subsection{Parameter Perturbation at Local Updates}
Although the bias-variance reduced estimator is useful to guarantee first-order optimality with small communication complexity in noncovex optimization, the algorithm often gets stuck at saddle points. To tackle this problem, we borrow the ideas of escaping saddle points in non-distributed nonconvex optimization. Particularly, to efficiently find second-order optimal points, we utilize \emph{parameter perturbation}. Parameter perturbation is a familiar approach in non-distributed nonconvex optimization. Specifically, Jin et al. \cite{jin2017escape, jin2021nonconvex} have considered the update rule of $x_{k+1} = x_k - \eta \nabla f(x_k) + \eta \xi_k$, where $\xi_k \sim \mathrm{Unif}(\mathcal B_r^d)$ for some small radius $r$. This algorithm is called Perturbed GD (PGD) or Noisy GD. Similar to this formulation, we add noise at each local update, i.e., $x_{k+1}^{(p)} = \widetilde x_{k+1}^{(p)} + \eta \xi_k^{(p)}$, where $\widetilde x_{k+1}^{(p)} = x_k^{(p)} - \eta v_k^{(p)}$. The intuition behind the noise addition is that random noise has some components along the negative curvature directions of the global objective around the saddle point, and we expect that noise addition helps the parameter proceed to the decreasing directions of $f$ and escape the saddle points. \par

{\bf{Necessity of local perturbation.}} Perturbing the global model at the server side is an intuitive way, but not sufficient for communication efficiency when we want to utilize small heterogeneity of the local datasets (i.e., $\zeta \ll L$). 
The bias-variance reduced estimator with local perturbation enables to escape \emph{multiple} global saddle points \emph{in local optimization} and achieves second-order optimality with communication complexity $\widetilde \Theta(\zeta/\varepsilon^2)$ for sufficiently large $\mathcal B$. In contrast, perturbing the global parameter at the server side only ensures to escape \emph{single} global saddle point at each round and only achieves communication complexity of $\widetilde \Theta(L/\varepsilon^2)$. This is the reason why local perturbation rather than global one is adopted.

%{\bf{Avoiding Parameter Averaging. }} Typical local methods periodically average the local parameters. This aggregation procedure is necessary to guarantee the convergence for heterogeneous local datasets because in local SGD each worker just runs SGD on his own local dataset and the local parameters trivially overfit to the local dataset if periodical averaging is not used. However, when we try to combine the parameter perturbation with local SGD, the parameter averaging is possibly problematic; when each worker runs local perturbed SGD, the averaged parameters may not escape global saddle points even if the local parameters are guaranteed to escape the saddle points of the local objectives. This is a technical challenge unique in distributed optimization when we aim to find second-order optimal points. To avoid this issue, we again utilize the consistency of the bias-variance reduced estimator. For each communication round, we randomly select a single worker $p$ and the selected worker $p$ only runs local optimization with the bias-variance reduced estimators and the parameter perturbations, and then the local output $x_K^{(p)}$ is communicated and set as the next global parameter.  

\subsection{Concrete Procedures}
The full description of our proposed Bias-Variance Reduced Local Perturbed SGD (BVR-L-PSGD) is given in Algorithm \ref{alg: bvr_l_psgd}. When we set the noise size $r = 0$, Algorithm \ref{alg: bvr_l_psgd} essentially matches BVR-L-SGD. Additionally setting $K=1$, Algorithm \ref{alg: bvr_l_psgd} matches SARAH. The algorithm requires $\Theta(ST)$ communication rounds. At each communication round, each worker computes large batch stochastic gradients and the server constructs $v_{I(0, t, s)}$ by aggregating them. $v_{I(0, t, s)}$ is used as an estimator of $\nabla f(x_{I(0, t, s)})$ to reduce computational cost. In line 14-21, we randomly select worker $p_{t, s}$ and only worker $p_{t, s}$ runs local optimization as described above. In the local optimization, we use SARAH like bias variance reduced estimator (line 16-18) rather than SVRG one and add noise (line 20) at each local update.

\section{Convergence Analysis}
\label{sec: analysis}
In this section, we provide convergence theory of BVR-L-PSGD (Algorithm \ref{alg: bvr_l_psgd}). All the omitted proofs are found in the supplementary material. For simple presentations, we use $\widetilde \Theta$ symbol to hide an extra poly-logarithmic factors that depend on $L, \rho, G, K, b, T, S, 1/\varepsilon, 1/q$, where $q$ represents the confidence parameter in high probability bounds.
\subsection{Finding First-Order Stationary Points}
First, we derive Descent Lemma for BVR-L-PSGD and first-order optimality guarantees by using it. 
\begin{proposition}[Descent Lemma]
\label{prop: first_order_optimality}
Let $ S \in \mathbb{N}$ and $I(k, t, s) \geq I(k_0, t_0, s_0) \in [KT S]\cup\{0\}$. Suppose that Assumptions \ref{assump: heterogeneous}, \ref{assump: local_loss_grad_lipschitzness}, \ref{assump: optimal_sol} and \ref{assump: bounded_loss_gradient} hold. Given $q \in (0, 1)$, $r > 0$, if we appropriately choose 
$\eta = \widetilde \Theta(1/L  \wedge 1/(K\zeta) \wedge \sqrt{b/K}/L \wedge \sqrt{Pb}/(\sqrt{KT}L))$, it holds that 
\begin{align}
    f(x_{I(k, t, s)}) \leq&\ f(x_{I(k_0, t_0, s_0)}) - \frac{\eta}{2}\sum_{i=I(k_0, t_0, s_0)}^{I(k-1, t, s)}\|\nabla f(x_i)\|^2   + \eta \Delta_Ir^2 + R_1 \notag
\end{align}
with probability at least $1-3 q $. Here, $\Delta_I := I(k, t, s) - I(k_0, t_0, s_0)$, $R_1 :=     - \frac{1}{4\eta}\sum_{i=I(k_0, t_0, s_0)}^{I(k, t, s)-1}\|x_{i+1} - x_{i}\|^2 +  \frac{c_\eta}{\eta}(\frac{\Delta_I\wedge K }{K} \sum_{i=I(0, t_0, s_0)}^{I(k_0, t_0, s_0)-1} \|x_{i+1} - x_i\|^2 + \frac{\Delta_I\wedge KT}{KT}\sum_{i=I(0, 0, s_0)}^{I( 0, t_0, s_0)-1}\|x_{i+1} - x_{i}\|^2) $ for some universal constant $c_\eta > 0$.  
\end{proposition}
From Proposition \ref{prop: first_order_optimality} with $I(k_0, t_0, s_0) \leftarrow 0$ and $I(k, t, s) \leftarrow KTS$ gives the following corollary. 
\begin{corollary}\label{cor: first_order_optimality}
Suppose that Assumptions \ref{assump: heterogeneous}, \ref{assump: local_loss_grad_lipschitzness}, \ref{assump: optimal_sol} and \ref{assump: bounded_loss_gradient} hold. Under the same setting as in Proposition \ref{prop: f_diff_decrease} and $S \geq \Theta((f(x_0) - f(x_*))/(\eta KT\varepsilon^2))$, with probability at least $1-3q$, there exists $i \in [KTS-1]\cup\{0\}$ such that $\|\nabla f(\widetilde x_i)\|\leq \varepsilon$.
%\begin{align}
%    \frac{1}{KTS}\sum_{i=0}^{KTS - 1}\|\nabla f(\widetilde x_i)\|^2 \leq \frac{4(f(x_0) - f(x_*))}{\eta KTS} + \frac{\varepsilon^2}{2} \notag
%\end{align}
%with probability at least $1-3q$. 
%Particularly, if we set $S \geq 8(f(x_0) - f(x_*))/(\eta KT\varepsilon^2)$,  with probability at least $1-3q$, there exists $i \in [KTS-1]\cup\{0\}$ such that $\|\nabla f(\widetilde x_i)\|\leq \varepsilon$.
\end{corollary}

\begin{remark}[Communication complexity]
The total number of communication rounds $\Theta(TS)$ becomes 
$\widetilde O \left( T +  \left(L/K + \zeta+ L/\sqrt{Kb} + \sqrt{T}L/\sqrt{KPb}\right)(f(\widetilde x_0) - f(x_*))/\varepsilon^2\right)$. Given local computation budget $\mathcal B$, we set $T := \Theta(1 + n/(\mathcal B P))$ and $Kb := \Theta(\mathcal B)$ with $b \leq \Theta(\sqrt{\mathcal B})$. Then, we have the averaged number of local computations per communication round $Kb + n/(PT) = \Theta(\mathcal B)$ and the communication complexity $TS$ with budget $\mathcal B$ becomes $${\textstyle{\widetilde O\left(1 + \frac{n}{\mathcal B P} + \frac{L}{\sqrt{\mathcal B}\varepsilon^2} + \frac{\sqrt{n}L}{\mathcal BP\varepsilon^2} +  \frac{\zeta}{\varepsilon^2}\right), }}$$
which matches the best known communication complexity \cite{murata2021bias}. 
\end{remark}

\subsection{Escaping Saddle Points}
Next, we show that BVR-L-PSGD implicitly exploits the negative curvature of $f$ around saddle points and efficiently escapes the saddle points by utilizing the asymptotic consistency of BVR estimator and the parameter perturbation at each local update. 

We rely on the technique of coupling sequence \cite{jin2021nonconvex}. Given saddle point $\widetilde x_{I(k_0, t_0, s_0)}$ and $\hat I \geq I(k_0, t_0, s_0)$, we define a new sequence $\{x_i'\}_{i=I(k_0, t_0, s_0)}^\infty$ 
as follows:
%that satisfies $x_{I(k_0, t_0, s_0)} - x_{I(k_0, t_0, s_0)}' = 2\eta \langle \xi_{I(k_0, t_0, s_0)}, \bm{e}_\mathrm{min}\rangle \bm{e}_\mathrm{min}$, where $\bm{e}_\mathrm{min}$ is the normalized eigenvector associated with the (negative) minimum eigenvalue of $\mathcal H := \nabla^2 f(\widetilde x_{I(k_0, t_0, s_0)})$, and all the other randomness are completely same as the ones of $\{x_i\}_{i=I(k_0, t_0, s_0)}^\infty$.

(1) $\langle \xi_{\widetilde I}', \bm{e}_\mathrm{min}\rangle = - \langle  \xi_{\widetilde I}, \bm{e}_\mathrm{min}\rangle$; (2) $\langle \xi_{\widetilde I}', \bm{e}_j\rangle = \langle  \xi_{\widetilde I}, \bm{e}_j \rangle$ for $j \in \{2, \ldots, d\}$; and (3) All the other randomness is completely same as the one of $\{x_i\}_{i=0}^{KTS-1}$. Let $r_0 := |\langle \xi_{\widetilde I}, \bm{e}_\mathrm{min}\rangle|$. Note that  $|\langle \xi_{\widetilde I} - \xi_{\widetilde I}', \bm{e}_\mathrm{min}\rangle| = 2 r_0$ and thus $\|\xi_{\widetilde I} - \xi_{\widetilde I}'\| = 2r_0$. Also, observe that $x_{\widetilde I+1} - x_{\widetilde I+1}' = \eta \langle \xi_{\widetilde I} - \xi_{\widetilde I}', \bm{e}_\mathrm{min}\rangle \bm{e}_\mathrm{min}$.
We define $\widetilde I$ used in the definition of coupling sequence as follows:
\begin{align*}
    \widetilde I := 
    \begin{cases}
        I(k_0, t_0, s_0), & (1/(\eta \lambda) \leq \sqrt{K}) \\
        I(k_0', t_0, s_0) - 1, &(\sqrt{K} < 1/(\eta \lambda) \leq K) \\
        I(0, t_0+1, s_0)-1, & (K < 1/(\eta \lambda) \leq KT) \\
        I(0, 0, s_0+1)-1. & (KT < 1/(\eta \lambda))
    \end{cases}
\end{align*}
Here, $k_0'$ is the minimum index $k$ that satisfies $k > k_0$ and $k \equiv 0\ (\mathrm{mod} \lceil \sqrt{K}\rceil)$. We can easily check that $\widetilde I - I(k_0, t_0, s_0) \leq 1/(\eta \lambda)$.

Then, we show that either of the two sequences $\{x_i\}$ or $\{x_i'\}$ efficiently escapes the saddle points by bounding the norm of the cumulative difference of $x_i$ and $x_i'$ from below. The novel and most difficult part of the analysis is to evaluate the norm of the cumulative difference of the deviations $\|\sum_{i=\widetilde I}^\mathcal J (1-\eta \mathcal H)^{\mathcal J - i}(v_i - \nabla f(x_i) - v_i' + \nabla f(x_i'))\|$ generated by the two sequences, where $v_i'$ denotes the BVR estimator at iteration $i$ generated by sequence $\{x_i'\}$. 
\par
\begin{proposition}[Implicit Negative Curvature Exploitation]
\label{prop: f_diff_decrease}
Let $I(k_0, t_0, s_0) \in [KTS]\cup \{0\}$. Suppose that Assumptions \ref{assump: heterogeneous}, \ref{assump: local_loss_grad_lipschitzness}, \ref{assump: optimal_sol}, \ref{assump: local_loss_hessian_lipschitzness} and \ref{assump: bounded_loss_gradient} hold, $\|\nabla f(\widetilde x_{I(k_0, t_0, s_0)})\| \leq \varepsilon$ and the minimum eigenvalue $\lambda_\mathrm{min}$ of $\mathcal H:= \nabla f(\widetilde x_{I(k_0, t_0, s_0)})$ satisfies $\lambda := - \lambda_\mathrm{min} > \sqrt{\rho \varepsilon}$. Under $b = \Omega(K \vee 1/(\sqrt{K}\rho\varepsilon) \vee T/(PK))$, if we appropriately choose $\mathcal J_{I(k_0, t_0, s_0)} = \widetilde \Theta(1/(\eta\lambda))$, $\eta = \widetilde \Theta(1/L \wedge 1/(K\zeta) \wedge \sqrt{b/K}/L \wedge \sqrt{Pb}/(\sqrt{KT}/L))$, with $\mathcal F_{I(k_0, t_0, s_0)} := c_\mathcal F \eta \mathcal J_{I(k_0, t_0, s_0)} r^2$ and $r := c_r\varepsilon$ ($c_\mathcal F = \Theta(1)$ and $c_r = \widetilde \Theta(1)$) we have
\begin{align*}
    f(x_{I(k_0, t_0, s_0)+\mathcal J_{I(k_0, t_0, s_p)}}) - f(x_{I(k_0, t_0 s_0)}) 
    \leq&\ -  \mathcal F_{I(k_0, t_0, s_0)} + R_2
\end{align*}
with probability at least $1/2 - 9 q /2$. Here, $R_2 :=  \frac{2c_\eta}{\eta}\frac{\mathcal J_{I(k_0, t_0, s_0)}\wedge K }{K} \sum_{i=I(0, t_0, s_0)}^{I(k_0, t_0, s_0)-1} \|x_{i+1} - x_i\|^2 + \frac{2c_\eta}{\eta}\frac{\mathcal J_{I(k_0, t_0, s_0)}\wedge KT}{KT}\sum_{i=I(0, 0, s_0)}^{I( 0, t_0, s_0)-1}\|x_{i+1} - x_{i}\|^2$
    for some universal constant $c_\eta > 0$.  
\end{proposition}
Proposition \ref{prop: f_diff_decrease} says the function value decreases by roughly $\mathcal F_{I(k_0, t_0, s_0)}$ and the global model escapes saddle points with probability at least $1/2$ after $\mathcal J_{I(k_0, t_0, s_0)}$ local steps.

\subsection{Finding Second-Order Stationary Points}
In this subsection, we derive final theorem that guarantees the second-order optimality of the global model by combining Propositions \ref{prop: first_order_optimality} and \ref{prop: f_diff_decrease}. \par
\begin{theorem}[Final Theorem]\label{thm: main}
Suppose that Assumptions \ref{assump: heterogeneous}, \ref{assump: local_loss_grad_lipschitzness}, \ref{assump: optimal_sol}, \ref{assump: local_loss_hessian_lipschitzness} and \ref{assump: bounded_loss_gradient} hold. Under $b = \Omega(K \vee 1/(\sqrt{K}\rho\varepsilon) \vee T/(PK))$, if we appropriately choose $\eta = \widetilde \Theta(1/L \wedge 1/(K\zeta) \wedge \sqrt{b/K}/L \wedge \sqrt{Pb}/(\sqrt{KT}L))$, $r = \widetilde \Theta(\varepsilon)$ and $S = \Theta(1 + (f(\widetilde x_0) - f(x_*))/(\eta KT \varepsilon^2))$, with probability at least $1/2$, there exists $i \in [KTS]\cup\{0\}$ such that $\widetilde x_i$ is $\varepsilon$-second-order optimal point of $f$\footnotemark. 
%\begin{align*}
%    \|\nabla f(\widetilde x_i)\| \leq \varepsilon \text{\ \ and\ \ } \nabla^2 f(\widetilde x_i) \succeq - \sqrt{\rho\varepsilon}.
%\end{align*} 
\end{theorem}
\footnotetext{One limitation of Theorem \ref{thm: main} is that it only guarantees the existence of $\varepsilon$-second-order optimal point $\widetilde x_i$ in the history of $\{\widetilde x_i\}_{i=0}^{KTS-1}$. However, this is also the case in the existing studies \cite{ge2019stabilized, li2019ssrgd}. We empirically found that the outputs of each communication rounds showed stable performances (see Section \ref{sec: experiments}).}
\begin{comment}
\begin{proof}
We choose $S \geq 8192(f(x_0) - f(x_*))/(\eta KT \varepsilon^2)$. Note that $\mathbb{E}[\iota_M] \geq KTS/8 \geq 1024(f(x_0)-f(x_*))/(\eta \varepsilon^2)$. \\
Suppose that $\mathbb{P}(\widetilde x_{\iota_m} \in \mathcal R_3) \leq 3/4$ for every $m \in [M-1]$. Then, since $- 7\eta/512 + 3/4 \times \eta/64 \leq -\eta/512$, we have $f(x_*) - f(x_0) \leq -(\eta/512)\mathbb{E}[\iota_{M}]\varepsilon^2$
and thus $\mathbb{E}[\iota_{M}] \leq 512(f(x_0) - f(x_*))/(\eta \varepsilon^2)$
from Proposition \ref{prop: second_order_stationary_probability}.
This contradicts the previous lower bound of $\mathbb{E}[\iota_M]$. Therefore, we conclude that there exists $m \in [M-1]$ such that $\mathbb{P}(\widetilde x_{i_m} \in \mathcal R_3) > 3/4$. Remember that $E_i$ is the event that $\widetilde x_{i'} \notin \mathcal R_3$ for all $i' \leq i$. This implies $\mathbb{P}(E_{\iota_{M-1}}^\complement) > 3/4 $, and thus $\mathbb{P}(E_{\iota_{M-1}}) \leq 1/4$. Finally, we bound $\mathbb{P}(E_{KTS})$. From the definition of $M$, we have $\mathbb{E}[\iota_{M-1}] < KTS/8$. Thus, from Markov's inequality, it holds that $\mathbb{P}(\iota_{M-1} \geq KTS) \leq 1/8$. This yields 
\begin{align*}
    \mathbb{P}(E_{KTS}) =&\  \mathbb{P}(E_{KTS}|\iota_{M-1} \geq KTS)\mathbb{P}( \iota_{M-1} \geq KTS) \\
    &+ \mathbb{P}(E_{KTS}|\iota_{M-1} < KTS)\mathbb{P}( \iota_{M-1} < KTS) \\
    \leq&\ 1 \times \frac{1}{8} + \mathbb{P}(E_{\iota_{M-1}}) \leq  1/2. 
\end{align*}

This finishes the proof. 
\end{proof}
\end{comment}
\begin{remark}[High probability bound]
Theorem \ref{thm: main} guarantees that Algorithm \ref{alg: bvr_l_psgd} finds an approximate second-order optimal point in $KTS$ iterations with probability at least $1/2$. Repeating Algorithm \ref{alg: bvr_l_psgd} $\mathrm{log}_2(1/ q )$ times guarantees that the same statement holds with probability at least $1 -  q $. 
\end{remark}

\begin{remark}[Communication complexity]
The total number of communication rounds $TS$ is given by
%\begin{align*}
%    \widetilde O \left( T +  \left(\frac{L}{K} + \zeta+ \frac{L}{\sqrt{Kb}} + \frac{\sqrt{T}L}{\sqrt{KPb}}\right)\frac{f(\widetilde x_0) -. f(x_*)}{\varepsilon^2}\right).
%\end{align*}
$\widetilde O \left( T +  \left(L/K + \zeta+ L/\sqrt{Kb} + \sqrt{T}L/\sqrt{KPb}\right)(f(\widetilde x_0) - f(x_*))/\varepsilon^2\right)$. 
Given local computation budget $\mathcal B$, we set $T := \Theta(1 + n/(\mathcal B P))$ and $Kb := \Theta(\mathcal B)$ with $b \leq \Theta(\sqrt{\mathcal B})$. Then, we have the averaged number of local computations per communication round $Kb + n/(PT) = \Theta(\mathcal B)$ and the communication complexity $\Theta(TS)$ with budget $\mathcal B$ becomes 
$${\textstyle{\widetilde O\left(1 + \frac{n}{\mathcal B P} + \frac{L}{\sqrt{\mathcal B}\varepsilon^2} + \frac{\sqrt{n}L}{\mathcal BP\varepsilon^2} +  \frac{\zeta}{\varepsilon^2}\right).}}$$
This implies that for $\zeta = o(L)$ the communication complexity is strictly smaller than the one of minibatch SGD $\widetilde O(1 + L/\varepsilon^2 + G^2/(\mathcal B P \varepsilon^4))$. 
Note that the rate matches to the one of  BVR-L-SGD \cite{murata2021bias}. Hence, our method finds second-order optimal points without hurting communication efficiency of the state-of-the-art first-order optimality guaranteed method. Furthermore, when $\mathcal B \to \infty$, we have $\widetilde \Theta(1 + \zeta/\varepsilon^2)$, that goes to $\widetilde \Theta(1)$ as $\zeta \to 0$. 
\end{remark}
In summary, BVR-L-PSGD enjoys the desirable  properties (i) and (ii) described in Section \ref{sec: intro}. 
%\begin{remark}[Computational complexity]
%Let $T$, $K$ and $b$ be as in the above remark. Under $\mathcal B \leq \Theta(1 + (n/P^2)\wedge \sqrt{n}L/(P\zeta))$, the total computational complexity per worker $\mathcal B TS$ becomes $\widetilde \Theta( (n/P) + (\sqrt{n}L)/(P\varepsilon^2))$. This indicates that the computational complexity of BVR-L-PSGD enjoys linear speedup with respect to the number of workers $P$ from the optimal rate $\Theta(n + \sqrt{n}{L}{\varepsilon^2})$ in non-distributed  nonconvex first-order optimization.
%\end{remark}

%Moreover, the computational complexity of our method achieves linear speedup with respect to the number of workers from the optimal rate in non-distributed nonconvex first-order optimization. 

\section{Numerical Resutls}
\label{sec: experiments}
In this section, we give some experimental results to verify our theoretical findings. \par
{\bf{Data Preparation.}} We artificially generated heterogeneous local datasets from  CIFAR10\footnotemark\footnotetext{\url{https://www.cs.toronto.edu/~kriz/cifar.html}.} dataset. The data preparation procedure is completely in accordance with \cite{murata2021bias} and the details are found in \cite{murata2021bias}. We set homogeneity parameter $q$ to $0.35$, which captures how similar the local datasets are ($q = 0.1$ corresponds to I.I.D. case and higher $q$ does to higher heterogeneity). 
\par
{\bf{Model. }}We conducted our experiments using a two-hidden layers fully connected neural network with $100$ hidden units and softplus activation. For loss function, we used the standard cross-entropy loss. We initialized parameters by uniformly sampling the parameters from $[-1/100, 1/100]$. 
\par
{\bf{Implemented Algorithms. }} Minibatch SGD, Noisy Minibatch SGD, BVR-L-SGD \cite{murata2021bias} and our proposed BVR-L-PSGD were implemented. We set $K = 64$ and $b = 16$, and thus $\mathcal B = 1024$. For BVR-L-PSGD, the noise radius was tuned from $r \in \{0.5, 2.5, 12.5\}$. For each algorithm, we tuned learning rate $\eta$ from $\{0.005, 0.01, 0.05, 0.1, 0.5, 1.0\}$. The details of the tuning procedure are found in the supplementary material.% (Section \ref{app_sec: experiment}).
\par

{\bf{Evaluation. }}
We compared the implemented algorithms using six criteria of train gradient norm $\|\nabla f(x)\|$; train loss; train accuracy; test gradient norm; test loss and test accuracy against the number of communication rounds. The total number of communication rounds was fixed to $1,000$ for each algorithm. 
We independently repeated the experiments $5$ times and report the mean and standard deviation of the above criteria. 
Due to the space limitation, we will only report train gradient norm, train loss and test accuracy in the main paper. The full results are found in the supplementary material. 
\par

{\bf{Results. }} 
\begin{figure}[t]
\begin{subfigmatrix}{3}
\subfigure[Train Gradient Norm]{\includegraphics[width=4.6cm]{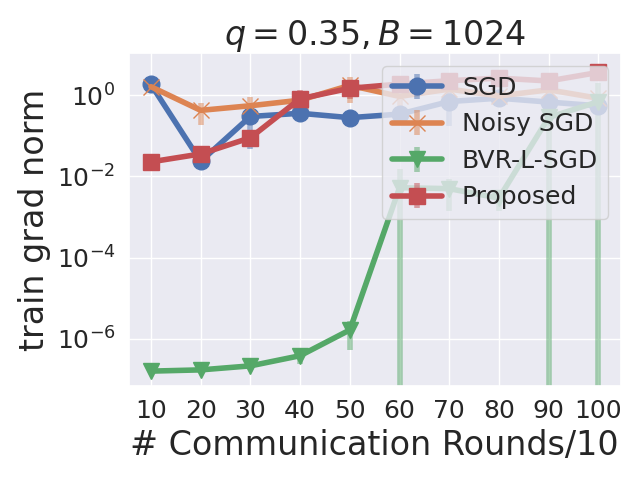}}
\subfigure[Train Loss]{\includegraphics[width=4.6cm]{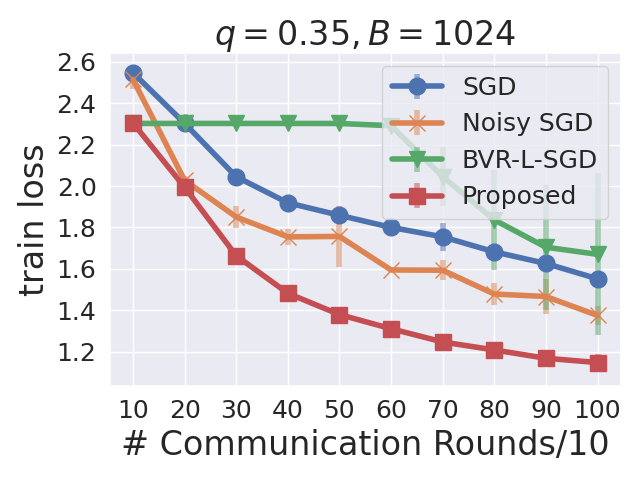}}
\subfigure[Test Accuracy]{\includegraphics[width=4.6cm]{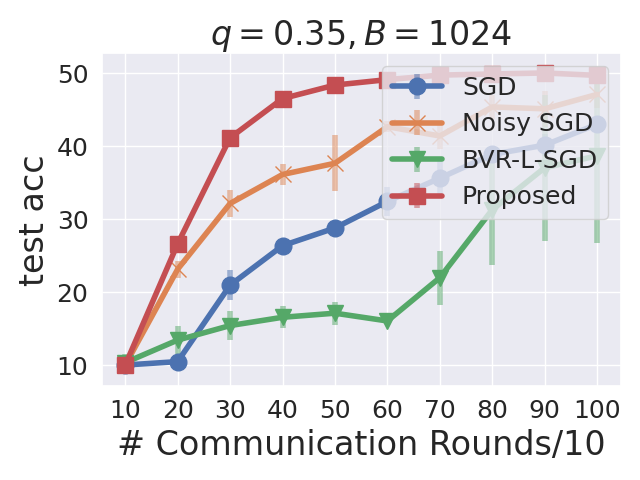}}
\end{subfigmatrix}
\caption{Comparison of (a) train gradient norm; (b) train loss; and (iii) test accuracy against the number of communication rounds for a three layered DNN on heterogeneous CIFAR10. }
\label{fig: comp}
\end{figure}
Figure \ref{fig: comp} shows the performances of BVR-L-SGD and our proposed algorithm. We can see that the both algorithms got stuck at a small gradient norm region in initial rounds. After that BVR-L-SGD showed unstable convergence and took a lot of time to escape the stucked region. In contrast, our proposed method efficiently escaped the stucked region and consistently achieves better train loss and test accuracy than BVR-L-SGD. Also, our method consistently outperformed Minibatch SGD and Noisy Minibatch SGD.

\section{Conclusion}
In this paper, we have studied a new local algorithm called Bias-Variance Reduced Local Perturbed SGD (BVR-L-PSGD) based on a combination of the bias-variance reduced gradient estimator with parameter perturbation to efficiently find second-order optimal points in centralized nonconvex distributed optimization. 
We have shown that BVR-L-PSGD enjoys second-order optimality without hurting the best known communication complexity for first-order optimality guarantees. Particularly, the communication complexity is better than non-local methods when Hessian heterogeneity $\zeta$ of local datasets is smaller than the smoothness of the local loss $L$ in order sense. Also, for sufficiently large $\mathcal B$, the communication complexity of our method approaches to $\widetilde \Theta(1)$ when the local datasets heterogeneity $\zeta$ goes to zero. The numerical results have validated our theoretical findings.

\section*{Acknowledgement}
TS was partially supported by JSPS KAKENHI (20H00576) and JST CREST. The authors would like to thank Kazusato Oko for his helpful advice. 

\bibliographystyle{plain}
\bibliography{main}

\clearpage

\appendix

\section{Supplementary Material for Numerical Results}
In this section, we give additional information and numerical results that complement the contents in Section \ref{sec: experiments}.
\subsection*{Parameter Tuning}
For the implemented algorithms, learning rate $\eta$ was tuned. Also, for Noisy Minibatch SGD and  BVR-L-PSGD, noise radius $r$ was also tuned. We ran each algorithm for all the patterns of the tuning parameters and chose the ones that maximized the minimum train accuracy.

\subsection*{Additional Numerical Results}
Here, we provide the full results of our numerical experiments. Figures \ref{app_fig: q=0.1} and \ref{app_fig: q=0.35} show the comparisons of the six criterion, i.e., train gradient norm, train loss, train accuracy, test gradient norm, test loss and test accuracy with fixed local computation budget $\mathcal B = 1,024$ under $q = 0.1$ (I.I.D. case) and $q = 0.35$ (heterogeneous case) respectively.

\subsection*{Computing Infrastructures}
\begin{itemize}
    \item OS: Ubuntu 16.04.6
    \item CPU: Intel(R) Xeon(R) CPU E5-2680 v4 @ 2.40GHz
    \item CPU Memory: 128 GB.
    \item GPU: NVIDIA Tesla P100.
    \item GPU Memory: 16 GB
    \item Programming language: Python 3.7.3.
    \item Deep learning framework: Pytorch 1.3.1.
\end{itemize}

\begin{figure}[h]
\begin{subfigmatrix}{3}
\subfigure[Train Graident Norm]{\includegraphics[width=4.6cm]{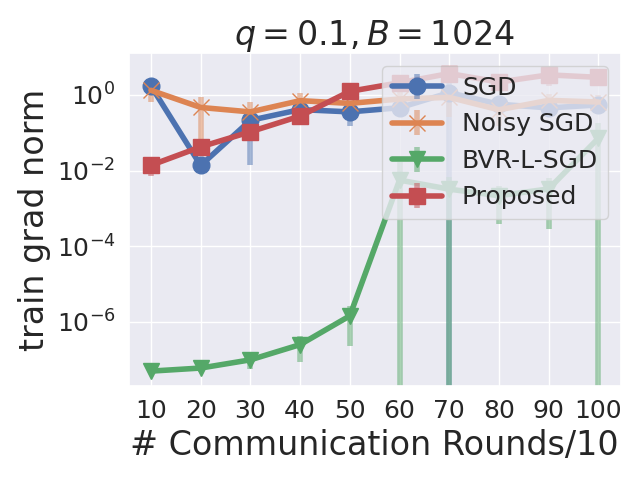}}
\subfigure[Train Loss]{\includegraphics[width=4.6cm]{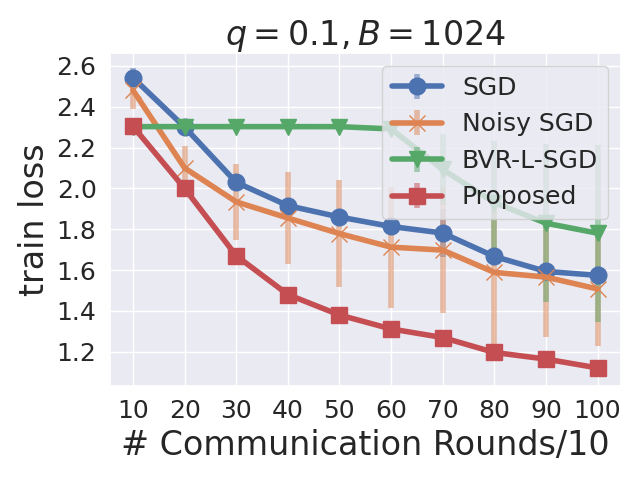}}
\subfigure[Train Accuracy]{\includegraphics[width=4.6cm]{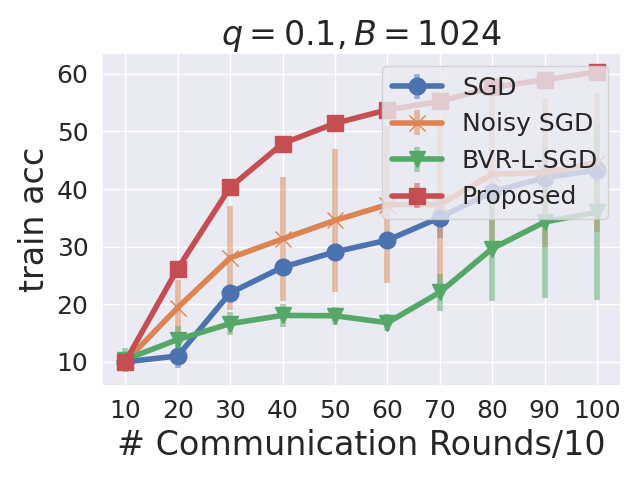}}
\end{subfigmatrix}
\begin{subfigmatrix}{3}
\subfigure[Test Gradient Norm]{\includegraphics[width=4.6cm]{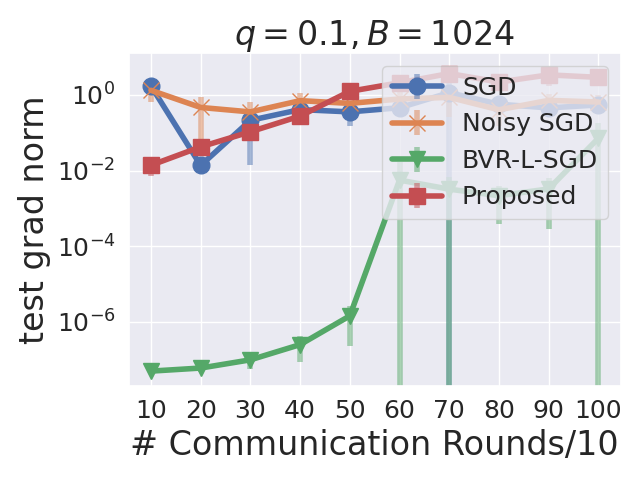}}
\subfigure[Test Loss]{\includegraphics[width=4.6cm]{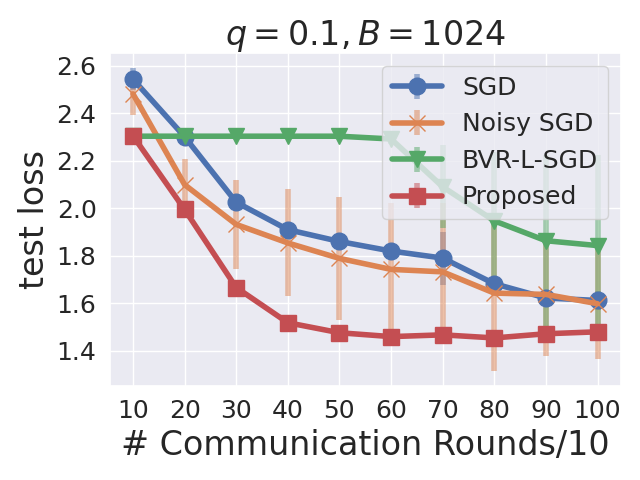}}
\subfigure[Test Accuracy]{\includegraphics[width=4.6cm]{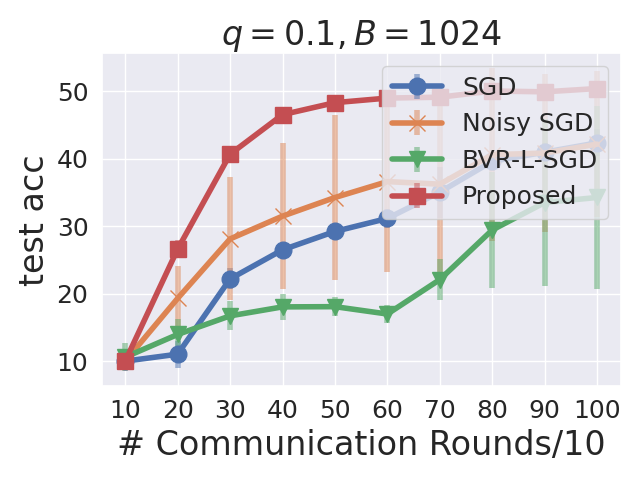}}
\end{subfigmatrix}
\caption{Comparison of the six criterion  against the number of communication rounds for a three layered DNN on I.I.D. CIFAR10 with $q = 0.1$. }
\label{app_fig: q=0.1}
\end{figure}

\begin{figure}[h]
\begin{subfigmatrix}{3}
\subfigure[Train Gradient Norm]{\includegraphics[width=4.6cm]{figs/q=0.35/train_grad_norm.png}}
\subfigure[Train Loss]{\includegraphics[width=4.6cm]{figs/q=0.35/train_loss.png}}
\subfigure[Train Accuracy]{\includegraphics[width=4.6cm]{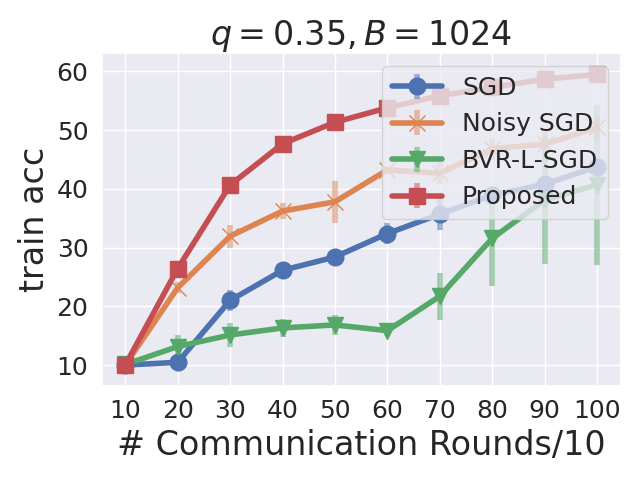}}
\end{subfigmatrix}
\begin{subfigmatrix}{3}
\subfigure[Test Gradient Norm]{\includegraphics[width=4.6cm]{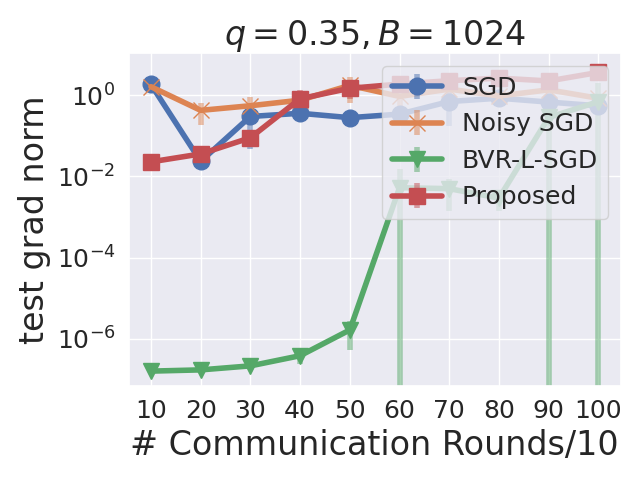}}
\subfigure[Test Loss]{\includegraphics[width=4.6cm]{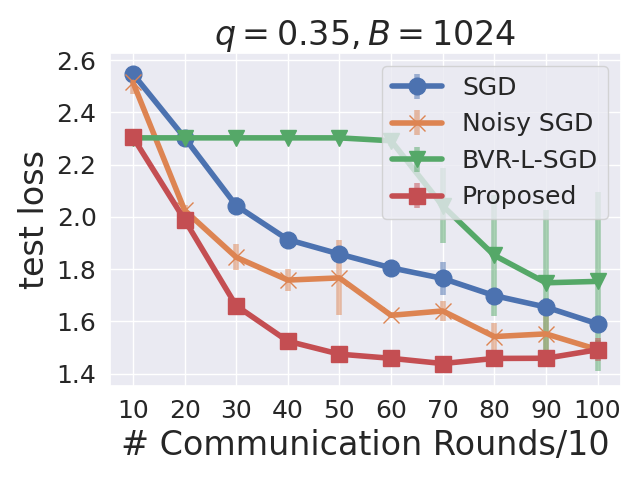}}
\subfigure[Test Accuracy]{\includegraphics[width=4.6cm]{figs/q=0.35/test_acc.png}}
\end{subfigmatrix}
\caption{Comparison of the six criterion  against the number of communication rounds for a three layered DNN on heterogeneous CIFAR10 with $q = 0.35$. }
\label{app_fig: q=0.35}
\end{figure}

\clearpage

\section{Convergence Analysis}
In this section, complete analysis of BVR-L-PSGD is provided. Particularly, detailed proofs of Proposition \ref{prop: first_order_optimality}, Corollary \ref{cor: first_order_optimality} (Subsection \ref{app: subsec: first_order}), Proposition \ref{prop: f_diff_decrease} (Subsection \ref{app: subsec: escapling_saddle} and Theorem \ref{thm: main} (Subsection \ref{app: subsec: main_thm}) are given. 

\subsection{Miscellaneous Results}
\begin{lemma}\label{lem: op_norm_bound}
Let $A \in \mathbb{R}^{d\times d}$ with the smallest and largest eigenvalues $\lambda_\mathrm{min} \in (-\infty, 0)$ and $ \lambda_\mathrm{max}\in [0, 1)$ respectively.  Then, for $J \in \mathbb{N}\cup\{0\}$, it holds that
\begin{align*}
    \|A(1-A)^J\| \leq (-\lambda_\mathrm{min})(1-\lambda_\mathrm{min})^J + \frac{e}{J+1}.
\end{align*}
\end{lemma}
\begin{proof}
First, when $J = 0$, trivially $\|A(1-A)^J\| \leq \mathrm{max}\{-\lambda_\mathrm{min}, \lambda_\mathrm{max}\} < (-\lambda_\mathrm{min}) +  e$. Thus, we assume $J > 0$. Note that $\|A(1-A)^J\| = \mathrm{sup}_{\sigma \in [\lambda_\mathrm{min}, \lambda_\mathrm{max}]} |\sigma(1-\sigma)^J|$. We consider the two cases $\sigma \in [\lambda_\mathrm{min}, 0)$ and $\sigma \in [0, \lambda_\mathrm{max}]$. 

In the former case, $h(\sigma):=|\sigma(1-\sigma)^J| = -\sigma(1-\sigma)^J$ is monotonically decreasing function on $(-\infty, 0)$ because the derivative function $h'(\sigma) = -(1-\sigma)^J + J\sigma(1-\sigma)^{J-1} = (1-\sigma)^{J-1}((J+1)\sigma - 1) < 0$, and hence $\mathrm{sup}_{\sigma \in [\lambda_\mathrm{min}, 0)} h(\sigma)\leq (-\lambda_\mathrm{min})(1-\lambda_\mathrm{min})^J$. 

In the latter case, $h(\sigma) = \sigma(1-\sigma)^J$ has the derivative function $h'(\sigma) = (1-\sigma)^{J} - J\sigma(1-\sigma)^{J-1} = (1-\sigma)^{J-1}(1 - (J+1)\sigma)$. Thus, it holds that $h'(1/(J+1)) = 0$, $h'(\sigma) > 0$ for $\sigma \in [0, 1/(J+1))$ and $h'(\sigma) < 0$ for $\sigma \in (1/(J+1), 1)$. Hence, for  $\sigma \in [0, \lambda_\mathrm{max}]$ with $\lambda_\mathrm{max} \in [0, 1)$, $h(\sigma) \leq h(1/(J+1)) \leq e/(J+1)$.

In summary, we have shown that $\mathrm{sup}_{\sigma \in [\lambda_\mathrm{min}, \lambda_\mathrm{max}]} h(\sigma)\leq (-\lambda_\mathrm{min})(1-\lambda_\mathrm{min})^J + e/(J+1)$. This is the desired result.
\end{proof}

\subsection{Concentration Inequalities}
\begin{comment}
\begin{lemma}[Vector-valued Hoeffding's inequality (Corollary 7 in {\color{red}Chi Jin, short note}]
\label{lem: vector_valued_hoeffding}
Let $X_1, \ldots, X_n$ be independent random vectors in $\mathbb{R}^d$. Suppose that $\{X_i\}_{i=1}^n$ satisfies the following conditions:
\begin{align*}
    \mathbb{E}[X_i] = 0 \text{ and } \|X_i\| \leq \sigma_i^2, \forall i \in [n]
\end{align*}
almost surely.
Then, for any $ q  \in (0, 1)$, with probability at least $1 - q$ it holds that
\begin{align*}
     \left\|\sum_{i=1}^n X_i\right\| \leq c\sqrt{\sum_{i=1}^n \sigma_i^2\mathrm{log}\frac{2d}{ q }}
\end{align*}
for some constant $c > 0$.
\end{lemma}
\end{comment}

\begin{lemma}[Corollary 8 in \cite{jin2019short}]
\label{lem: martingale_concentration_conditioned}
Let $X_1, \ldots, X_n$ be random vectors in $\mathbb{R}^d$. Suppose that $\{X_i\}_{i=1}^n$ and corresponding filtrations $\{\mathfrak F_{i}\}_{i=1}^n$ satisfies the following conditions:
\begin{align*}
    \mathbb{E}[X_i\mid \mathfrak F_{i-1}] = 0 \text{ and } \mathbb{P}(\|X_i\| \geq s\mid \mathfrak F_{i-1}) \leq 2 e^{-\frac{s^2}{2\sigma_i^2}}, \forall s \in \mathbb{R}, \forall i \in [n]
\end{align*}
for random variables $\{\sigma_i\}_{i=1}^n$ with $\sigma_i \in \mathfrak F_{i-1}$ ($i \in [n]$).
Then, for any $ q  \in (0, 1)$ and $A > a > 0$, with probability at least $1 -  q $ it holds that
\begin{align*}
    \sum_{i=1}^n \sigma_i^2 \geq A \text{ or } \left\|\sum_{i=1}^n X_i\right\| \leq c\sqrt{\mathrm{max}\left\{\sum_{i=1}^n \sigma_i^2, a\right\}\left(\mathrm{log}\frac{2d}{q } + \mathrm{log}\mathrm{log}\frac{A}{a}\right)}
\end{align*}
for some constant $c > 0$.
\end{lemma}
%\footnotetext{Constant $c$ is upper bounded by $e^{1/e}/\sqrt{2} \approx 1$. For the details, read the proofs of {\color{red}chi jin short note} and Lemma 5.5 of {\color{red}Vershynin}.}
Note that if $X$ is bounded and centered random vector, i.e., $\|X\| \leq \sigma$ a.s.  and $\mathbb{E}[X] = 0$, it holds that $\mathbb{P}(\|X\|\geq s) \leq 2e^{-s^2/2\sigma^2}$ for every $s \in \mathbb{R}$. Hence, $\|X_i\| \leq \sigma_i^2$ a.s. and $\mathbb{E}[X_i] = 0$ conditioned on $\mathfrak F_{i-1}$ is a sufficient condition for applying Lemma \ref{lem: martingale_concentration_conditioned}.

\subsection{Finding First-Order Stationary Points}
\label{app: subsec: first_order}

\subsection*{Proof of Proposition \ref{prop: first_order_optimality}}
We fix $k \in [K]\cup\{0\}$, $t \in [T-1]\cup\{0\}$ and $s \in [S-1]\cup\{0\}$. 
From $L$-smoothness of $f$, we have
\begin{align*}
    f(x_{I(k+1, t, s)}) \leq f(x_{I(k, t, s)}) + \langle \nabla f(x_{I(k, t, s)}), x_{I(k+1, t, s)} - x_{I(k, t, s)}\rangle + \frac{L}{2}\|x_{I(k+1, t, s)} - x_{I(k, t, s)}\|^2.
\end{align*}
From this inequality, we have
\begin{align}
    f(x_{I(k+1, t, s)}) \leq&\ f(x_{I(k, t, s)})
    + \langle \nabla f(x_{I(k, t, s)}) - v_{I(k, t, s)} + \xi_{I(k, t, s)}, x_{I(k+1, t, s)} - x_{I(k, t, s)}\rangle \notag \\
    &+ \langle v_{I(k, t, s)} - \xi_{I(k, t, s)}, x_{I(k+1, t, s)} - x_{I(k, t, s)}\rangle + \frac{L}{2}\|x_{I(k+1, t, s)} - x_{I(k, t, s)}\|^2 \notag \\
    =&\ f(x_{I(k, t, s)})
    + \langle \nabla f(x_{I(k, t, s)}) - v_{I(k, t, s)} + \xi_{I(k, t, s)}, x_{I(k+1, t, s)} - x_{I(k, t, s)}\rangle \notag\\
    &- \left(\frac{1}{\eta} - \frac{L}{2}\right)\|x_{I(k+1, t, s)} - x_{I(k, t, s)}\|^2 \notag\\
    =& f(x_{I(k, t, s)})
    + \frac{\eta}{2}\|v_{I(k, t, s)} - \xi_{I(k, t, s)} - \nabla f(x_{I(k, t, s)})\|^2 - \frac{\eta}{2}\|\nabla f(x_{I(k, t, s)})\|^2 \notag\\ 
    &+ \frac{1}{2\eta}\|x_{I(k+1, t, s)} - x_{I(k, t, s)}\|^2 - \left(\frac{1}{\eta} - \frac{L}{2}\right)\|x_{I(k+1, t, s)} - x_{I(k, t, s)}\|^2 \notag\\
    =&\ f(x_{I(k, t, s)})
    + \frac{\eta}{2}\|v_{I(k, t, s)} - \xi_{I(k, t, s)} - \nabla f(x_{I(k, t, s)})\|^2 - \frac{\eta}{2}\|\nabla f(x_{I(k, t, s)})\|^2 \notag\\
    &- \left(\frac{1}{2\eta} - \frac{L}{2}\right)\|x_{I(k+1, t, s)} - x_{I(k, t, s)}\|^2 \notag\\
    \leq&\  f(x_{I(k, t, s)})
    + \eta\|v_{I(k, t, s)} - \nabla f(x_{I(k, t, s)})\|^2 - \frac{\eta}{2}\|\nabla f(x_{I(k, t, s)})\|^2 \notag\\
    &- \left(\frac{1}{2\eta} - \frac{L}{2}\right)\|x_{I(k+1, t, s)} - x_{I(k, t, s)}\|^2 + \eta\|\xi_I(k, t, s)\|^2\notag\\
    \leq&\  f(x_{I(k, t, s)})
    + \eta\|v_{I(k, t, s)} - \nabla f(x_{I(k, t, s)})\|^2 - \frac{\eta}{2}\|\nabla f(x_{I(k, t, s)})\|^2 \notag \\
    &- \left(\frac{1}{2\eta} - \frac{L}{2}\right)\|x_{I(k+1, t, s)} - x_{I(k, t, s)}\|^2 + \eta r^2 \label{ineq: descent_one_iter}.
\end{align}
Here, for the first equality we used the fact $v_{I(k, t, s)} - \xi_{I(k, t, s)} = -(1/\eta)(x_{I(k+1, t, s)} - x_{I(k, t, s)})$. The second equality follows from the facts $v_{I(k, t, s)} - \xi_{I(k, t, s)} = (1/\eta)(x_{I(k+1, t, s)} - x_{I(k, t, s)})$ and $\langle a - b, -b\rangle = (1/2)(\|a - b\|^2 - \|a\|^2 + \|b\|^2)$ for any $a, b \in \mathbb{R}^d$. For the second inequality, we used the relation $\|a + b\| \leq 2(\|a\|^2 + \|b\|^2)$ for any $a, b \in \mathbb{R}^d$. The last inequality holds from the definition of $\xi_{I(k, t, s)}$. \par

Thus, for every $k, k_0 \in [K-1]$, $t, t_0 \in [T-1]$ and $s, s_0 \in [S-1]$ ($I(k, t, s) \geq I(k_0, t_0, s_0)$), we have

\begin{align}
    f(x_{I(k, t, s)}) \leq&\  f(x_{I(k_0, t_0, s_0)})
    + \eta \sum_{i=I(k_0, t_0, s_0)}^{I(k, t, s)-1} \|v_i - \nabla f(x_{i})\|^2 \notag \\
    &- \frac{\eta}{2} \sum_{i=I(k_0, t_0, s_0)}^{I(k, t, s)-1} \|\nabla f(x_{i})\|^2 - \left(\frac{1}{2\eta} - \frac{L}{2}\right)\sum_{i=I(k_0, t_0, s_0)}^{I(k, t, s)-1} \|x_{i+1} - x_{i}\|^2 \notag \\
    &+ \eta (I(k, t, s) - I(k_0, t_0, s_0)) r^2 \label{ineq: descent_multi_iters}.
\end{align}

Now we bound the deviation $\|v_{I(k, t, s)} - \nabla f(x_{I(k, t, s)})\|^2$. Observe that
\begin{align*}
    v_{I(k, t, s)} - \nabla f(x_{I(k, t, s)}) =&\  g_{I(k, t, s)} - g_{I(k, t, s)}^{\mathrm{ref}} + v_{I(k-1,t, s)} - \nabla f(x_{I(k, t, s)}) \\
    =&\ g_{I(k, t, s)} - g_{I(k, t, s)}^{\mathrm{ref}} + \nabla f_{p_{t, s}}(x_{I(k-1, t, s)}) - \nabla f_{p_{t, s}}(x_{I(k, t, s)}) \\
    &+ \nabla f_{p_{t, s}}(x_{I(k, t, s)}) - \nabla f_{p_{t, s}}(x_{I(k-1, t, s)}) + \nabla f(x_{I(k-1, t, s)}) - \nabla f(x_{I(k, t, s)}) \\
    &+ v_{I(k-1,t, s)} - \nabla f(x_{I(k-1, t, s)}) \\
    =&\ \sum_{\kappa=0}^{k-1}(g_{I(\kappa+1, t, s)} - g_{I(\kappa+1, t, s)}^{\mathrm{ref}} + \nabla f_{p_{t, s}}(x_{I(\kappa, t, s)}) - \nabla f_{p_{t, s}}(x_{I(\kappa+1, t, s)})) \\
    &+ \sum_{\kappa=0}^{k-1}(\nabla f_{p_{t, s}}(x_{I(\kappa+1, t, s)}) - \nabla f_{p_{t, s}}(x_{I(\kappa, t, s)}) + \nabla f(x_{I(\kappa, t, s)}) - \nabla f(x_{I(\kappa+1, t, s)})) \\
    &+ v_{I(0, t, s)} - \nabla f(x_{I(0, t, s)}).
\end{align*}
Further, we have
\begin{align*}
    v_{I(0, t, s)} - \nabla f(x_{I(0, t, s)}) =&\ \frac{1}{P}\sum_{p=1}^P (g_{I(0, t, s)}^{(p)} - g_{I(0, t, s)}^{(p), \mathrm{ref}} + v_{I(0, t-1, s)} - \nabla f(x_{I(0, t, s)}) \\
    =&\ \frac{1}{P}\sum_{p=1}^P (g_{I(0, t, s)}^{(p)} - g_{I(0, t, s)}^{(p), \mathrm{ref}} + \nabla f(x_{I(0, t-1, s)}) - \nabla f(x_{I(0, t, s)})\\
    &+ v_{I(0, t-1, s)} - \nabla f(x_{I(0, t-1, s)}) \\
    =&\ \sum_{\tau=0}^{t-1}\frac{1}{P}\sum_{p=1}^P (g_{I(0, \tau+1, s)}^{(p)} - g_{I(0, \tau+1, s)}^{(p), \mathrm{ref}} + \nabla f(x_{I(0, \tau, s)}) - \nabla f(x_{I(0, \tau+1, s)}) \\
    &+ v_{I(0, 0, s)} - \nabla f(x_{I(0, 0, s)}).
\end{align*}
Note that the last term is exactly zero from the definition of $v_{I(0, 0, s)}$. 

We define 
\begin{align*}
\begin{cases}
    \alpha_{I(\kappa, t, s)} :=&\ g_{I(\kappa, t, s)} - g_{I(\kappa, t, s)}^{\mathrm{ref}} + \nabla f_{p_{t, s}}(x_{I(\kappa-1, t, s)}) - \nabla f_{p_{t, s}}(x_{I(\kappa, t, s)}), \\
    \beta_{I(\kappa, t, s)} :=&\ \nabla f_{p_{t, s}}(x_{I(\kappa, t, s)}) - \nabla f_{p_{t, s}}(x_{I(\kappa-1, t, s)}) + \nabla f(x_{I(\kappa-1, t, s)}) - \nabla f(x_{I(\kappa, t, s)}), \\
    \gamma_{I(0, \tau, s)} :=&\  \frac{1}{P}\sum_{p=1}^P (g_{I(0, \tau, s)}^{(p)} - g_{I(0, \tau, s)}^{(p), \mathrm{ref}} + \nabla f(x_{I(0, \tau-1, s)}) - \nabla f(x_{I(0, \tau, s)}),
\end{cases}
\end{align*}
and
\begin{align*}
\begin{cases}
    A_{I(k, t, s)} :=&\ \sum_{\kappa=0}^{k-1} \alpha_{I(\kappa+1, t, s)}, \\
    B_{I(k, t, s)} :=&\ \sum_{\kappa=0}^{k-1} \beta_{I(\kappa+1, t, s)}, \\
    C_{I(0, t, s)} :=&\ \sum_{\tau=0}^{t-1} \gamma_{I(0, \tau+1, s)}. \\
\end{cases}
\end{align*}
Note that $\mathbb{E}[A_{I(k, t, s)}] = \mathbb{E}[C_{I(k, t, s)}] = 0$. 
Using these definitions, we have
\begin{align*}
    \|v_{I(k, t, s)} - \nabla f(x_{I(k, t, s)})\|^2 \leq&\ 3(\left\|A_{I(k, t, s)}\right\|^2 + \left\|B_{I(k, t, s)}\right\|^2 + \left\|C_{I(k, t, s)}\right\|^2).
\end{align*}
We denote all the randomness up to iteration $I(\kappa-1, t, s)$ as $\mathfrak F_{I(\kappa-1, t, s)}$. 

\subsection*{Bounding $\|A_{I(k, t, s)}\|$}
Let $\alpha_{l, I(\kappa, t, s)} := \nabla \ell(x_{I(\kappa, t, s)}, z_{l, I(\kappa, t, s)}) - \nabla \ell(x_{I(\kappa-1, t, s)}, z_{l, I(\kappa, t, s)}) + \nabla f_{p_{t, s}}(x_{I(\kappa-1, t, s)}) + \nabla f_{p_{t, s}}(x_{I(\kappa, t, s)})$. Then, $\alpha_{I(\kappa, t, s)} = (1/b)\sum_{l=1}^b \alpha_{l, I(\kappa, t, s)}$. Observe that $\alpha_{l, I(\kappa, t, s)}$ satisfies \begin{align*}
    \mathbb{E}[\alpha_{l, I(\kappa, t, s)}\mid \mathfrak F_{I(\kappa-1, t, s)}] = 0
\end{align*}
and
\begin{align*}
    \mathbb{P}(\|\alpha_{l, I(\kappa, t, s)}\|\geq s \mid \mathfrak F_{ I(\kappa-1, t, s)}) \leq 2e^{-\frac{s^2}{2\left(\sigma_{I(\kappa, t, s)}^{(\alpha)}\right)^2}}
\end{align*} 
for every $s \in \mathbb{R}$ and $\kappa \in [k]$, where $\sigma_{I(\kappa, t, s)}^{(\alpha)} := 2L\|x_{I(\kappa, t, s)} - x_{I(\kappa-1, t, s)}\|$. Here, we used the fact that $\|\nabla \ell(x_{I(\kappa, t, s)}, z_{l, I(\kappa, t, s)}) - \nabla \ell(x_{I(\kappa-1, t, s)}, z_{l, I(\kappa, t, s)}) + \nabla f_{p_{t, s}}(x_{I(\kappa-1, t, s)}) + \nabla f_{p_{t, s}}(x_{I(\kappa, t, s)})\| \leq 2L \|x_{I(\kappa, t, s)} - x_{I(\kappa-1, t, s)}\|$ from $L$-smoothness of $\ell$. Note that $\{\alpha_{l, I(\kappa, t, s)}\}_{l=1}^{b_k}$ is I.I.D. sequence with at least $b$ samples and $\|\alpha_{\ell, I(\kappa, t, s)}\| \leq 4 G$ almost surely from Assumption \ref{assump: bounded_loss_gradient}. From these results, we can use Lemma \ref{lem: martingale_concentration_conditioned} with $A = 4KG$ and $a = \widetilde \varepsilon$ ($\widetilde \varepsilon$ is some positive number and will be defined later) and get
\begin{align*}
    \left\|A_I(k, t, s)\right\|^2 \leq \frac{c^2}{b}\left(\left(\sum_{\kappa=0}^{k-1} \left(\sigma_{I(\kappa+1, t, s)}^{(\alpha)}\right)^2\right)+\widetilde \varepsilon\right)\left(\mathrm{log}\frac{2d}{ q }+\mathrm{log}\mathrm{log}\frac{4KG}{\widetilde \varepsilon}\right)
\end{align*}
with probability at least $1 -  q $ for some constant $c > 0$. Also, note that
$\|A_I(k, t, s)\| \leq 4kG$ almost surely.

\subsection*{Bounding $\|B_{I(k, t, s)}\|$}
Observe that
\begin{align*}
    \beta_{I(\kappa, t, s)} =&\ \nabla f_{p_{t, s}}(x_{I(\kappa, t, s)}) - \nabla f_{p_{t, s}}(x_{I(\kappa-1, t, s)}) + \nabla f(x_{I(\kappa-1, t, s)}) - \nabla f(x_{I(\kappa, t, s)}) \\
    =&\  (\nabla f_{p_{t, s}} - \nabla f)(x_{I(\kappa, t, s)}) - (\nabla f_{p_{t, s}} - \nabla f)(x_{I(\kappa-1, t, s)}) \\
    =&\ \left(\int_0^1 (\nabla^2 f_{p_{t, s}} - \nabla^2 f)(\theta x_{I(\kappa, t, s)} + (1 - \theta) x_{I(\kappa-1, t, s)})d\theta\right) (x_{I(\kappa, t, s)} - x_{I(\kappa-1, t, s)}).
\end{align*}
Hence, from Assumption \ref{assump: heterogeneous}, we get
\begin{align*}
    \|\beta_{I(\kappa, t, s)}\| \leq \zeta\|x_{I(\kappa, t ,s)} - x_{I(\kappa-1, t, s)}\|=: \sigma_{I(\kappa, t, s)}^{(\beta)
    }.
\end{align*}
This gives
\begin{align*}
    \left\|B_{I(k, t, s)}\right\|^2 \leq k\sum_{\kappa=0}^{k-1} \left(\sigma_{I(\kappa+1, t, s)}^{(\beta)}\right)^2.
\end{align*}
Here we used the relation $(\sum_{i=1}^m |a_i|)^2 \leq m\sum_{i=1}^m a_i^2$ for every $\{a_i\}_{i=1}^m \subset \mathbb{R}$. 
Also, note that $\|B_I(k, t, s)\| \leq 4kG$ almost surely.

\subsection*{Bounding $\|C_{I(0, t, s)}\|$}
The argument is similar to the one of the case of the first term. From Lemma  \ref{lem: martingale_concentration_conditioned}, the third term $\|C_{I(0, t, s)}\|$ can be bounded as
\begin{align*}
    \left\|C_I(0, t, s)\right\|^2 \leq \frac{c^2}{PKb}\left(\left(\sum_{\tau=0}^{t-1} 
    \left(\sigma_{I(0, \tau+1, s)}^{(\gamma)}\right)^2\right)+\widetilde \varepsilon\right)\left(\mathrm{log}\frac{2d}{ q }+\mathrm{log}\mathrm{log}\frac{4KTG}{\widetilde \varepsilon}\right),
\end{align*}
with probability at least $1 -  q $, where $\sigma_{I(0, \tau, s)}^{(\gamma)} := 2L\sum_{\kappa=0}^{K-1}\|x_{I(\kappa+1, \tau-1, s)} - x_{I(\kappa, \tau-1, s)}\|$ ($\geq 2L\|x_{I(0, \tau, s)} - x_{I(0, \tau-1, s)}\|$). Here, we used the fact that $\{g_{I(0, \tau, s)}^{(p)} - g_{I(0, \tau, s)}^{(p), \mathrm{ref}}\}_{p=1}^P$ is independent and each of them is constructed from $Kb$ i.i.d. data samples. Also, note that $\|C_I(k, t, s)\| \leq 4TG$ almost surely.

Put the three results all together, we obtain 
\begin{equation*}
  \begin{split}
    \|v_{I(k, t, s)} - \nabla f(x_{I(k, t, s)})\|^2 \leq&\  \frac{3c^2}{b}\left(\left(\sum_{\kappa=0}^{k-1} \left(\sigma_{I(\kappa+1, t, s)}^{(\alpha)}\right)^2\right)+\widetilde \varepsilon\right)\left(\mathrm{log}\frac{2KT Sd}{ q }+\mathrm{log}\mathrm{log}\frac{4KG}{\widetilde \varepsilon}\right) \\
    &+ 3k\sum_{\kappa=0}^{k-1} \left(\sigma_{I(\kappa+1, t, s)}^{(\beta)}\right)^2 \\
    &+ \frac{3c^2}{PKb}\left(\left(\sum_{\tau=0}^{t-1} \left(\sigma_{I(0, \tau+1, s)}^{(\gamma)}\right)^2\right)+\widetilde \varepsilon\right)\left(\mathrm{log}\frac{2KT Sd}{ q }+\mathrm{log}\mathrm{log}\frac{4TG}{\widetilde \varepsilon}\right)
  \end{split}
\end{equation*}
for every $k \in [K-1]$, $t \in [T-1]$ and $s \in [ S-1]$  with probability at least $1 - 3 q $ for some constant $c > 0$.
We set $ q  \leftarrow  q /(KT S)$. %Further, we set $\widetilde \varepsilon := \mathrm{min}\{r^2/(12 c^2 (\mathrm{log}(2KT Sd/q)+\mathrm{log}\mathrm{log}(4KTG))), (r^2/12c^2)^2\} = \widetilde \Theta(r^{4})$. Then, 
%it holds that that 
Now, we set 
{\color{blue}\begin{align*}
    6c^2 \widetilde \varepsilon \left(\mathrm{log}\frac{2KT Sd}{ q } + \mathrm{log}\mathrm{log}\frac{4KTG}{\widetilde \varepsilon}\right) \leq r^2.
\end{align*}
}
%Here, we used the fact that $x \mathrm{log}\mathrm{log}(1/x) \leq 2\sqrt{x}$ by $\mathrm{log}(1/x) \leq 1/x - 1$ for $x > 0$. 
Then, we have
\begin{equation}
  \begin{split}
    \|v_{I(k, t, s)} - \nabla f(x_{I(k, t, s)})\|^2 \leq&\  \frac{3c^2}{b}\left(\sum_{\kappa=0}^{k-1} \left(\sigma_{I(\kappa+1, t, s)}^{(\alpha)}\right)^2\right)\left(\mathrm{log}\frac{2KT Sd}{ q }+\mathrm{log}\mathrm{log}\frac{G}{\widetilde \varepsilon}\right) \\
    &+ 3k\sum_{\kappa=0}^{k-1} \left(\sigma_{I(\kappa+1, t, s)}^{(\beta)}\right)^2 \label{ineq: variance_bound}\\
    &+ \frac{3c^2}{PKb}\left(\sum_{\tau=0}^{t-1} \left(\sigma_{I(0, \tau+1, s)}^{(\gamma)}\right)^2\right)\left(\mathrm{log}\frac{2KT Sd}{ q }+\mathrm{log}\mathrm{log}\frac{G}{\widetilde \varepsilon}\right) \\
    &+r^2
  \end{split}
\end{equation}

for every $I(k, t, s) \in [KT  S]\cup\{0\}$. \par
Let 
\begin{align*}
    V(k, t, s) :=&\  12c^2\left(\frac{L^2}{b} + K\zeta^2 + \frac{L^2 T}{Pb}\right)\left( \sum_{\kappa=0}^{k-1}\|x_{I(\kappa+1, t, s)} - x_{I(\kappa, t, s)}\|^2 + \frac{1}{T}\sum_{\tau=0}^{t-1} \sum_{\kappa=0}^{K-1}\|x_{I(\kappa+1, \tau, s)} - x_{I(\kappa, \tau, s)}\|^2\right) \\
    & \times \left(\mathrm{log}\frac{2KT Sd}{ q }+\mathrm{log}\mathrm{log}\frac{4KTG}{\widetilde \varepsilon}\right). 
\end{align*}
Observe that $\|v_{I(k, t, s)} - \nabla f(x_{I(k, t, s)})\|^2 \leq V(k, t, s) + r^2$ and $V(k, t, s) \leq V(k', t', s)$ for $k' \geq k$ and $t' \geq t$. \par 

Now, we bound $\sum_{i=I(k_0, t_0, s_0)}^{I(k, t, s)}\|v_i - \nabla f(x_i)\|^2$ by dividing three cases. 
\subsection*{Case I. $s=s_0$ and $t=t_0$.}
We bound $\sum_{i=I(k_{-}, t_{-}, s_{-})}^{I(k, t, s)}\|v_i - \nabla f(x_i)\|^2$ for general $k_{-}, t_{-}$ and $s_{-}$ with $k_{-} \leq k$,  $t_{-} = t$ and $s_{-} = s$.

\begin{align*}
    &\sum_{i=I(k_{-}, t_{-}, s_{-})}^{I(k, t_-, s_-)}\|v_i - \nabla f(x_i)\|^2 \\
    \leq&\ \sum_{k'=k_{-}}^k V(k', t_{-}, s_{-}) + (k-k_{-}+1)r^2 \\
    \leq&\ (k-k_{-}+1) V(k, t_{-}, s_{-})  + (k-k_{-}+1)r^2 \\
    \leq&\ 12c^2\left(\frac{K L^2}{b} + K^2\zeta^2 + \frac{K L^2 T}{Pb}\right) \\
    &\times \left(\frac{k-k_{-}+1}{K} \sum_{\kappa=0}^{k-1}\|x_{I(\kappa+1, t_{-}, s_{-})} - x_{I(\kappa, t_{-}, s_{-})}\|^2 + \frac{k-k_{-}+1}{KT}\sum_{\tau=0}^{t_{-}-1} \sum_{\kappa=0}^{K-1}\|x_{I(\kappa+1, \tau, s_{-})} - x_{I(\kappa, \tau, s_{-})}\|^2\right) \\
    & \times \left(\mathrm{log}\frac{2KTSd}{ q }+\mathrm{log}\mathrm{log}\frac{4KTG}{\widetilde \varepsilon}\right) \\
    &+ (k-k_{-}+1) r^2.
\end{align*}
Since 
\begin{align*}
    &\frac{k-k_{-}+1}{K} \sum_{\kappa=0}^{k-1}\|x_{I(\kappa+1, t_{-}, s_{-})} - x_{I(\kappa, t_{-}, s_{-})}\|^2 + \frac{k-k_{-}+1}{KT}\sum_{\tau=0}^{t_{-}-1} \sum_{\kappa=0}^{K-1}\|x_{I(\kappa+1, \tau, s_{-})} - x_{I(\kappa, \tau, s_{-})}\|^2 \\
    \leq&\ \sum_{\kappa=k_{-}}^{k-1}\|x_{I(\kappa+1, t_{-}, s_{-})} - x_{I(\kappa, t_{-}, s_{-})}\|^2 + \frac{k-k_{-}+1}{K}\sum_{\kappa=0}^{k_--1}\|x_{I(\kappa+1, t_-, s_{-})} - x_{I(\kappa, t_-, s_-)}\|^2 \\
    &+ \frac{k-k_{-}+1}{KT}\sum_{\tau=0}^{t_{-}-1} \sum_{\kappa=0}^{K-1}\|x_{I(\kappa+1, \tau, s_{-})} - x_{I(\kappa, \tau, s_{-})}\|^2 \\ 
    =&\ \sum_{i=I(k_{-}, t_{-}, s_{-})}^{I(k, t_-, s_-)-1} \|x_{i+1} - x_i\|^2 + \frac{(I(k, t, s) - I(k_{-}, t_{-}, s_{-})+1)\wedge K }{K} \sum_{i=I(0, t_{-}, s_{-})}^{I(k_{-}, t_{-}, s_{-})-1} \|x_{i+1} - x_i\|^2 \\
    &+ \frac{(I(k, t_-, s_-) - I(k_{-}, t_{-}, s_{-}) + 1)\wedge KT}{KT}\sum_{i=I(0, 0, s_{-})}^{I(0, t_{-}, s_{-})-1} \|x_{i+1} - x_{i}\|^2, 
\end{align*}
we get
\begin{align*}
    &\sum_{i=I(k_{-}, t_{-}, s_{-})}^{I(k, t_-, s_-)}\|v_i - \nabla f(x_i)\|^2 \\
    \leq&\ 12c^2\left(\frac{K L^2}{b} + K^2\zeta^2 + \frac{K L^2 T}{Pb}\right) \\
    &\times \left(\sum_{i=I(k_{-}, t_{-}, s_{-})}^{I(k, t_-, s_-)-1} \|x_{i+1} - x_i\|^2 + \frac{(I(k, t_-, s_-) - I(k_{-}, t_{-}, s_{-}) + 1)\wedge K }{K} \sum_{i=I(0, t_{-}, s_{-})}^{I(k_{-}, t_{-}, s_{-})-1} \|x_{i+1} - x_i\|^2 \right. \\
    &+ \left. \frac{(I(k, t_-, s_-) - I(k_{-}, t_{-}, s_{-}) + 1)\wedge KT}{KT}\sum_{i=I(0, 0, s_{-})}^{I(0, t_{-}, s_{-})-1}\|x_{i+1} - x_{i}\|^2\right) \\
    & \times \left(\mathrm{log}\frac{2KT Sd}{ q }+\mathrm{log}\mathrm{log}\frac{4KTG}{\widetilde \varepsilon}\right) \\
    &+ (I(k, t_-, s_-) - I(k_-, t_-, s_-)+1) r^2.
\end{align*}

Setting $k_{-} \leftarrow k_0$, $t_{-} \leftarrow t_0$ and $s_{-} \leftarrow s_0$ gives the desired bound.

\subsection*{Case II. $s=s_0$ and $t>t_0$.}
Note that $I(k, t, s_0) - I(k_0, t_0, s_0) \geq K$.
Again, we consider $\sum_{i=I(k_{-}, t_{-}, s_{-})}^{I(k, t, s_-)}\|v_i - \nabla f(x_i)\|^2$ for general $k_{-}, t_{-}$ and $s_{-}$ with $k \geq k_{-}$,  $t > t_{-}$ and $s = s_{-}$.

\begin{align*}
    &\sum_{i=I(k_{-}, t_{-}, s_{-})}^{I(k, t, s_-)}\|v_i - \nabla f(x_i)\|^2 \\
    \leq&\ \sum_{i=I(k_{-}, t_{-}, s_{-})}^{I(K-1, t_{-}, s_{-})} \|v_{i} - \nabla f(x_i)\|^2 + \sum_{t'=t_{-}+1}^{t-1} \sum_{i=I(0, t', s_{-})}^{I(K-1, t', s_{-})} \|v_{i} - \nabla f(x_{i})\|^2 +  \sum_{i=I(0, t, s_{-})}^{I(k, t, s_{-})} \|v_{i} - \nabla f(x_{i})\|^2.
\end{align*}
Using the result of Case I, the first term can be bounded as follows:
\begin{align*}
    &\sum_{i=I(k_{-}, t_{-}, s_{-})}^{I(K-1, t_{-}, s_{-})} \|v_{i} - \nabla f(x_i)\|^2 \\
     \leq&\ 12c^2\left(\frac{K L^2}{b} + K^2\zeta^2 + \frac{K L^2 T}{Pb}\right) \\
    &\times \left(\sum_{i=I(k_{-}, t_{-}, s_{-})}^{I(K-1, t_{-}, s_-)-1} \|x_{i+1} - x_i\|^2 + \frac{(I(K-1, t_{-}, s_{-}) - I(k_{-}, t_{-}, s_{-}) + 1)\wedge K }{K} \sum_{i=I(0, t_{-}, s_{-})}^{I(k_{-}, t_{-}, s_{-})-1} \|x_{i+1} - x_i\|^2 \right. \\
    &\ \ \ \ \ \ \ \ + \left. \frac{(I(K-1, t_{-}, s_{-}) - I(k_{-}, t_{-}, s_{-}) + 1)\wedge KT}{KT}\sum_{i=I(0, 0, s_{-})}^{I(0, t_{-}, s_{-})-1}\|x_{i+1} - x_{i}\|^2\right) \\
    & \times \left(\mathrm{log}\frac{2KT Sd}{ q }+\mathrm{log}\mathrm{log}\frac{4KTG}{\widetilde \varepsilon}\right) \\
    &+ (I(K-1, t_-, s_-) - I(k_-, t_-, s_-) + 1) r^2.
\end{align*}

Similarly, the second term can be bounded as:
\begin{align*}
     &\sum_{t'=t_{-}+1}^{t-1} \sum_{i=I(0, t', s_{-})}^{I(K-1, t', s_{-})} \|v_{i} - \nabla f(x_{i})\|^2 \\
     \leq&\ 12c^2\left(\frac{K L^2}{b} + K^2\zeta^2 + \frac{K L^2 T}{Pb}\right) \\
    &\times \left(\sum_{t'=t_{-}+1}^{t-1}\sum_{i=I(0, t', s_{-})}^{I(K-1, t', s_-)-1} \|x_{i+1} - x_i\|^2 \right. \\
    &\ \ \ \ \ \ \ \ + \left. \sum_{t'=t_{-}+1}^{t-1}\frac{(I(K-1, t', s_{-}) - I(0, t', s_{-})+1)\wedge KT}{KT}\sum_{i=I(0, 0, s_{-})}^{I(0, t', s_{-})-1}\|x_{i+1} - x_{i}\|^2\right) \\
    & \times \left(\mathrm{log}\frac{2KT Sd}{ q }+\mathrm{log}\mathrm{log}\frac{4KTG}{\widetilde \varepsilon}\right) \\
    \leq&\ 12c^2\left(\frac{K L^2}{b} + K^2\zeta^2 + \frac{K L^2 T}{Pb}\right) \\
    &\times \left(\sum_{i=I(0, t_{-}+1, s_{-})}^{I(K-1, t-1, s_-)-1} \|x_{i+1} - x_i\|^2 + \sum_{t'=t_{-}+1}^{t-1}\frac{(I(K-1, t', s_{-}) - I(0, t', s_{-}) + 1)\wedge KT}{KT}\sum_{i = I(0, t_-, s_-)}^{I(0, t', s_-)-1}\|x_{i+1} - x_i\|^2\right. \\
    &\ \ \ \ \ \ \ \ + \left.  \frac{(I(K-1, t-1, s_{-}) -I(0, t_{-}+1, s_{-}) + 1)\wedge KT}{KT}\sum_{i=I(0, 0, s_{-})}^{I( 0, t_{-}, s_{-})-1}\|x_{i+1} - x_{i}\|^2\right) \\
    & \times \left(\mathrm{log}\frac{2KT Sd}{q }+\mathrm{log}\mathrm{log}\frac{4KTG}{\widetilde \varepsilon}\right)
    \\
    &+ (I(K-1, t - 1, s_-) - I(0, t_-+1, s_-) + 1) r^2. 
\end{align*}
Now, we bound the second term as follows:
\begin{align*}
    &\sum_{t'=t_{-}+1}^{t-1}\frac{(I(K-1, t', s_{-}) - I(0, t', s_{-})+1)\wedge KT}{KT}\sum_{i = I(0, t_-, s_-)}^{I(0, t', s_-)-1}\|x_{i+1} - x_i\|^2 \\
    \leq&\ \frac{(I(K-1, t_-+1, s_-) - I(0, t_-+1, s_-)+1) \wedge KT}{KT}\sum_{i=I(0, t_-, s_-}^{I(k_-, t_-, s_-)-1}\|x_{i+1} - x_i\|^2 + \sum_{i=I(k_-, t_-, s_-)}^{I(0, t_-+1, s_-)-1}\|x_{i+1} - x_i\|^2 \\
    &+  \sum_{t'=t_{-}+2}^{t-1}\frac{(I(K-1, t', s_{-}) - I(0, t', s_{-})+1)\wedge KT}{KT}\sum_{i = I(0, t_-, s_-)}^{I(0, t', s_-)-1}\|x_{i+1} - x_i\|^2 \\
    \leq&\ \frac{(I(K-1, t_-+1, s_-) - I(0, t_-+1, s_-)+1) \wedge KT}{KT}\sum_{i=I(0, t_-, s_-}^{I(k_-, t_-, s_-)-1}\|x_{i+1} - x_i\|^2 + \sum_{i=I(k_-, t_-, s_-)}^{I(K-1, t-1, s_)-1}\|x_{i+1} - x_i\|^2. 
\end{align*}
Using this, we have
\begin{align*}
    &\sum_{t'=t_{-}+1}^{t-1} \sum_{i=I(0, t', s_{-})}^{I(K-1, t', s_{-})} \|v_{i} - \nabla f(x_{i})\|^2 \\
    \leq&\ 12c^2\left(\frac{K L^2}{b} + K^2\zeta^2 + \frac{K L^2 T}{Pb}\right) \\
    &\times \left(2\sum_{i=I(k_-, t_-, s_{-})}^{I(K-1, t-1, s_-)-1} \|x_{i+1} - x_i\|^2 + \frac{(I(K-1, t_-+1, s_{-}) - I(0, t_-+1, s_{-}) + 1)\wedge K}{K}\sum_{i = I(0, t_-, s_-)}^{I(k_-, t_-, s_-)-1}\|x_{i+1} - x_i\|^2\right. \\
    &\ \ \ \ \ \ \ \ + \left.  \frac{(I(K-1, t-1, s_{-}) -I(0, t_{-}+1, s_{-}) + 1)\wedge KT}{KT}\sum_{i=I(0, 0, s_{-})}^{I( 0, t_{-}, s_{-})-1}\|x_{i+1} - x_{i}\|^2\right) \\
    & \times \left(\mathrm{log}\frac{2KT Sd}{q }+\mathrm{log}\mathrm{log}\frac{4KTG}{\widetilde \varepsilon}\right)
    \\
    &+ (I(K-1, t - 1, s_-) - I(0, t_-+1, s_-) + 1) r^2.
\end{align*}

Finally, we bound the last term:
\begin{align*}
    &\sum_{i=I(0, t, s_{-})}^{I(k, t, s_{-})-1} \|v_{i} - \nabla f(x_{i})\|^2 \\
    \leq&\ 12c^2\left(\frac{K L^2}{b} + K^2\zeta^2 + \frac{K L^2 T}{Pb}\right) \\
    &\times \left(\sum_{i=I(0, t, s_{-})}^{I(k, t, s_{-})-1} \|x_{i+1} - x_i\|^2 +  \frac{(I(k, t, s_{-}) - I(0, t, s_{-})+1)\wedge KT}{KT}\sum_{i=I(0, 0, s_{-})}^{I(0, t, s_{-})-1}\|x_{i+1} - x_{i}\|^2\right) \\
    & \times \left(\mathrm{log}\frac{2KT Sd}{ q }+\mathrm{log}\mathrm{log}\frac{4KTG}{\widetilde \varepsilon}\right) \\
    \leq&\ 12c^2\left(\frac{K L^2}{b} + K\zeta^2 + \frac{K L^2 T}{Pb}\right) \\
    &\times \left(\sum_{i=I(k_{-}, t_{-}, s_{-})}^{I(k, t, s_{-})-1} \|x_{i+1} - x_i\|^2 + \frac{(I(k, t, s_{-}) - I(0, t, s_{-})+1)\wedge KT}{KT}\sum_{i=I(0, 0, s_{-})}^{I(k_{-}, t_{-}, s_{-})-1}\|x_{i+1} - x_{i}\|^2\right) \\
    & \times \left(\mathrm{log}\frac{2KT Sd}{ q }+\mathrm{log}\mathrm{log}\frac{4KTG}{\widetilde \varepsilon}\right) \\
    &+ (I(k, t, s_{-}) - I(0, t, s_{-}) + 1)r^2.
\end{align*}
Summing the upper bounds of the three terms, we get
\begin{align*}
     &\sum_{i=I(k_{-}, t_{-}, s_{-})}^{I(k, t, s_-)}\|v_i - \nabla f(x_i)\|^2 \\
    \leq&\ 48c^2\left(\frac{K L^2}{b} + K^2\zeta^2 + \frac{K L^2 T}{Pb}\right) \\
    &\times \left(\sum_{i=I(k_{-}, t_{-}, s_{-})}^{I(k, t, s_{-})-1} \|x_{i+1} - x_i\|^2 + \frac{(I(k, t, s_{-}) - I(k_{-}, t_{-}, s_{-})+1)\wedge K }{K} \sum_{i=I(0, t_{-}, s_{-})}^{I(k_{-}, t_{-}, s_{-})-1} \|x_{i+1} - x_i\|^2 \right. \\
    &+ \left. \frac{(I(k, t, s_{-}) -I(k_{-}, t_{-}, s_{-})+1)\wedge KT}{KT}\sum_{i=I(0, 0, s_{-})}^{I( 0, t_{-}, s_{-})-1}\|x_{i+1} - x_{i}\|^2\right) \\
    & \times \left(\mathrm{log}\frac{2KT Sd}{ q }+\mathrm{log}\mathrm{log}\frac{4KTG}{\widetilde \varepsilon}\right)
     \\
    &+ (I(k, t, s_{-}) - I(k_-, t_-, s_{-}) + 1)r^2.
\end{align*}

Setting $k_{-} \leftarrow k_0$, $t_{-} \leftarrow t_0$ and $s_{-} \leftarrow s_0$ gives the desired bound.

\subsection*{Case III. $s>s_0$}
In this case, note that $I(k, t, s) - I(k_0, t_0, s_0) \geq KT$ holds. Observe that 
\begin{align*}
    &\sum_{i=I(k_0, t_0, s_0)}^{I(k, t, s)}\|v_i - \nabla f(x_i)\|^2 \\
    \leq&\ \sum_{i=I(k_0, t_0, s_0)}^{I(K-1, T-1, s_0)} \|v_i - \nabla f(x_i)\|^2 + \sum_{s'=s_0+1}^{s-1}\sum_{i=I(0, 0, s')}^{I(K-1, T-1, s')} \|v_i - \nabla f(x_i)\|^2 + \sum_{i=I(0, 0, s)}^{I(k, t, s)} \|v_i - \nabla f(x_i)\|^2.
\end{align*}

Using the result of Case II, we bound the three terms. 

The first term can be bounded as follows:
\begin{align*}
    &\sum_{i=I(k_0, t_0, s_0)}^{I(K-1, T-1, s_0)} \|v_i - \nabla f(x_i)\|^2 \\
    \leq&\ 48c^2\left(\frac{K L^2}{b} + K^2\zeta^2 + \frac{K L^2 T}{Pb}\right) \\
    &\times \left(\sum_{i=I(k_0, t_0, s_0)}^{I(K-1, T-1, s_0)-1} \|x_{i+1} - x_i\|^2 + \frac{(I(K-1, T-1, s_0) - I(k_0, t_0, s_0)+1)\wedge K }{K} \sum_{i=I(0, t_0, s_0)}^{I(k_0, t_0, s_0)-1} \|x_{i+1} - x_i\|^2 \right. \\
    &+ \left. \frac{(I(K-1, T-1, s_0) - I(k_0, t_0, s_0)+1)\wedge KT}{KT} \sum_{i=I(0, 0, s_0)}^{I(0, t_0, s_0)-1} \|x_{i+1}-x_i\|^2 \right) \\
    & \times \left(\mathrm{log}\frac{2KT Sd}{ q }+\mathrm{log}\mathrm{log}\frac{G}{\widetilde \varepsilon}\right)     \\
    &+ (I(K-1, T-1, s_{0}) - I(k_0, t_0, s_{0}) + 1)r^2.
\end{align*}

Similarly, the second term can be bounded as
\begin{align*}
    &\sum_{s'=s_0+1}^{s-1}\sum_{i=I(0, 0, s')}^{I(K-1, T-1, s')} \|v_i - \nabla f(x_i)\|^2 \\
    \leq&\ 48c^2\left(\frac{K L^2}{b} + K^2\zeta^2 + \frac{K L^2 T}{Pb}\right) \left(\sum_{s'=s_0+1}^{s-1}\sum_{i=I(0, 0, s')}^{I(K-1, T-1, s')-1} \|x_{i+1} - x_i\|^2 \right) \left(\mathrm{log}\frac{2KT Sd}{ q }+\mathrm{log}\mathrm{log}\frac{4KTG}{\widetilde \varepsilon}\right)    \\
    &+ (I(K-1, T-1, s-1) - I(k_0, t_0, s_{0}+1) + 1)r^2.
\end{align*}

We bound the last term as
\begin{align*}
    &\sum_{i=I(0, 0, s)}^{I(k, t, s)} \|v_i - \nabla f(x_i)\|^2 \\
    \leq&\ 48c^2\left(\frac{K L^2}{b} + K^2\zeta^2 + \frac{K L^2 T}{Pb}\right) \\
    &\times \left(\sum_{i=I(0, 0, s)}^{I(k, t, s)-1} \|x_{i+1} - x_i\|^2 \right) \left(\mathrm{log}\frac{KT Sd}{ q }+\mathrm{log}\mathrm{log}\frac{4KTG}{\widetilde \varepsilon}\right)    \\
    &+ (I(k, t, s) - I(0, 0, s) + 1)r^2.
\end{align*}

Summing up the three terms, we get
\begin{align*}
    &\sum_{i=I(k_0, t_0, s_0)}^{I(k, t, s)} \|v_i - \nabla f(x_i)\|^2 \\
    \leq&\ 48c^2\left(\frac{K L^2}{b} + K^2\zeta^2 + \frac{K L^2 T}{Pb}\right) \\
    &\times \left(\sum_{i=I(k_0, t_0, s_0)}^{I(k, t, s)-1} \|x_{i+1} - x_i\|^2 + \frac{(I(k, t, s) - I(k_0, t_0, s_0)+1)\wedge K }{K} \sum_{i=I(0, t_0, s_0)}^{I(k_0, t_0, s_0)-1} \|x_{i+1} - x_i\|^2 \right. \\
    &+ \left.  \frac{(I(k, t, s) - I(k_0, t_0, s_0)+1)\wedge KT}{KT} \sum_{i=I(0, 0, s_0)}^{I(0, t_0, s_0)-1} \|x_{i+1}-x_i\|^2\right) \\
    &\times \left(\mathrm{log}\frac{2KT Sd}{ q }+\mathrm{log}\mathrm{log}\frac{4KTG}{\widetilde \varepsilon}\right)    \\
    &+ (I(k, t, s) - I(k_0, t_0, s_{0}) + 1)r^2.
\end{align*}

Combining the three cases, we obtain 
\begin{align*}
     &\sum_{i=I(k_0, t_0, s_0)}^{I(k, t, s)}\|v_i - \nabla f(x_i)\|^2 \\
    \leq&\ 48c^2\left(\frac{K L^2}{b} + K^2\zeta^2 + \frac{K L^2 T}{Pb}\right) \\
    &\times \left(\sum_{i=I(k_0, t_0, s_0)}^{I(k, t, s)-1} \|x_{i+1} - x_i\|^2 + \frac{(I(k, t, s) - I(k_0, t_0, s_0)+1)\wedge K }{K} \sum_{i=I(0, t_0, s_0)}^{I(k_0, t_0, s_0)-1} \|x_{i+1} - x_i\|^2 \right. \\
    &+ \left. \frac{(I(k, t, s) -I(k_0, t_0, s_0)+1)\wedge KT}{KT}\sum_{i=I(0, 0, s_0)}^{I( 0, t_0, s_0)-1}\|x_{i+1} - x_{i}\|^2\right) \\
    & \times \left(\mathrm{log}\frac{2KT Sd}{ q }+\mathrm{log}\mathrm{log}\frac{4KTG}{\widetilde \varepsilon}\right)\\
    &+ (I(k, t, s) - I(k_0, t_0, s_{0}) + 1)r^2.
\end{align*}

Combining this bound with (\ref{ineq: descent_multi_iters}), we obtain 
\begin{align*}
    &f(x_{I(k, t, s)})   \\
    \leq&\ f(x_{I(k_0, t_0, s_0)}) - \frac{\eta}{2}\sum_{i=I(k_0, t_0, s_0)}^{I(k, t, s)-1}\|\nabla f(x_i)\|^2\\
    &- \left(\frac{1}{2\eta} - \frac{L}{2} - 48c^2 \eta \left(\frac{KL^2}{b} +  K^2\zeta^2 + \frac{KTL^2}{Pb}\right)\left(\mathrm{log}\frac{2KT Sd}{ q } + \mathrm{log}\mathrm{log}\frac{4KTG}{\widetilde \varepsilon}\right)\right)\sum_{i=I(k_0, t_0, s_0)}^{I(k, t, s)-1}\|x_{i+1} - x_{i}\|^2 \\
    &+ \Biggl\{48c^2 \eta \left(\frac{KL^2}{b} +  K^2\zeta^2 + \frac{KTL^2}{Pb}\right)\left(\mathrm{log}\frac{2KT Sd}{ q } + \mathrm{log}\mathrm{log}\frac{4KTG}{\widetilde \varepsilon}\right)\\
    &\times \left(\frac{(I(k, t, s) - I(k_0, t_0, s_0))\wedge K }{K} \sum_{i=I(0, t_0, s_0)}^{I(k_0, t_0, s_0)-1} \|x_{i+1} - x_i\|^2 \right. \\
    &+ \left. \frac{(I(k, t, s) -I(k_0, t_0, s_0))\wedge KT}{KT}\sum_{i=I(0, 0, s_0)}^{I( 0, t_0, s_0)-1}\|x_{i+1} - x_{i}\|^2\right)\Biggr\} \\
    &+ \eta (I(k, t, s) - I(k_0, t_0, s_0))r^2
\end{align*}
with probability at least $1 - 3 q $.

We can choose {\color{blue}$\eta = \widetilde \Theta(1/L \wedge \sqrt{b/K}/L \wedge 1/(K\zeta) \wedge \sqrt{Pb}/(\sqrt{KT}L))$} such that {\color{blue}$\eta \leq 1/(8L)$} and 
{\color{blue}
\begin{align*}
    48c^2 \eta \left(\frac{KL^2}{b} +  K^2\zeta^2 + \frac{KTL^2}{Pb}\right)\left(\mathrm{log}\frac{2KT Sd}{ q } + \mathrm{log}\mathrm{log}\frac{4KTG}{\widetilde \varepsilon}\right) \leq \frac{c_\eta}{\eta}.
\end{align*}
}
for some constant $c_\eta \in (0, 1/4)$. 
Then, the above result can be simplified as 
\begin{align}
    f(x_{I(k, t, s)})  
    \leq&\ f(x_{I(k_0, t_0, s_0)}) - \frac{\eta}{2}\sum_{i=I(k_0, t_0, s_0)}^{I(k, t, s)-1}\|\nabla f(x_i)\|^2 \notag\\
    &- \frac{1}{8\eta}\sum_{i=I(k_0, t_0, s_0)}^{I(k, t, s)-1}\|x_{i+1} - x_{i}\|^2 \notag\\
    &+ \frac{c_\eta}{\eta}\left(\frac{(I(k, t, s) - I(k_0, t_0, s_0))\wedge K }{K} \sum_{i=I(0, t_0, s_0)}^{I(k_0, t_0, s_0)-1} \|x_{i+1} - x_i\|^2 \right. \notag\\
    &+ \left. \frac{(I(k, t, s) -I(k_0, t_0, s_0))\wedge KT}{KT}\sum_{i=I(0, 0, s_0)}^{I( 0, t_0, s_0)-1}\|x_{i+1} - x_{i}\|^2\right) \notag\\
    &+ \eta (I(k, t, s) - I(k_0, t_0, s_0))r^2 \label{ineq: f_diff}
\end{align}
with probability at least $1-3 q $.
\qed

Also, we bound $\|x_{I(k, t, s)} - x_{I(k_0, t_0, s_0)}\|^2$. Note that
\begin{align*}
    \|x_{I(k, t, s)} - x_{I(k_0, t_0, s_0)}\|^2  \leq&\ (I(k, t, s) - I(k_0, t_0, s_0))\sum_{i=I(k_0,t_0, s_0)}^{I(k, t, s)-1}\|x_{i+1} - x_i\|^2 \\
    \leq&\ 8\eta (I(k, t, s) - I(k_0, t_0, s_0)) \Biggl\{f(x_{I(k_0, t_0, s_0)}) - f(x_{I(k, t, s)}) \Biggr. \\
    &- \frac{1}{8\eta} \sum_{i=I(k_0,t_0, s_0)}^{I(k-1, t, s)}\|x_{i+1} - x_i\|^2 \\
    &+ \frac{c_\eta}{\eta}\left(\frac{(I(k, t, s) - I(k_0, t_0, s_0))\wedge K }{K} \sum_{i=I(0, t_0, s_0)}^{I(k_0, t_0, s_0)-1} \|x_{i+1} - x_i\|^2 \right. \notag\\
    &+ \left. \frac{(I(k, t, s) -I(k_0, t_0, s_0))\wedge KT}{KT}\sum_{i=I(0, 0, s_0)}^{I( 0, t_0, s_0)-1}\|x_{i+1} - x_{i}\|^2\right) \\
    &+  \eta (I(k, t, s) - I(k_0, t_0, s_0))r^2 \Biggr\}.
\end{align*}
for every $k, k_0 \in [K-1]$, $t, t_0 \in [T-1]$ and $s, s_0 \in [ S-1]$ ($I(k, t, s) \geq I(k_0, t_0, s_0)$) with probability at least $1 - 3 q $. \par

By the way, we also derive (loose) almost sure bound as follows:
From (\ref{ineq: descent_one_iter}) and the fact that $\|v_{I(k, t, s)} - \nabla f(x_{I(k, t, s)})\|^2 \leq 3(4KG^2 + 4KG^2 + 4TG^2) \leq 36 KT G^2$ almost surely, it holds that
\begin{align}
    f(x_{I(k, t, s)}) - f(x_{I(k_0, t_0, s_0)}) \leq&\ 36 \eta KT (I(k, t, s) - I(k_0, t_0, s_0))G^2 + \eta (I(k, t, s) - I(k_0, t_0, s_0))r^2 \notag \\
    \leq&\ 36 \eta K^2T^2  S G^2 + \eta KT S r^2 \notag \\
    \leq&\ 36\eta K^2T^2 S(G^2+ r^2)\label{ineq: almost_sure_object_bound}
\end{align}
almost surely. \par

\begin{comment}
\begin{corollary}
Suppose that Assumptions \ref{assump: heterogeneous}, \ref{assump: local_loss_grad_lipschitzness}, \ref{assump: optimal_sol} and \ref{assump: bounded_loss_gradient} hold. Given $q \in (0, 1)$, if we appropriately choose 
$\eta = \widetilde \Theta(1/L \wedge \sqrt{b/K}/L \wedge 1/(K\zeta) \wedge \sqrt{Pb}/(\sqrt{KT}L))$ and $r = \Theta(\varepsilon)$, it holds that 
\begin{align}
    \frac{1}{KTS}\sum_{i=0}^{KTS - 1}\|\nabla f(\widetilde x_i)\|^2 \leq \frac{4(f(x_0) - f(x_*))}{\eta KTS} + \frac{\varepsilon^2}{2} \notag
\end{align}
with probability at least $1-3q$. Particularly, if we set $S \geq 8(f(x_0) - f(x_*))/(\eta KT\varepsilon^2)$,  with probability at least $1-3q$, there exists $i \in [KTS]\cup\{0\}$ such that $\|\nabla f(\widetilde x_i)\|\leq \varepsilon$.
\end{corollary}
\end{comment}

\subsection*{Proof of Corollary \ref{cor: first_order_optimality}}
Using Proposition \ref{prop: first_order_optimality} with $I(k, t, s) = KTS$ and $I(k_0, t_0, s_0) = 0$, we have
\begin{align*}
    f(x_{KTS})  
    \leq&\ f(x_0) - \frac{\eta}{2}\sum_{i=0}^{KTS - 1}\|\nabla f(x_i)\|^2 \\
    &- \frac{1}{8\eta}\sum_{i=0}^{KTS-1}\|x_{i+1} - x_{i}\|^2 + \eta KTS r^2.
\end{align*}

Then, $-\|\nabla f(x_{i})\|^2 \leq -(1/2)\|\nabla f(\widetilde x_{i})\|^2 + \|\nabla f(x_{i}) - f(\widetilde x_{i})\|^2 \leq -(1/2)\|\nabla f(\widetilde x_{I(k, t, s)})\|^2 + \eta^2L^2r^2 \leq -(1/2)\|\nabla f(\widetilde x_{i})\|^2 + r^2$ gives 
\begin{align*}
    f(x_{KTS})  
    \leq&\ f(x_0) - \frac{\eta}{4}\sum_{i=0}^{KTS-1}\|\nabla f(\widetilde x_i)\|^2 \\
    &- \frac{1}{8\eta}\sum_{i=0}^{KTS-1}\|x_{i+1} - x_{i}\|^2 + 2\eta KTS r^2.
\end{align*}

Choosing {\color{blue}$c_r \leq 1/4$} immediately leads the desired result.
\qed

\subsection{Escaping Saddle Points}
\label{app: subsec: escapling_saddle}

Given $\{x_i\}_{i=0}^{KTS-1}$, we introduce the concept of coupling sequence \cite{jin2021nonconvex}. Given $x_{I(k_0, t_0, s_0)}$, let  $\{\bm{e}_i\}_{i=1}^d$ be the normalized eigenvectors of $\nabla f(\widetilde x_{I(k_0, t_0, s_0)})$ associated with the eigenvalues $\lambda_1 < \cdots < \lambda_d$. We set  $\bm{e}_\mathrm{min} := \bm{e}_1$ and $\lambda_\mathrm{min} := \lambda_1$. We assume that $\lambda := -\lambda_\mathrm{min} > \sqrt{\rho\varepsilon}$.

Then, for given $\hat I \geq I(k_0, t_0, s_0)$, we define coupling sequence $\{x_i'\}_{i=0}^{KTS-1}$ as follows: (1) $\langle \xi_{\widetilde I}', \bm{e}_\mathrm{min}\rangle = - \langle  \xi_{\widetilde I}, \bm{e}_\mathrm{min}\rangle$; (2) $\langle \xi_{\widetilde I}', \bm{e}_j\rangle = \langle  \xi_{\widetilde I}, \bm{e}_j \rangle$ for $j \in \{2, \ldots, d\}$; and (3) All the other randomness is completely same as the one of $\{x_i\}_{i=0}^{KTS-1}$. Let $r_0 := |\langle \xi_{\widetilde I}, \bm{e}_\mathrm{min}\rangle|$. Note that  $|\langle \xi_{\widetilde I} - \xi_{\widetilde I}', \bm{e}_\mathrm{min}\rangle| = 2 r_0$ and thus $\|\xi_{\widetilde I} - \xi_{\widetilde I}'\| = 2r_0$. Also, observe that $x_{\widetilde I+1} - x_{\widetilde I+1}' = \eta \langle \xi_{\widetilde I} - \xi_{\widetilde I}', \bm{e}_\mathrm{min}\rangle \bm{e}_\mathrm{min}$.
We define $\widetilde I$ used in the definition of the coupling sequence as follows:
\begin{align*}
    \widetilde I := 
    \begin{cases}
        I(k_0, t_0, s_0), & (1/(\eta \lambda) \leq \sqrt{K}) \\
        I(k_0', t_0, s_0) - 1, &(\sqrt{K} < 1/(\eta \lambda) \leq K) \\
        I(0, t_0+1, s_0)-1, & (K < 1/(\eta \lambda) \leq KT) \\
        I(0, 0, s_0+1)-1. & (KT < 1/(\eta \lambda))
    \end{cases}
\end{align*}
Here, $k_0'$ is the minimum index $k$ that satisfies $k > k_0$ and $k \equiv 0\ (\mathrm{mod} \lceil \sqrt{K}\rceil)$. We can easily check that $\widetilde I - I(k_0, t_0, s_0) \leq 1/(\eta \lambda)$.

\par

Note that
\begin{align}
    \mathbb{P}\left(r_0 \geq  \frac{qr}{2\sqrt{d}}\right) \geq 1 -  q \label{ineq: unstack_prop}
\end{align}
from the arguments in Section A.2 of \cite{jin2017escape}.

\begin{comment}
The objective of this section is to prove the following proposition: 
\begin{proposition}
\label{prop: f_diff_decrease}
Let $I(k_0, t_0, s_0) \in [KTS]\cup \{0\}$. Suppose that Assumptions \ref{assump: heterogeneous}, \ref{assump: local_loss_grad_lipschitzness}, \ref{assump: optimal_sol}, \ref{assump: local_loss_hessian_lipschitzness} and \ref{assump: bounded_loss_gradient} hold, $\|\nabla f(\widetilde x_{I(k_0, t_0, s_0)})\| \leq \varepsilon$ and the minimum eigenvalue $\lambda_\mathrm{min}$ of $\mathcal H:= \nabla f(\widetilde x_{I(k_0, t_0, s_0)})$ satisfies $\lambda < -\sqrt{\rho \varepsilon}$. If we appropriately choose $\mathcal J_{I(k_0, t_0, s_0)} = \widetilde \Theta(1/(\eta\lambda))$, $\eta = \widetilde \Theta(1/L \wedge 1/(K\zeta) \wedge \sqrt{b/K}/L \wedge \sqrt{Pb}/(\sqrt{KT}/L))$, with $\mathcal F_{I(k_0, t_0, s_0)} := c_\mathcal F \eta \mathcal J_{I(k_0, t_0, s_0)} r^2$ and $r := c_r\varepsilon$ ($c_\mathcal F = \Theta(1)$ and $c_r = \widetilde \Theta(1)$) we have
\begin{align}
    &f(x_{I(k_0, t_0, s_0)+\mathcal J}) - f(x_{I(k_0, t_0 s_0)}) \notag \\
    \leq&\ -  \mathcal F_{I(k_0, t_0, s_0)} \notag \\
    &+ \frac{c_\eta}{\eta}\left(\frac{\mathcal J_{I(k_0, t_0, s_0)}\wedge K }{K} \sum_{i=I(0, t_0, s_0)}^{I(k_0, t_0, s_0)-1} \|x_{i+1} - x_i\|^2 + \frac{\mathcal J_{I(k_0, t_0, s_0)}\wedge KT}{KT}\sum_{i=I(0, 0, s_0)}^{I( 0, t_0, s_0)-1}\|x_{i+1} - x_{i}\|^2\right)\label{ineq: f_diff_escape_saddle}.
\end{align}
with probability at least $1/2 - 9 q /2$. 
\end{proposition}
\end{comment}

To prove Proposition \ref{prop: f_diff_decrease}, first note that the following result:
\begin{proposition}
\label{prop: improve_or_localize}
Let $k_0 \in [K]\cup\{0\}$, $t_0 \in [T-1]\cup\{0\}$ and $s_0 \in [S-1]\cup\{0\}$. Fix any $\mathcal J \in \{1, \ldots, I(0, 0, S) - I(k_0, t_0, s_0)\}$ and $\mathcal F > 0$. Under the same conditions as Proposition \ref{prop: first_order_optimality}, it holds that
\begin{align*}
    &\mathrm{min}\left\{f(x_{I(k_0, t_0, s_0)+\mathcal J}) - f(x_{I(k_0, t_0 s_0)}), f(x_{I(k_0, t_0, s_0)+\mathcal J}') - f(x_{I(k_0, t_0 s_0)}')\right\} \\ \leq&\  - \mathcal F + \frac{2c_\eta}{\eta}\left(\frac{\mathcal J\wedge K }{K} \sum_{i=I(0, t_0, s_0)}^{I(k_0, t_0, s_0)-1} \|x_{i+1} - x_i\|^2 + \frac{\mathcal J\wedge KT}{KT}\sum_{i=I(0, 0, s_0)}^{I( 0, t_0, s_0)-1}\|x_{i+1} - x_{i}\|^2\right)\\ 
    \text{ or }& \forall J \in [\mathcal J]: \mathrm{max}\left\{\|x_{I(k_0, t_0, s_0)+J} - x_{I(k_0, t_0, s_0)}\|^2, \|x_{I(k_0, t_0, s_0)+J}' - x_{I(k_0, t_0, s_0)}'\|^2\right\} \\
    \leq&\ 8\eta \mathcal J \left(\mathcal F + 2\eta \mathcal J r^2\right)
\end{align*}
with probability at least $1 - 6 q $. 
\end{proposition}
\begin{proof}
First note that $x_{i} = x_i'$ for $i \leq I(k_0, t_0, s_0)$. 
From the bounds of $\|x_{I(k_0, t_0, s_0)+J} - x_{I(k_0, t_0, s_0)}\|^2$ and $\|x_{I(k_0, t_0, s_0)+J}' - x_{I(k_0, t_0, s_0)}'\|^2$, we can see that
\begin{align*}
    &\mathrm{max}\left\{\|x_{I(k_0, t_0, s_0)+J} - x_{I(k_0, t_0, s_0)}\|^2, \|x_{I(k_0, t_0, s_0)+J}' - x_{I(k_0, t_0, s_0)}'\|^2\right\} \\
    \leq&\ 8\eta J \Biggl(\mathrm{max}\left\{f(x_{I(k_0, t_0, s_0)}) - f(x_{I(k_0, t_0 s_0)+J}) - \frac{1}{8\eta} \sum_{i=I(k_0,t_0, s_0)}^{I(k_0, t_0, s_0)+J-1}\|x_{i+1} - x_i\|^2, \right. \\
    & \left. f(x_{I(k_0, t_0, s_0)}') - f(x_{I(k_0, t_0 s_0)+J}') - \frac{1}{8\eta} \sum_{i=I(k_0,t_0, s_0)}^{I(k_0, t_0, s_0)+J-1}\|x_{i+1}' - x_i'\|^2\right\} \Biggr. \\
    &+ \Biggl. \frac{c_\eta}{\eta}\left(\frac{ J\wedge K }{K} \sum_{i=I(0, t_0, s_0)}^{I(k_0, t_0, s_0)-1} \|x_{i+1} - x_i\|^2 + \frac{J\wedge KT}{KT}\sum_{i=I(0, 0, s_0)}^{I( 0, t_0, s_0)-1}\|x_{i+1} - x_{i}\|^2\right)+ \eta  {J} r^2\Biggr)
\end{align*}
for every $J \in [\mathcal J]$ with probability at least $1 - 6 q $.

We define $I(k_J, t_J, s_J) := I(k_0, t_0, s_0) + J$. Note that $s_J \geq s_0$. From (\ref{ineq: f_diff}), 
\begin{align*}
    &f(x_{I(k_0, t_0, s_0)}) - f(x_{I(k_0, t_0, s_0)+J}) - \frac{1}{8\eta} \sum_{i=I(k_0,t_0, s_0)}^{I(k_0, t_0, s_0)+J-1}\|x_{i+1} - x_i\|^2 \\
    =&\ f(x_{I(k_0, t_0, s_0)}) - f(x_{I(k_0, t_0, s_0)+\mathcal J}) \\
    &+ f(x_{I(k_0, t_0, s_0)+\mathcal J}) - f(x_{I(k_0, t_0, s_0)+J}) -  \frac{1}{8\eta} \sum_{i=I(k_0,t_0, s_0)}^{I(k_0, t_0, s_0)+J-1}\|x_{i+1} - x_i\|^2\\
    \leq&\ f(x_{I(k_0, t_0, s_0)}) - f(x_{I(k_0, t_0, s_0)+\mathcal J}) \\
    &-  \frac{1}{8\eta} \sum_{i=I(k_0,t_0, s_0)}^{I(k_0, t_0, s_0)+J-1}\|x_{i+1} - x_i\|^2\\
    &+ \frac{c_\eta}{\eta}\left(\frac{(\mathcal J - J)\wedge K }{K} \sum_{i=I(0, t_J, s_J)}^{I(k_J, t_J, s_J)-1} \|x_{i+1} - x_i\|^2 + \frac{(\mathcal J - J)\wedge KT}{KT}\sum_{i=I(0, 0, s_J)}^{I(0, t_J, s_J)-1}\|x_{i+1} - x_{i}\|^2\right) \\
    &+ \eta (\mathcal J - J)  r^2 \\
    \leq&\ f(x_{I(k_0, t_0, s_0)}) - f(x_{I(k_0, t_0, s_0)+\mathcal J}) \\
    &+ \frac{c_\eta}{\eta}\left(\frac{\mathcal J\wedge K }{K} \sum_{i=I(0, t_0, s_0)}^{I(k_0, t_0, s_0)-1} \|x_{i+1} - x_i\|^2 + \frac{\mathcal J\wedge KT}{KT}\sum_{i=I(0, 0, s_0)}^{I( 0, t_0, s_0)-1}\|x_{i+1} - x_{i}\|^2\right) \\
    &+ \eta \mathcal  J r^2.
\end{align*}
Here, for the last inequality, we used the fact that $I(0, t_J, s_J) \geq I(0, t_0, s_0)$. Also, we assumed {\color{blue}$c_\eta  \leq 1/8$}. 

Similarly, we can show that
\begin{align*}
    &f(x_{I(k_0, t_0, s_0)}') - f(x_{I(k_0, t_0, s_0)+J}') - \frac{1}{8\eta} \sum_{i=I(k_0,t_0, s_0)}^{I(k_0, t_0, s_0)+J-1}\|x_{i+1}' - x_i'\|^2 \\
    \leq&\ f(x_{I(k_0, t_0, s_0)}') - f(x_{I(k_0, t_0, s_0)+\mathcal J}') \\
    &+\frac{c_\eta}{\eta}\left(\frac{\mathcal J\wedge K }{K} \sum_{i=I(0, t_0, s_0)}^{I(k_0, t_0, s_0)-1} \|x_{i+1} - x_i\|^2 + \frac{\mathcal J\wedge KT}{KT}\sum_{i=I(0, 0, s_0)}^{I( 0, t_0, s_0)-1}\|x_{i+1} - x_{i}\|^2\right) \\
    &+ \eta \mathcal  J r^2.
\end{align*}
Therefore, we get
\begin{align*}
    &\mathrm{max}\left\{\|x_{I(k_0, t_0, s_0)+J} - x_{I(k_0, t_0, s_0)}\|^2, \|x_{I(k_0, t_0, s_0)+J}' - x_{I(k_0, t_0, s_0)}'\|^2\right\} \\
    \leq&\ 8\eta J \Biggl\{-\mathrm{min}\left\{f(x_{I(k_0, t_0, s_0)+\mathcal J}) - f(x_{I(k_0, t_0 s_0)}), f(x_{I(k_0, t_0, s_0)+\mathcal J}') - f(x_{I(k_0, t_0 s_0)}')\right\} \Biggr. \\
    &+ \Biggl. \frac{2c_\eta}{\eta}\left(\frac{\mathcal J\wedge K }{K} \sum_{i=I(0, t_0, s_0)}^{I(k_0, t_0, s_0)-1} \|x_{i+1} - x_i\|^2 + \frac{\mathcal J\wedge KT}{KT}\sum_{i=I(0, 0, s_0)}^{I( 0, t_0, s_0)-1}\|x_{i+1} - x_{i}\|^2\right) + 2\eta \mathcal  J r^2 \Biggr\}
\end{align*}
for every $J \in [\mathcal J]$.
Now, suppose that 
\begin{align}
    &\mathrm{min}\left\{f(x_{I(k_0, t_0, s_0)+\mathcal J}) - f(x_{I(k_0, t_0 s_0)}), f(x_{I(k_0, t_0, s_0)+\mathcal J}') - f(x_{I(k_0, t_0 s_0)}')\right\} \notag \\ >&\  - \mathcal F + \frac{2c_\eta}{\eta}\left(\frac{\mathcal J\wedge K }{K} \sum_{i=I(0, t_0, s_0)}^{I(k_0, t_0, s_0)-1} \|x_{i+1} - x_i\|^2 + \frac{\mathcal J\wedge KT}{KT}\sum_{i=I(0, 0, s_0)}^{I( 0, t_0, s_0)-1}\|x_{i+1} - x_{i}\|^2\right). \label{app: ineq: coupling_sequence_min}
\end{align}

Then, using  (\ref{app: ineq: coupling_sequence_min}), we obtain 
\begin{align*}
    &\mathrm{max}\left\{\|x_{I(k_0, t_0, s_0)+J} - x_{I(k_0, t_0, s_0)}\|^2, \|x_{I(k_0, t_0, s_0)+J}' - x_{I(k_0, t_0, s_0)}'\|^2\right\} \\
    \leq&\ 8\eta \mathcal J(\mathcal F + 2\eta \mathcal Jr^2).
\end{align*}
This finishes the proof. 
\end{proof}

We fix $k_0 \in [K-1]$, $t_0 \in [T-1]$, $s_0 \in [S-1]$ and $\mathcal J_{I(k_0, t_0, s_0)} \in \mathbb{N}$. Let  $\mathcal F_{I(k_0, t_0, s_0)} := c_\mathcal F \eta \mathcal J_{I(k_0, t_0, s_0)} r^2$. From this definition, (\ref{ineq: f_diff}) immediately implies that
\begin{align}
    &f(x_{I(k_0, t_0, s_0)+J})- f(x_{I(k_0, t_0, s_0)}) \notag \\
    \leq&\  \frac{c_\eta}{\eta}\left(\frac{J\wedge K }{K} \sum_{i=I(0, t_0, s_0)}^{I(k_0, t_0, s_0)-1} \|x_{i+1} - x_i\|^2 + \frac{ J\wedge KT}{KT}\sum_{i=I(0, 0, s_0)}^{I( 0, t_0, s_0)-1}\|x_{i+1} - x_{i}\|^2\right)+ \eta \mathcal  J r^2. \notag \\
    =&\  \frac{2}{c_\mathcal F}\mathcal F_{I(k_0, t_0, s_0)} + \frac{c_\eta}{\eta}\left(\frac{\mathcal J_{I(k_0, t_0, s_0)} \wedge K }{K} \sum_{i=I(0, t_0, s_0)}^{I(k_0, t_0, s_0)-1} \|x_{i+1} - x_i\|^2 + \frac{\mathcal J_{I(k_0, t_0, s_0)}\wedge KT}{KT}\sum_{i=I(0, 0, s_0)}^{I( 0, t_0, s_0)-1}\|x_{i+1} - x_{i}\|^2\right) \label{ineq: f_diff_always_bound}
\end{align}
for every $J \in [\mathcal J_{I(k_0, t_0, s_0)}]$ with probability at least $1-3q$. Here, for simplifying the notations, we set $\mathcal F:= \mathcal F_{I(k_0, t_0, s_0)}$ and $\mathcal J := \mathcal J_{I(k_0, t_0, s_0)}$. \par

We want to show the following proposition:
\begin{proposition}\label{prop: unstack}
Under the same conditions as Proposition \ref{prop: f_diff_decrease}, it holds that
\begin{align*}
    & \mathrm{max}\left\{\|x_{I(k_0, t_0, s_0)+J} - x_{I(k_0, t_0, s_0)}\|^2, \|x_{I(k_0, t_0, s_0)+J}' - x_{I(k_0, t_0, s_0)}'\|^2\right\} > 8(c_\mathcal F+2)\eta^2 \mathcal{J}^2r^2
\end{align*}
for some $J \in [\mathcal J]$ with probability at least $1 - 3q$.
\end{proposition}

\begin{proof}
We consider the event $H$ that is an intersection of (\ref{ineq: unstack_prop}), (\ref{ineq: bound_A_hat}) and (\ref{ineq: bound_C_hat}) (derived later), which holds probability at least $1-3q$. From now, the arguments are conditioned on $H$. Observe that $8\eta \mathcal J(\mathcal F + 2\eta\mathcal Jr^2) = 8(c_\mathcal F + 2)\eta^2\mathcal J^2 r^2$.

Suppose that
\begin{align*}
    & \mathrm{max}\left\{\|x_{I(k_0, t_0, s_0)+J} - x_{I(k_0, t_0, s_0)}\|^2, \|x_{I(k_0, t_0, s_0)+J}' - x_{I(k_0, t_0, s_0)}'\|^2\right\} \leq 8(c_\mathcal F+2)\eta^2 \mathcal{J}^2r^2, 
\end{align*}
which implies
\begin{align*}
    & \mathrm{max}\left\{\|x_{I(k_0, t_0, s_0)+J} - x_{I(k_0, t_0, s_0)}\|, \|x_{I(k_0, t_0, s_0)+J}' - x_{I(k_0, t_0, s_0)}'\|\right\} \leq 2\sqrt{2(c_\mathcal F + 2)}\eta \mathcal J r.
\end{align*}
for every $J \in [\mathcal J]$.  
Then, we have
\begin{align*}
    & \mathrm{max}\left\{\|x_{I(k_0, t_0, s_0)+J} - \widetilde x_{I(k_0, t_0, s_0)}\|, \|x_{I(k_0, t_0, s_0)+J}' - \widetilde x_{I(k_0, t_0, s_0)}\|\right\}\\
    \leq&\  \mathrm{max}\left\{\|x_{I(k_0, t_0, s_0)+J} - x_{I(k_0, t_0, s_0)}\|, \|x_{I(k_0, t_0, s_0)+J}' -  x_{I(k_0, t_0, s_0)}'\|\right\} + \eta r \\
    \leq&\ 4\sqrt{c_\mathcal F + 2}\eta \mathcal J r=: U_\Delta .
\end{align*}

We will derive a contradiction. Now, we consider the quantity $\|x_i - x_i'\|^2$ for $i > \widetilde I$. $w_i$ denotes $x_i - x_i'$. Since $\xi_i = \xi_i'$ for $i \neq \hat I$, for $I \geq \widetilde I$, we have that 
\begin{align*}
    w_{I+1} =&\ x_{I+1} - x_{I+1}' \\
    =&\ w_I - \eta(v_I - v_I') - \eta(\xi_I - \xi_I') \\
    =&\ w_I - \eta(\nabla f(x_I) - \nabla f(x_I') + v_I - \nabla f(x_I) - v_I' + \nabla f(x_I'))\\
    =&\ w_I - \eta((\mathcal H +  \Delta_I)w_I +  v_I - \nabla f(x_I) - v_I' + \nabla f(x_I')) \\
    =&\ (1 - \eta \mathcal H)w_I - \eta( \Delta_I w_I + y_I) \\
    =&\ \eta(1-\eta \mathcal H)^{I - \widetilde I} \hat \xi_{\widetilde I}
    - \eta \sum_{i=\widetilde I}^{I} (1 - \eta \mathcal H)^{I - i} ( \Delta_i w_i + y_i), 
\end{align*}
where $\mathcal H := \nabla^2 f(\widetilde x_{I(k_0, t_0, s_0)})$, $ \Delta_i := \int_0^1 (\nabla^2 f(\theta x_i + (1-\theta) x_i') - \mathcal H)d\theta$, $y_i := v_i - \nabla f(x_i) - v_i' + \nabla f(x_i')$ and $\hat \xi_i = \xi_i - \xi_i'$. Let $\lambda := -\lambda_\mathrm{min}(\nabla^2 f(\widetilde x_{I(k_0, t_0, s_0)}) > \sqrt{\rho \varepsilon}$. For the last inequality, we used $\widetilde x_{\widetilde I} = \widetilde x_{\widetilde I}'$. \\
First we give an upper bound of the term $\|\eta(1 - \eta \mathcal H)^{I - \widetilde I}\hat \xi_{\widetilde I}\|$. Since $\hat \xi_{\widetilde I} = \xi_{\widetilde I} - \xi_{\widetilde I}' = 2 \langle \xi_{\widetilde I}, \bm{e}_\mathrm{min}\rangle \bm{e}_\mathrm{min}$, we have
\begin{align*}
     \eta (1 - \eta\mathcal{H})^{I-\widetilde I}\hat \xi_{\widetilde I} = 2\eta(1+\eta \lambda)^{I-\widetilde I}\langle \xi_{\widetilde I}, \bm{e}_\mathrm{min}\rangle \bm e_\mathrm{min}.
\end{align*}
Since $r_0 = 2|\langle \xi_{\widetilde I}, \bm{e}_\mathrm{min}\rangle|$, we have
\begin{align}
    \left\|\eta (1 - \eta\mathcal{H})^{I-\widetilde I}\hat \xi_{\widetilde I}\right\| =  \eta(1+\eta\lambda)^{I - \widetilde I}r_0 =: U_{\hat \xi}(I). \label{ineq: cumulative_noise}
\end{align}

From now, we will show that the following claims hold for $ I \in \{0, \ldots, I(k_0, t_0, s_0)+\mathcal J\}$ with probability at least $1 -  q $ using mathematical induction:
\begin{align*}
    \|w_I\| \leq c_\mathrm{upper}^{(w)} \cdot\eta(1+\eta\lambda)^{I-\widetilde I}r_0=: U_{w}(I)
\end{align*}
for  $c_\mathrm{upper}^{(w)} = \widetilde \Theta(1) > 0$, and
\begin{align*}
    \|y_I\| \leq c_\mathrm{upper}^{(y)} \cdot \eta^2\lambda \left(L+ \frac{ \sqrt{K}L}{\sqrt{b}} + K\zeta + \frac{\sqrt{KT}L}{\sqrt{Pb}}\right)(1 + \eta\lambda)^{I - \widetilde I} r_0  =: U_{y}(I)
\end{align*}
for some $c_\mathrm{upper}^{(y)} = \widetilde \Theta(1) > 0$. Observe that $U_\xi$, $U_w$ and $U_y$ are monotonically increasing with respect to $I$ for $I \geq \widetilde I)$. First we check the case $I \in \{0, \ldots, \widetilde I\}$. In this case, the both claims trivially holds from the definition of $\{x_i'\}_{i=0}^{KTS-1}$ because $w_{i} = y_{i} = 0$ for $i \leq \widetilde I$. Suppose that the two claims hold for the cases $\{0, \ldots, I\}$ with $I \geq \widetilde I$. We want to show that the two claims also hold for the case $I+1 > \widetilde I$. 
\begin{align*}
    \|w_{I+1}\| \leq&\ \eta \sum_{i=\widetilde I}^I (1+\eta\lambda)^{I-i}\| \Delta_i w_i\| + \eta \sum_{i=\widetilde I}^I (1+\eta\lambda)^{I-i}\|y_i\| + U_{\hat \xi}(I+1).
\end{align*}
Here we used inequality (\ref{ineq: cumulative_noise}). 
Observe that
\begin{align*}
    \| \Delta_i w_i\| \leq&\ \| \Delta_i\|\|w_i\| \\
    \leq&\ \| \Delta_i\| U_w(i) \\
    \leq&\ \| \Delta_i\| U_w(I+1)
\end{align*}
and 
\begin{align*}
    \| \Delta_i\| \leq&\ \rho \int_0^1 \|\theta x_i + (1 - \theta) x_i' - \widetilde x_{I(k_0, t_0, s_0)}\|d\theta \\
    \leq&\  \rho \mathrm{max}\{\|x_i - \widetilde x_{I(k_0, t_0, s_0)}\|, \|x_i' - \widetilde x_{I(k_0, t_0, s_0)}'\|\} \\
    \leq&\ \rho U_\Delta.
\end{align*}
Hence, we get
\begin{align}
    &\eta \sum_{i=\widetilde I}^I (1+\eta\lambda)^{I-i}\| \Delta_i w_i\| \notag \\
    \leq&\ \eta (I-\widetilde I) \rho U_\Delta  U_w(I+1) \notag \\
    \leq&\ \eta \rho  \mathcal J U_\Delta   U_w(I+1) \label{ineq: delta_w_sum} 
\end{align}
Similarly, from the inductive assumption on $\|y_i\|$,
\begin{align}
    &\eta \sum_{i=\widetilde I}^I (1+\eta\lambda)^{I-i}\|y_i\| \notag \\
    \leq&\ c_\mathrm{upper}^{(y)} \eta\cdot \eta\lambda \mathcal J \cdot \eta \left(L+\frac{ \sqrt{K}L}{\sqrt{b}} + K\zeta + \frac{\sqrt{KT}L}{\sqrt{Pb}}\right)(1 + \eta\lambda)^{I - \widetilde I} r_0 \notag \\
    \leq&\ \left(\frac{c_\mathrm{upper}^{(y)}\eta \lambda \mathcal J\left(L+\frac{ \sqrt{K}L}{\sqrt{b}} + K\zeta + \frac{\sqrt{KT}L}{\sqrt{Pb}}\right)}{c_\mathrm{upper}^{(w)}}\right) U_w(I). \label{ineq: y_sum}
\end{align}

These results imply
\begin{align*}
    \|w_{I+1}\| \leq&\ \eta \rho \mathcal J U_\Delta  U_w(I+1) +\left( \frac{c_\mathrm{upper}^{(y)}\eta \lambda \mathcal J\cdot\eta \left(L+\frac{ \sqrt{K}L}{\sqrt{b}} + K\zeta + \frac{\sqrt{KT}L}{\sqrt{Pb}}\right)}{c_\mathrm{upper}^{(w)}}\right)U_w(I) + U_{\hat \xi}(I+1) \\
    \leq&\ \left(\eta \rho \mathcal J U_\Delta  +\frac{c_\mathrm{upper}^{(y)}\eta \lambda \mathcal J\cdot \eta \left(L+ \frac{ \sqrt{K}L}{\sqrt{b}} + K\zeta + \frac{\sqrt{KT}L}{\sqrt{Pb}}\right)}{c_\mathrm{upper}^{(w)}} + \frac{1}{c_\mathrm{upper}^{(w)}}\right)U_w(I+1).
\end{align*}
Here, we again used the monotonicity of $U_w(i)$ with respect to $i$. Now, we define 
{\color{blue}$\mathcal J := \mathcal J_{I(k_0, t_0, s_0)} := c_\mathcal{J}/(\eta\lambda)$} ($\leq c_\mathcal J/(\eta \sqrt{\rho \varepsilon})$) for some  $c_\mathcal{J} = \widetilde \Theta (1) \geq 2$, which does not depend on index $I(k_0, t_0, s_0)$ and will be determined later. 
Also, we set {\color{blue}$c_\mathrm{upper}^{(w)} \geq 3 $ and $c_\mathrm{upper}^{(y)} := c_\mathrm{upper}^{(w)}$}. These definitions with appropriate {\color{blue}$\eta \leq  1/(c_\mathcal J(L + \sqrt{K}L/\sqrt{b} + K\zeta + \sqrt{KT}L/\sqrt{Pb}))\}\times 1/( 6c_\mathrm{upper}^{(w)}) = \widetilde \Theta(1/L \wedge \sqrt{b/K}/L) \wedge 1/(K\zeta) \wedge \sqrt{Pb}/(\sqrt{KT}L)$} and {\color{blue} $c_r \leq 1/(24\sqrt{c_\mathcal F + 2}c_\mathcal J^2 c_\mathrm{upper}^{(w)})$} give
\begin{align}
    \eta \rho \mathcal J U_\Delta  \leq 4 c_r\sqrt{c_\mathcal F + 4} \times \eta^2\mathcal J^2 \rho \varepsilon \leq  \frac{1}{6c_\mathrm{upper}^{(w)}} \leq \frac{1}{18} \label{ineq: delta_w_sum_coef}
\end{align}
and 
\begin{align}
    \frac{c_\mathrm{upper}^{(y)}\eta \lambda \mathcal J \cdot \eta \left(L+ \frac{ \sqrt{K}L}{\sqrt{b}} + K\zeta + \frac{\sqrt{KT}L}{\sqrt{Pb}}\right)}{c_\mathrm{upper}^{(w)}} \leq \frac{1}{6c_\mathrm{upper}^{(w)}} \leq \frac{1}{18}. \label{ineq: y_sum_coef}
\end{align}
Hence, we obtain 
\begin{align*}
    \|w_{I+1}\| \leq \frac{4}{9}U_w(I+1) \leq U_w(I+1).
\end{align*}

Next, we consider the quantity $\|y_{I+1}\|$. Let $k, t, s$ be $I+1 = I(k, t, s)$. We define 
\begin{align*}
\begin{cases}
    \alpha_{I(\kappa, t, s)} :=& g_{I(\kappa, t, s)} - g_{I(\kappa, t, s)}^{\mathrm{ref}} + \nabla f_{p_{t, s}}(x_{I(\kappa-1, t, s)}) - \nabla f_{p_{t, s}}(x_{I(\kappa, t, s)}), \\
    \beta_{I(\kappa, t, s)} :=& \nabla f_{p_{t, s}}(x_{I(\kappa, t, s)}) - \nabla f_{p_{t, s}}(x_{I(\kappa-1, t, s)}) + \nabla f(x_{I(\kappa-1, t, s)}) - \nabla f(x_{I(\kappa, t, s)}), \\
    \gamma_{I(0, \tau, s)} :=&  \frac{1}{P}\sum_{p=1}^P (g_{I(0, \tau, s)}^{(p)} - g_{I(0, \tau, s)}^{(p), \mathrm{ref}} + \nabla f(x_{I(0, \tau-1, s)}) - \nabla f(x_{I(0, \tau, s)}).
\end{cases}
\end{align*}
Similarly, we define
\begin{align*}
\begin{cases}
    \alpha_{I(\kappa, t, s)}' :=& g_{I(\kappa, t, s)}' - (g_{I(\kappa, t, s)}^{\mathrm{ref}})' + \nabla f_{p_{t, s}}(x_{I(\kappa-1, t, s)}') - \nabla f_{p_{t, s}}(x_{I(\kappa, t, s)}'), \\
    \beta_{I(\kappa, t, s)}':=& \nabla f_{p_{t, s}}(x_{I(\kappa, t, s)}') - \nabla f_{p_{t, s}}(x_{I(\kappa-1, t, s)}') + \nabla f(x_{I(\kappa-1, t, s)}') - \nabla f(x_{I(\kappa, t, s)}'), \\
    \gamma_{I(0, \tau, s)}' :=&  \frac{1}{P}\sum_{p=1}^P ((g_{I(0, \tau, s)}^{(p)})' - (g_{I(0, \tau, s)}^{(p), \mathrm{ref}})' + \nabla f(x_{I(0, \tau-1, s)}') - \nabla f(x_{I(0, \tau, s)}')
\end{cases}
\end{align*}
that are associated with sequence $\{x_i'\}_{i=I(k_0, t_0, s_0)}^\infty$.
Let $\hat{\alpha}_{I(\kappa, t, s)} = \alpha_{I(\kappa, t, s)} - \alpha_{I(\kappa, t, s)}'$, $\hat{\beta}_{I(\kappa, t, s)} = \beta_{I(\kappa, t, s)} - \beta_{I(\kappa, t, s)}'$ and $\hat{\gamma}_{I(\kappa, t, s)} = \gamma_{I(\kappa, t, s)} - \gamma_{I(\kappa, t, s)}'$. Then we further define
\begin{align*}
    \begin{cases}
        \hat A_{I(k, t, s)} :=& \sum_{\kappa=0}^{k-1}\hat \alpha_{I(\kappa+1, t, s)}, \\
        \hat B_{I(k, t, s)} :=& \sum_{\kappa=0}^{k-1}\hat \beta_{I(\kappa+1, t, s)}, \\
        \hat C_{I(0, t, s)} :=& \sum_{\tau=0}^{t-1}\hat \gamma_{I(0, \tau+1, s)}
    \end{cases}
\end{align*}
These definitions give 
\begin{align*}
    y_{I+1} =&\ v_{I+1} - \nabla f(x_{I+1}) - v_{I+1}' + \nabla f(x_{I+1}') \\ 
    =&\ \hat A_{I(k, t, s)} + \hat B_{I(k, t, s)} + \hat C_{I(0, t, s)}\\
    &+ v_{I(0, 0, s)} - \nabla f(x_{I(0, 0, s)}) - v_{I(0, 0, s)}' + \nabla f(x_{I(0, 0, s)}')
\end{align*}

This implies
\begin{align*}
    \|y_{I+1}\| =&\ \|v_{I+1} - \nabla f(x_{I+1}) - v_{I+1}' + \nabla f(x_{I+1}')\| \\ 
    \leq&\ \left\|\hat A_{I(k, t, s)}\right\| + \left\|\hat B_{I(k, t, s)}\right\| + \left\|\hat C_{I(0, t, s)}\right\|.
\end{align*}
Here, we used the fact that $v_{I(0, 0, s)} - \nabla f(x_{I(0, 0, s)}) - v_{I(0, 0, s)}' + \nabla f(x_{I(0, 0, s)}') = 0$. 
\subsection*{Bounding $\|\hat A_{I(k, t, s)}\|$}
Observe that $\hat \alpha_I(\kappa, t, s)$ satisfies \begin{align*}
    \mathbb{E}[\hat \alpha_I(\kappa, t, s)\mid \mathcal F_{I(\kappa-1, t, s)}] = 0.
\end{align*}
Let 
\begin{align*}
    &\hat u_{l, I(\kappa, t, s)}^{(\alpha)} \\
    :=&\  \nabla \ell(x_{I(\kappa, t, s)}, z_{l, I(\kappa, t, s)}) - \nabla \ell(x_{I(\kappa, t, s)}', z_{l, I(\kappa, t, s)}) - (\nabla \ell(x_{I(\kappa-1, t, s)}, z_{l, I(\kappa, t, s)}) - \nabla \ell(x_{I(\kappa-1, t, s)}', z_{l, I(\kappa, t, s)}) \\
    &+ (\nabla f_{p_{t, s}}(x_{I(\kappa-1, t, s)}) - \nabla f_{p_{t, s}}(x_{I(\kappa-1, t, s)}')) - (\nabla f_{p_{t, s}}(x_{I(\kappa, t, s)}) - \nabla f_{p_{t, s}}(x_{I(\kappa, t, s)}')). 
\end{align*}
Note that $\mathbb{E}[\hat u_{l, I(\kappa, t, s)}^{(\alpha)}| \mathcal F_{I(\kappa, t, s)-1}] = 0$ and $\hat \alpha_{I(\kappa, t, s)} = (1/b)\sum_{l=1}^{b'} \hat u_{l, I(\kappa, t, s)}^{(\alpha)}$.
Observe that 
\begin{align*}
    &\|\hat u_{l, I(\kappa, t, s)}^{(\alpha)}\| \\
    =&\ \|\nabla \ell(x_{I(\kappa, t, s)}, z_{l, I(\kappa, t, s)}) - \nabla \ell(x_{I(\kappa, t, s)}', z_{l, I(\kappa, t, s)}) - (\nabla \ell(x_{I(\kappa-1, t, s)}, z_{l, I(\kappa, t, s)}) - \nabla \ell(x_{I(\kappa-1, t, s)}', z_{l, I(\kappa, t, s)})) \\
    &+ (\nabla f_{p_{t, s}}(x_{I(\kappa-1, t, s)}) - \nabla f_{p_{t, s}}(x_{I(\kappa-1, t, s)}')) - (\nabla f_{p_{t, s}}(x_{I(\kappa, t, s)}) - \nabla f_{p_{t, s}}(x_{I(\kappa, t, s)}'))\| \\
    =&\ \left\|\int_0^1 \nabla^2 \ell(\theta x_{I(\kappa, t, s)} + (1-\theta)x_{I(\kappa, t, s)}', z_{l, I(\kappa, t, s)})d\theta(x_
    {I(\kappa, t, s)} - x_{I(\kappa, t, s)}')\right. \\
    &- \int_0^1 \nabla^2 \ell(\theta x_{I(\kappa-1, t, s)} + (1-\theta)x_{I(\kappa-1, t, s)}', z_{l, I(\kappa, t, s)})d\theta(x_
    {I(\kappa-1, t, s)} - x_{I(\kappa-1, t, s)}') \\
    &+\int_0^1 \nabla^2 f_{p_{t, s}}(\theta x_{I(\kappa, t, s)} + (1-\theta)x_{I(\kappa-1, t, s)}')d\theta(x_
    {I(\kappa, t, s)} - x_{I(\kappa, t, s)}') \\
    &- \left. \int_0^1 \nabla^2 f_{p_{t, s}}(\theta x_{I(\kappa-1, t, s)} + (1-\theta)x_{I(\kappa-1, t, s)}')d\theta(x_
    {I(\kappa-1, t, s)} - x_{I(\kappa-1, t, s)}')\right\| \\
    =&\ \|\mathcal H_{z_{l, I(\kappa, t, s)}}w_{I(\kappa, t, s)} +  \Delta_{ z_{l, I(\kappa, t, s)}, I(\kappa, t, s)}w_{I(\kappa, t, s)} - (\mathcal H_{z_{l, I(\kappa, t, s)}}w_{I(\kappa-1, t, s)} +  \Delta_{z_{l, I(\kappa, t, s)}, I(\kappa-1, t, s)}w_{I(\kappa-1, t, s)}) \\
    &+ \mathcal H_{p_{t, s}} w_{I(\kappa, t, s)} +  \Delta_{p_{t, s}, I(\kappa, t, s)} w_{I(\kappa, t, s)}- (\mathcal H_{p_{t, s}} w_{I(\kappa-1, t, s)} +  \Delta_{p_{t, s}, I(\kappa-1, t, s)}w_{I(\kappa-1, t, s)})\|\\ 
    \leq&\ \|(\mathcal H_{z_{l, I(\kappa, t, s)}} - \mathcal H_{p_{t, s}})(w_{I(\kappa, t, s)} - w_{I(\kappa-1, t, s)})\| \\
    &+ \|( \Delta_{I(\kappa, t, s), z_{l, I(\kappa, t, s)}} -  \Delta_{I(\kappa, t, s)})w_{I(\kappa, t, s)}\| + \|( \Delta_{I(\kappa-1, t, s), z_{l, I(\kappa, t, s)}} -  \Delta_{I(\kappa-1, t, s)})w_{I(\kappa-1, t, s)}\| \\
    \leq&\ 2L\|w_{I(\kappa, t, s)} - w_{I(\kappa-1, t, s)}\| \\
    &+ 2\rho \mathrm{max}\{\|x_{I(\kappa, t, s)} - \widetilde x_{I(k_0, t_0, s_0)}\|, \|x_{I(\kappa, t, s)}' - \widetilde x_{I(k_0, t_0, s_0)}\|, \\
    &\hspace{4.6em}\|x_{I(\kappa-1, t, s)} - \widetilde x_{I(k_0, t_0, s_0)}\|, \|x_{I(\kappa-1, t, s)}' - \widetilde x_{I(k_0, t_0, s_0)}\|\}(\|w_{I(\kappa, t, s)}\| + \|w_{I(\kappa-1, t, s)}\|) \\
    \leq&\ 2L\|w_{I(\kappa, t, s)} - w_{I(\kappa-1, t, s)}\| + 4\rho U_\Delta  U_w(I+1).
\end{align*}
Here, $\mathcal H_{z} := \nabla^2 \ell(\widetilde x_{I(k_0, t_0, s_0)}, z)$, $\mathcal H_{p_{t, s}} := \nabla^2 f_{p_{t, s}}(\widetilde x_{I(k_0, t_0, s_0)})$, $\Delta_{z, I(\kappa, t, s)} := \int_0^1 (\nabla^2 \ell(\theta x_{I(\kappa, t, s)} + (1 - \theta)x_{I(\kappa, t, s)}', z) - \mathcal H_{z})d\theta$ and $\Delta_{p_{t, s}, I(\kappa, t, s)} := \int_0^1 (\nabla^2 f_{p_{t, s}}(\theta x_{I(\kappa, t, s)} + (1 - \theta)x_{I(\kappa, t, s)}') - \mathcal H_{p_{t, s}})d\theta$.  \\
We define 
\begin{align*}
    \hat \sigma_{I(\kappa, t, s)}^{(\alpha)} :=&\ 2L\|w_{I(\kappa, t, s)} - w_{I(\kappa-1, t, s)}\| +4\rho U_\Delta  U_w(I+1).
\end{align*}
Here, for the last inequality, we used the  inductive assumption on $\|w_{I(\kappa, t, s)}\|$ for $I(\kappa, t, s) \leq I(k-1, t, s)$ and the proven bound for $\|w_{I(k, t, s)}\|$. Also, we used the simple fact that $(1+\eta\lambda)^{I(\kappa, t, s) - I(k_0, t_0, s_0)} \leq (1+\eta\lambda)^{I+1 - I(k_0, t_0, s_0)}$ 
Hence, we have
\begin{align*}
    \mathbb{P}(\|\hat u_{l, I(\kappa, t, s)}^{(\alpha)}\|\geq s \mid \mathfrak F_{I(\kappa-1, t, s)}) \leq 2e^{-\frac{s^2}{2\left(\hat \sigma_{I(\kappa, t, s)}^{(\alpha)}\right)^2}}
\end{align*} 
for every $s \in \mathbb{R}$ and $\kappa \in [k]$. Also note that $\{\hat u_{l, I(\kappa, t, s)}^{(\alpha)}\}_{l=1}^{b_\kappa}$ is i.i.d. sequence conditioned on $\mathfrak F_{I(\kappa-1, t, s)}$.  Also note that $\|\hat \alpha_{I(\kappa, t, s)}\| \leq 8 G$ almost surely from Assumption \ref{assump: bounded_loss_gradient}. From these results, we can use Lemma \ref{lem: martingale_concentration_conditioned} with $A = 8kG$ and $a = \widetilde \varepsilon'$ ($\widetilde \varepsilon'$ is some positive number and will be defined later) and get
\begin{align}
    \left\|\hat A_{I(k, t, s)}\right\| \leq c\sqrt{\left(\left(\sum_{\kappa=0}^{k-1}\frac{1}{b_{\kappa+1}} \left(\hat \sigma_{I(\kappa+1, t, s)}^{(\alpha)}\right)^2\right)+\widetilde \varepsilon'\right)\left(\mathrm{log}\frac{2KTSd}{ q }+\mathrm{log}\mathrm{log}\frac{8KG}{\widetilde \varepsilon'}\right)} \label{ineq: bound_A_hat}
\end{align}
for every $k \in [K]\cup\{0\}$, $t\in [T-1]\cup\{0\}$ and $s \in [S]\cup\{0\}$ with probability at least $1 -  q $ for some constant $c > 0$. Note that this event always holds under $H$. 

\subsection*{Bounding $\|\hat B_{I(k, t, s)}\|$}
Observe that
\begin{align*}
    \hat B_{I(k, t, s)} =&\ \nabla f_{p_{t, s}}(x_{I(k, t, s)}) - \nabla f_{p_{t, s}}(x_{I(0, t, s)}) + \nabla f(x_{I(0, t, s)}) - \nabla f(x_{I(k, t, s)}) \\
    &+ \nabla f_{p_{t, s}}(x_{I(k, t, s)}') - \nabla f_{p_{t, s}}(x_{I(0, t, s)}') + \nabla f(x_{I(0, t, s)}') - \nabla f(x_{I(k, t, s)}') \\
    =&\ \int_0^1 \nabla^2 f_{p_{t, s}}(\theta x_{I(k, t, s)} + (1-\theta)x_{I(k, t, s)}')d\theta (x_{I(k, t, s)} - x_{I(k, t, s)}') \\
    &- \int_0^1 \nabla^2 f_{p_{t, s}}(\theta x_{I(0, t, s)} + (1-\theta)x_{I(0, t, s)}')d\theta (x_{I(0, t, s)} - x_{I(0, t, s)}') \\
    &+ \int_0^1 \nabla^2 f(\theta x_{I(k, t, s)} + (1-\theta)x_{I(k, t, s)}')d\theta (x_{I(k, t, s)} - x_{I(k, t, s)}') \\
    &- \int_0^1 \nabla^2 f(\theta x_{I(0, t, s)} + (1-\theta)x_{I(0, t, s)}')d\theta (x_{I(0, t, s)} - x_{I(0, t, s)}') \\
    =&\ (\mathcal H_{p_{t, s}} +  \Delta_{p_{t, s}, I(\kappa, t, s)})w_{I(k, t, s)} - (\mathcal H_{p_{t, s}} +  \Delta_{p_{t, s}, I(0, t, s)}) w_{I(0, t, s)}\\
    &+ (\mathcal H +  \Delta_{I(0, t, s)})w_{I(0, t, s)} - (\mathcal H +  \Delta_{I(k, t, s)})w_{I(k, t, s)} \\
    =&\ (\mathcal H_{p_{t, s}} - \mathcal H)(w_{I(k, t, s)} - w_{I(0, t, s)}) \\
    &+ ( \Delta_{I(k, t, s), p_{t, s}} -  \Delta_{I(k, t, s)})w_{I(k, t, s)} - ( \Delta_{I(0, t, s), p_{t, s}} -  \Delta_{I(0, t, s)})w_{I(0, t, s)}.
\end{align*}
This implies that
\begin{align*}
    \left\|\hat B_{I(k, t, s)}\right\| \leq&\  \zeta \|w_{I(k, t, s)} - w_{I(0, t, s)}\| + 4\rho U_\Delta  U_w(I+1) \\
    \leq&\ \zeta \sum_{\kappa=0}^{k-1}\|w_{I(\kappa+1, t, s)} - w_{I(\kappa, t, s)}\| + 4\rho U_\Delta  U_w(I+1).
\end{align*}

\subsection*{Bounding $\|\hat C_{I(0, t, s)}\|$}
The argument is similar to the case of $\|\hat A_{I(k, t, s)}\|$. From Lemma \ref{lem: martingale_concentration_conditioned}, the third term $\|\hat C_{I(0, t, s)}\|$ can be bounded as
\begin{align}
    \left\|\hat C_{I(0, t, s)}\right\| \leq \frac{c}{\sqrt{PKb}}\sqrt{\left(\left(\sum_{\tau=0}^{t-1} 
    \left(\hat \sigma_{I(0, \tau+1, s)}^{(\gamma)}\right)^2\right)+\widetilde \varepsilon'\right)\left(\mathrm{log}\frac{2KTSd}{ q }+\mathrm{log}\mathrm{log}\frac{8TG}{\widetilde \varepsilon'}\right)} \label{ineq: bound_C_hat}
\end{align}
for every $t \in [T-1]\cup\{0\}$ and $s \in [S-1]\cup\{0\}$ with probability at least $1 -  q$,  where
\begin{align*}
    \sigma_{I(0, \tau, s)}^{(\gamma)} :=&\ 2L\|w_{I(0, \tau, s)} - w_{I(0, \tau-1, s)}\| +  4\rho U_\Delta  U_w(I+1).
\end{align*} 
Here, we used the facts that $\{g_{I(0, \tau, s)}^{(p)} - g_{I(0, \tau, s)}^{(p)\mathrm{ref}} + \nabla f_p(x_{I(0, \tau-1, s)}) - \nabla f_p(x_{I(0, \tau, s)})\}_{p=1}^P$ has mean zero and each of them is constructed from $Kb$ i.i.d. data samples, and $\{(g_{I(0, \tau, s)}^{(p)})' - (g_{I(0, \tau, s)}^{(p)\mathrm{ref}})' + \nabla f_p(x_{I(0, \tau-1, s)}') - \nabla f_p(x_{I(0, \tau, s)}')\}_{p=1}^P$ possesses the same property. 

Hence, we have
\begin{align*}
    &\|y_{I+1}\| \\
    =&\ \|v_{I+1} - \nabla f(x_{I+1}) - v_{I+1}' + \nabla f(x_{I+1}')\| \\ 
    \leq&\ \left\|\hat A_{I(k, t, s)}\right\| + \left\|\hat B_{I(k, t, s)}\right\| + \left\|\hat C_{I(0, t, s)}\right\| \\
    \leq&\ \left\{ c\sqrt{8L^2\sum_{\kappa=0}^{k-1}\frac{1}{b_{\kappa+1}} \|w_{I(\kappa+1, t, s)} - w_{I(\kappa, t, s)}\|^2 + 32K \rho^2 U_\Delta ^2 U_w(I+1)^2+\widetilde \varepsilon'} \right. \\
    &+\zeta \sum_{\kappa=0}^{k-1} \|w_{I(\kappa+1, t, s)} - w_{I(\kappa, t, s)}\| + 4 \rho  U_\Delta  U_w(I+1) \\
    &+ \left. \frac{c}{\sqrt{PKb}}\sqrt{8L^2 \sum_{\tau=0}^{t-1} \|w_{I(0, \tau+1, s)} - w_{I(0, \tau, s)}\|^2 + 32 T\rho^2 U_\Delta ^2 U_w(I+1)^2+\widetilde \varepsilon'}\right\} \\
    &\times \sqrt{ \mathrm{log}\frac{2KTSd}{q } + \mathrm{log}\mathrm{log}\frac{8KTG}{\widetilde \varepsilon'}}
\end{align*}

Now, we further bound the term $\|w_{I(\kappa+1, \tau, s)} - w_{I(\kappa, \tau, s)}\|$.

To do this, it is important to carefully distinguish the three cases: $I(\kappa+1, \tau, s) = \widetilde I+ 1$, $I(\kappa+1, \tau, s) < \widetilde I+1$ and $I(\kappa+1, \tau, s) > \widetilde I+1$. 

For the former case, note that $\|w_{\widetilde I+1} - w_{\widetilde I}\| = \|w_{\widetilde I+1}\| = \eta r_0$.
Also note that $\|w_{I(\kappa+1, \tau, s)} - w_{I(\kappa, \tau, s)}\|=0$ for $I(\kappa+1, \tau, s) < \widetilde I + 1$.

\subsection*{Case I. $1/(\eta \lambda)\leq \sqrt{K}$. }
In this case, $\widetilde I = I(k_0, t_0, s_0)$. 
Suppose that $s = s_0$ and $t = t_0$. Then, since $1/(\eta\lambda) \leq \sqrt{K}$, it holds that
\begin{align*}
    \sum_{\kappa=0}^{k-1}\frac{1}{b_{\kappa+1}}\|w_{I(\kappa+1, t, s)} - w_{I(\kappa, t, s)}\|^2 \leq&\ \frac{1}{b}\sum_{i \in \{I(0, t, s), \ldots, I(k-1, t, s)\}\setminus\{\widetilde I\}}\|w_{i+1} - w_{i}\|^2 + \frac{\eta^2 r_0^2}{b} \\
    \leq&\ \frac{1}{b}\sum_{i \in \{I(0, t, s), \ldots, I(k-1, t, s)\}\setminus\{\widetilde I\}}\|w_{i+1} - w_{i}\|^2 +  \frac{\eta^4\lambda^2K r_0^2}{b}
\end{align*}
and 
\begin{align*}
    \sum_{\kappa=0}^{k-1}\|w_{I(\kappa+1, t, s)} - w_{I(\kappa, t, s)}\| \leq&\ \sum_{i \in \{I(0, t, s), \ldots, I(k-1, t, s)\}\setminus\{\widetilde I\}}\|w_{i+1} - w_{i}\| + \eta r_0 \\
    \leq&\ \sum_{i \in \{I(0, t, s), \ldots, I(k-1, t, s)\}\setminus\{\widetilde I\}}\|w_{i+1} - w_{i}\| + \eta^2\lambda K r_0.
\end{align*}

Also, $\sum_{\tau=0}^{t-1} \|w_{I(0, \tau+1, s)} - w_{I(0, \tau, s)}\|^2 = 0$. 

Next, suppose that $s = s_0$ and $t > t_0$. Since $I(0, t, s) > I(k_0, t_0, s_0)$,  $\|w_{I(k_0+1, t_0, s_0)} - w_{I(k_0, t_0, s_0)}\|$ does not appear in the two terms
\begin{align*}
    \sum_{\kappa=0}^{k-1}\frac{1}{b_{\kappa+1}}\|w_{I(\kappa+1, t, s)} - w_{I(\kappa, t, s)}\|^2 \leq \frac{1}{b}\sum_{\kappa=0}^{k-1}\|w_{I(\kappa+1, t, s)} - w_{I(\kappa, t, s)}\|^2
\end{align*}
and 
\begin{align*}
     \sum_{\kappa=0}^{k-1}\|w_{I(\kappa+1, t, s)} - w_{I(\kappa, t, s)}\|.
\end{align*}

Also, since $I(k_0, t_0, s_0) > I(0, 0, s)$, 
\begin{align*}
    \sum_{\tau=0}^{t-1} \|w_{I(0, \tau+1, s)} - w_{I(0, \tau, s)}\|^2 =&\ \sum_{\tau=\{0, \ldots, t-1\}\setminus\{t_0\}} \|w_{I(0, \tau+1, s)} - w_{I(0, \tau, s)}\|^2 + \|w_{I(0, t_0+1, s)} - w_{I(0, t_0, s)}\|^2 \\ \leq&\ K\sum_{i \in \{I(0, 0, s), \ldots, I(0, t, s)-1\}\setminus \{I(0, t_0, s_0), \ldots, I(0, t_0+1, s_0)-1\}}\|w_{i+1} - w_{i}\|^2 \\
    &+ 2K \sum_{i \in \{I(0, t_0, s), \ldots, I(0, t_0+1, s)-1\}\setminus\{I(k_0, t_0, s_0)\}}\|w_{i+1} - w_{i}\|^2 + 2\eta^2 r_0^2 \\
    \leq&\ 2K\sum_{i \in \{I(0, 0, s), \ldots, I(0, t, s)-1\}\setminus \{I(k_0, t_0, s_0)\}}\|w_{i+1} - w_{i}\|^2 + 2\eta^2r_0^2 \\
    =&\ 2K\sum_{i \in \{I(0, 0, s), \ldots, I(0, t, s)-1\}\setminus\{\widetilde I\}}\|w_{i+1} - w_{i}\|^2 + 2\eta^2r_0^2 
    %\\
    %\leq&\ 2K\sum_{i \in \{I(0, 0, s), \ldots, I(0, t, s)-1\}\setminus\{\widetilde I\}}\|w_{i+1} - w_{i}\|^2 + 2\eta^4\lambda^2 K r_0^2.
\end{align*}

Finally, when $s > s_0$, $\|w_{\widetilde I + 1} - w_{\widetilde I}\|$ never appears in the bound of $\|y_I\|$. 

\subsection*{Case II. $\sqrt{K} < 1/(\eta \lambda)\leq K$. }
In this case, $\widetilde I = I(k_0' , t_0, s_0) - 1$, where $k_0'$ is the minimum number that satisfies $k_0' > k_0$ and $k_0' \equiv 0 \ (\mathrm{mod} \lceil \sqrt{K}\rceil)$. Note that $b_{k_0'} = \lceil \sqrt{K}\rceil b$. 

Suppose that $s = s_0$ and $t = t_0$. Then, since $1/(\eta\lambda) \leq K$, it holds that
\begin{align*}
    \sum_{\kappa=0}^{k-1}\frac{1}{b_{\kappa+1}}\|w_{I(\kappa+1, t, s)} - w_{I(\kappa, t, s)}\|^2 \leq&\ \frac{1}{b}\sum_{i \in \{I(0, t, s), \ldots, I(k-1, t, s)\}\setminus\{\widetilde I\}}\|w_{i+1} - w_{i}\|^2 + \frac{\eta^2 r_0^2}{\sqrt{K}b}
\end{align*}
and 
\begin{align*}
    \sum_{\kappa=0}^{k-1}\|w_{I(\kappa+1, t, s)} - w_{I(\kappa, t, s)}\| \leq&\ \sum_{i \in \{I(0, t, s), \ldots, I(k-1, t, s)\}\setminus\{\widetilde I\}}\|w_{i+1} - w_{i}\| + \eta r_0 \\
    \leq&\ \sum_{i \in \{I(0, t, s), \ldots, I(k-1, t, s)\}\setminus\{\widetilde I\}}\|w_{i+1} - w_{i}\| + \eta^2\lambda K r_0.
\end{align*}

Also, $\sum_{\tau=0}^{t-1} \|w_{I(0, \tau+1, s)} - w_{I(0, \tau, s)}\|^2 = 0$.

Next, suppose that $s = s_0$ and $t > t_0$. Since $I(0, t, s) > I(k_0, t_0, s_0)$,  $\|w_{I(k_0+1, t_0, s_0)} - w_{I(k_0, t_0, s_0)}\|$ does not appear in the two terms
\begin{align*}
    \sum_{\kappa=0}^{k-1}\frac{1}{b_{\kappa+1}}\|w_{I(\kappa+1, t, s)} - w_{I(\kappa, t, s)}\|^2 \leq \frac{1}{b}\sum_{\kappa=0}^{k-1}\|w_{I(\kappa+1, t, s)} - w_{I(\kappa, t, s)}\|^2
\end{align*}
and 
\begin{align*}
     \sum_{\kappa=0}^{k-1}\|w_{I(\kappa+1, t, s)} - w_{I(\kappa, t, s)}\|.
\end{align*}

Also, similar to Case I, since $I(k_0', t_0, s_0) > I(0, 0, s)$, 
\begin{align*}
    \sum_{\tau=0}^{t-1} \|w_{I(0, \tau+1, s)} - w_{I(0, \tau, s)}\|^2 
    %=&\ \sum_{\tau=\{0, \ldots, t-1\}\setminus\{t_0\}} \|w_{I(0, \tau+1, s)} - w_{I(0, \tau, s)}\|^2 + \|w_{I(0, t_0+1, s)} - w_{I(0, t_0, s)}\|^2 \\ \leq&\ K\sum_{i \in \{I(0, 0, s), \ldots, I(0, t, s)-1\}\setminus \{I(0, t_0, s_0), \ldots, I(0, t_0+1, s_0)-1\}}\|w_{i+1} - w_{i}\|^2 \\
    %&+ 2K \sum_{i \in \{I(0, t_0, s), \ldots, I(0, t_0+1, s)-1\}\setminus\{I(k_0, t_0, s_0)\}}\|w_{i+1} - w_{i}\|^2 + 2\eta^2 r_0^2 \\
    \leq&\ 2K\sum_{i \in \{I(0, 0, s), \ldots, I(0, t, s)-1\}\setminus \{I(k_0, t_0, s_0)\}}\|w_{i+1} - w_{i}\|^2 + 2\eta^2r_0^2 \\
    =&\ 2K\sum_{i \in \{I(0, 0, s), \ldots, I(0, t, s)-1\}\setminus\{\widetilde I\}}\|w_{i+1} - w_{i}\|^2 + 2\eta^2r_0^2.
\end{align*}

Finally, when $s > s_0$, $\|w_{\widetilde I + 1} - w_{\widetilde I}\|$ never appears in the bound of $\|y_I\|$. 

\subsection*{Case III. $K < 1/(\eta \lambda) \leq KT$. }
In this case, $\widetilde I = I(0, t_0+1, s_0) - 1$. Since $I+1 = I(k, t, s) > \widetilde I$, if $s = s_0$, then we can see that $t \geq t_0+1 > t_0$. Then, $\|w_{\widetilde I +1} - w_{\widetilde I}\|$ does not appear in the two terms
\begin{align*}
    \sum_{\kappa=0}^{k-1}\frac{1}{b_{\kappa+1}}\|w_{I(\kappa+1, t, s)} - w_{I(\kappa, t, s)}\|^2 \leq \frac{1}{b}\sum_{\kappa=0}^{k-1}\|w_{I(\kappa+1, t, s)} - w_{I(\kappa, t, s)}\|^2
\end{align*}
and 
\begin{align*}
     \sum_{\kappa=0}^{k-1}\|w_{I(\kappa+1, t, s)} - w_{I(\kappa, t, s)}\|.
\end{align*}

Observe that 
\begin{align*}
    \sum_{\tau=0}^{t-1} \|w_{I(0, \tau+1, s)} - w_{I(0, \tau, s)}\|^2 =&\ \sum_{\tau=\{0, \ldots, t-1\}\setminus\{t_0\}} \|w_{I(0, \tau+1, s)} - w_{I(0, \tau, s)}\|^2 + \|w_{I(0, t_0+1, s)} - w_{I(0, t_0, s)}\|^2 \\ 
    \leq&\ K\sum_{i \in \{I(0, 0, s), \ldots, I(0, t, s)-1\}\setminus \{I(0, t_0, s_0), \ldots, I(0, t_0+1, s_0)-1\}}\|w_{i+1} - w_{i}\|^2 \\
    &+ 2K \sum_{i \in \{I(0, t_0, s), \ldots, I(0, t_0+1, s)-2\}}\|w_{i+1} - w_{i}\|^2 + 2\eta^2 r_0^2 \\
    \leq&\ 2K\sum_{i \in \{I(0, 0, s), \ldots, I(0, t, s)-1\}\setminus \{\widetilde I\}}\|w_{i+1} - w_{i}\|^2 + 2\eta^2r_0^2.
\end{align*}

When $s > s_0$, $\|w_{\widetilde I +1} - w_{\widetilde I}\|$ never appears in the bound of $\|y_{I+1}\|$.

\subsection*{Case IV. $KT < 1/(\eta\lambda)$.}
In this case, $\widetilde I = I(0, 0, s_0+1) - 1$. Since $I+1 = I(k, t, s) > \widetilde I$, we know that $s \geq s_0+1 > s_0$. Hence, $\|w_{\widetilde I +1} - w_{\widetilde I}\|$ never appears in the bound of $\|y_{I+1}\|$.

In summary, we have
\begin{align*}
    &\|y_{I+1}\| \\
    \leq&\ \left\{c\sqrt{8L^2\left(\frac{1}{b}\sum_{i \in \{I(0, t, s), \ldots, I(k-1, t, s)\}\setminus\{\widetilde I\}}\|w_{i+1} - w_{i}\|^2 + \frac{\left(\eta^2\lambda^2K + 1/\sqrt{K}\right)\eta^2r_0^2}{b}\right) + 32K \rho^2 U_\Delta ^2 U_w(I+1)^2+\widetilde \varepsilon'} \right. \\
    &+ \zeta \left(\sum_{i \in \{I(0, t, s), \ldots, I(k-1, t, s)\}\setminus\{\widetilde I\}}\|w_{i+1} - w_{i}\| + \eta^2\lambda K r_0\right) + 4 \rho  U_\Delta  U_w(I+1) \\
    &+ \left. \frac{c}{\sqrt{PKb}}\sqrt{8L^2\left(2K \sum_{i \in \{I(0, 0, s), \ldots, I(0, t, s)-1\}\setminus\{\widetilde I\}}\|w_{i+1} - w_{i}\|^2 + 2\eta^2r_0^2\right) + 32 T\rho^2 U_\Delta ^2 U_w(I+1)^2+\widetilde \varepsilon'}\right\} \\
    &\times \sqrt{ \mathrm{log}\frac{2KTSd}{q } + \mathrm{log}\mathrm{log}\frac{8KTG}{\widetilde \varepsilon'}}.
\end{align*}

Now, we bound $\|w_{I(\kappa+1, \tau, s)} - w_{I(\kappa, \tau, s)}\|$ for the case $I(\kappa+1, \tau, s) > \widetilde I + 1$. 
\begin{align*}
    &\|w_{I(\kappa+1, \tau, s)} - w_{I(\kappa, \tau, s)}\| \\
    =&\ \left\|\eta (1 - \eta \mathcal H)^{I(\kappa+1, \tau, s) - \widetilde I}\hat \xi_{\widetilde I} - \eta \sum_{i=\widetilde I}^{I(\kappa, \tau, s)} (1 - \eta \mathcal H)^{I(\kappa, \tau, s) - i} ( \Delta_i w_i + y_i) \right. \\
    &- \left. \eta (1 - \eta \mathcal H)^{I(\kappa, \tau, s) - \widetilde I}\hat \xi_{\widetilde I} + \eta \sum_{i=\widetilde I}^{I(\kappa, \tau, s)-1} (1 - \eta \mathcal H)^{I(\kappa, \tau, s)-1 - i} ( \Delta_i w_i + y_i)\right\| \\
    =&\ \left\| -\eta^2 \mathcal H(1 - \eta \mathcal H)^{I(\kappa, t, s) - \widetilde I}\hat \xi_{\widetilde I} \right. \\
    &+ \left. \eta \sum_{i=\widetilde I}^{I(\kappa, \tau, s)-1} \eta \mathcal H (1 - \eta \mathcal H)^{I(\kappa, \tau, s) - 1 - i}( \Delta_i w_i + y_i) -\eta(\Delta_{I(\kappa, \tau, s)}w_{I(\kappa, \tau, s)} + y_{I(\kappa, \tau, s)})\right\| \\
    \leq&\ \eta\left\| \eta \mathcal H(1 - \eta \mathcal H)^{I(\kappa, t, s) - \widetilde I}\hat \xi_{\widetilde I}\right\| \\
    &+ \eta \sum_{i=\widetilde I}^{I(\kappa, \tau, s)-1} \left\|\eta \mathcal H (1 - \eta \mathcal H)^{I(\kappa, \tau, s) - 1 - i}\right\|\left\| \Delta_i w_i + y_i\right\| + \eta\|\Delta_{I(\kappa, \tau, s)}w_{I(\kappa, \tau, s)} + y_{I(\kappa, \tau, s)}\| \\
    \leq&\ \eta^2 \lambda(1 + \eta \lambda)^{I(\kappa, t, s) - \widetilde I}r_0 \\
    &+ \eta \sum_{i=\widetilde I}^{I(\kappa, \tau, s)-1} \left(\eta \lambda(1 + \eta \lambda)^{I(\kappa, t, s) - 1 - i} + \frac{e}{I(\kappa, t, s) - i}\right)\left\| \Delta_i w_i + y_i\right\| + \eta\|\Delta_{I(\kappa, \tau, s)}w_{I(\kappa, \tau, s)} + y_{I(\kappa, \tau, s)}\|.
\end{align*}
For the second inequality, we used the following two facts:
\begin{align*}
    \left\|\eta \mathcal H(1 - \eta\mathcal H)^J \hat \xi_{\widetilde I}\right\| \leq \eta \lambda (1 + \eta \lambda)^J \|\hat \xi_{\widetilde I}\|
\end{align*}
and
\begin{align*}
    \left\|\eta \mathcal H(1 - \eta\mathcal H)^J\right\| \leq \eta \lambda (1 + \eta \lambda)^J + \frac{e}{J+1}
\end{align*}
for $J \in \mathbb{N}\cup\{0\}$. The former inequality holds because $\hat \xi_{\widetilde I} = 2\langle \xi_{\widetilde I}, \bm{e}_\mathrm{min}\rangle \bm{e}_\mathrm{min}$ and $\bm{e}_\mathrm{min}$ is the minimum eigenvector of $\mathcal H$. The latter inequality is the direct result of the from Lemma \ref{lem: op_norm_bound}. 

Then, we further bound the upper bound as follows:
\begin{align*}
    &\|w_{I(\kappa+1, \tau, s)} - w_{I(\kappa, \tau, s)}\| \\
    \leq&\ \eta^2 \lambda(1 + \eta \lambda)^{I(\kappa, t, s) - \widetilde I}r_0 \\
    &+ \eta \sum_{i=\widetilde I}^{I(\kappa, \tau, s)-1} \left(\eta \lambda(1 + \eta \lambda)^{I(\kappa, t, s) - 1 - i} + \frac{e}{I(\kappa, t, s) - i}\right)\left\| \Delta_i w_i + y_i\right\| + \eta\|\Delta_{I(\kappa, \tau, s)}w_{I(\kappa, \tau, s)} + y_{I(\kappa, \tau, s)}\|\\
    \leq&\ \eta^2 \lambda(1 + \eta \lambda)^{I(\kappa, t, s) - \widetilde I}r_0 \\
    &+ 4e(1+\mathrm{log}\mathcal J) \eta\rho U_\Delta  U_w(I)+ 2e(1+\mathrm{log}\mathcal J)\eta(1+\eta \lambda \mathcal J) U_y(I). \\
    \leq&\ 4e(1+\mathrm{log}\mathcal J)\eta\rho U_\Delta  U_w(I) + \left(\frac{1}{c_\mathrm{upper}^{(y)}\left(L+ \frac{ \sqrt{K}L}{\sqrt{b}} + K\zeta + \frac{\sqrt{KT}L}{\sqrt{Pb}}\right)} + 2e(1+\mathrm{log}\mathcal J)\eta(1+\eta \lambda \mathcal J)\right) U_y(I) \\
    =:&\ U_{\hat w}(I). 
\end{align*}

For the first inequality, we used $\| \Delta_i\| \leq \rho U_\Delta $, the inductive assumptions on $\|w_i\|$ and $\|y_i\|$ for $i \leq I(k, t, s)-1$ and $\sum_{i=i_0}^{i'} 1/(i+1 - i_0) \leq 1 + \mathrm{log}(i' + 1- i_0)$ for $i' \geq i_0$. 

Concretely, we computed
\begin{align*}
     &\sum_{i=\widetilde I}^{I(\kappa, \tau, s)-1} \left(\eta \lambda(1 + \eta \lambda)^{I(\kappa, t, s) - 1 - i} + \frac{e}{I(\kappa, t, s) - i}\right) \left\| \Delta_i w_i\right\| \\
     \leq&\ \sum_{i=\widetilde I}^{I(\kappa, \tau, s)-1} \left(\eta \lambda(1 + \eta \lambda)^{I(\kappa, t, s) - 1 - i} + \frac{e}{I(\kappa, t, s) - i}\right) \rho U_\Delta  U_w(i) \\
     \leq&\ \rho U_\Delta (1 + e(1+\mathrm{log}\mathcal J)) U_w(I) \\
     \leq&\ 2e(1+\mathrm{log}\mathcal J) \rho U_\Delta  U_w(I).
\end{align*}
Also, we computed
\begin{align*}
     &\sum_{i=\widetilde I}^{I(\kappa, \tau, s)-1} \left(\eta \lambda(1 + \eta \lambda)^{I(\kappa, t, s) - 1 - i} + \frac{e}{I(\kappa, t, s) - i}\right) \left\|y_i\right\| \\ 
     \leq&\ \sum_{i=\widetilde I}^{I(\kappa, \tau, s)-1} \left(\eta \lambda(1 + \eta \lambda)^{I(\kappa, t, s) - 1 - i} + \frac{e}{I(\kappa, t, s) - i}\right) \\
     &\times \left(c_\mathrm{upper}^{(y)} 
     \eta^2\lambda\left(L+ \frac{ \sqrt{K}L}{\sqrt{b}} + K\zeta + \frac{\sqrt{KT}L}{\sqrt{Pb}}\right) (1 + \eta\lambda)^{i - \widetilde I} r_0\right) \\
     \leq&\ c_\mathrm{upper}^{(y)}\eta^3\lambda^2 \mathcal J \left(L+\frac{ \sqrt{K}L}{\sqrt{b}} + K\zeta + \frac{\sqrt{KT}L}{\sqrt{Pb}}\right)(1 + \eta\lambda)^{I(\kappa, t, s) - 1 - \widetilde I} r_0 \\
     &+ c_\mathrm{upper}^{(y)}e(1+ \mathrm{log}\mathcal J )\eta^2\lambda \left(L+ \frac{ \sqrt{K}L}{\sqrt{b}} + K\zeta + \frac{\sqrt{KT}L}{\sqrt{Pb}}\right)(1 + \eta\lambda)^{I(\kappa, t, s) - \widetilde I} r_0\\
     \leq&\ e(1+\mathrm{log}\mathcal J)(1+\eta \lambda \mathcal J) U_y(I).
\end{align*}
\begin{comment}
Here, for the second inequality, we used the fact that 
\begin{align*}
    \sum_{i=i_0}^{i'} \frac{1}{(i'+1 - i)}\frac{1}{(i+1 - i_0)} \leq \frac{2(1+\mathrm{log}(i'+1 - i_0))}{i'+1 - i_0}
\end{align*}
for $i \geq i_0$
This inequality holds because 
\begin{align*}
    \frac{1}{(i'+1 - i)}\frac{1}{(i+1 - i_0)} = \frac{1}{i+2 - i_0}\left(\frac{1}{(i'+1 - i)} + \frac{1}{(i+1 - i_0)}\right), 
\end{align*}
and the two terms $\sum_{i=i_0}^{i'} 1/(i+1 - i_0)$ and $\sum_{i=i_0}^{i'} 1/(i'+1 - i)$ can be bounded by $1 + \mathrm{log}(i'+1 - i_0)$. \\ 
\end{comment}
Using the bound of $\|w_{I(\kappa+1, \tau, s)} - w_{I(\kappa, \tau, s)}\|$, we get

\begin{align*}
    &\|y_{I+1}\| \\
    \leq&\ \left\{c\sqrt{8L^2\left(\frac{K}{b}U_{\hat w}(I) + \frac{\left(\eta^2\lambda^2K + 1/\sqrt{K}\right)\eta^2r_0^2}{b}\right) + 32K \rho^2 U_\Delta ^2 U_w(I+1)^2+\widetilde \varepsilon'} \right. \\
    &+ \zeta \left(KU_{\hat w}(I) + \eta^2\lambda K r_0\right) + 4 \rho  U_\Delta  U_w(I+1) \\
    &+ \left. \frac{c}{\sqrt{PKb}}\sqrt{8L^2\left(2K^2T U_{\hat w}(I) + 2\eta^2r_0^2\right) + 32 T\rho^2 U_\Delta ^2 U_w(I+1)^2+\widetilde \varepsilon'}\right\} \\
    &\times \sqrt{ \mathrm{log}\frac{2KTSd}{q } + \mathrm{log}\mathrm{log}\frac{8KTG}{\widetilde \varepsilon'}} \\
    \leq&\ \left\{ \left(\frac{2\sqrt{2}c\sqrt{K}L}{\sqrt{b}} + K\zeta + \frac{4c\sqrt{KTL}}{\sqrt{PKb}}\right)U_{\hat w}(I) +\left(\frac{2\sqrt{2}c\eta\lambda \sqrt{K}L}{\sqrt{b}} + \frac{2\sqrt{2}cL}{K^{1/4}\sqrt{b}} + \eta\lambda K\zeta + \frac{4cL}{\sqrt{PKb}}\right)\eta r_0 \right. \\
    &+ \left. \left(\frac{4\sqrt{2}c\sqrt{K}}{\sqrt{b}} + 4 + \frac{4\sqrt{2}c\sqrt{T}}{\sqrt{PKb}}\right)\rho U_\Delta U_w(I+1) + 2c\sqrt{\widetilde \varepsilon'}
    \right\} \\
    &\times \sqrt{ \mathrm{log}\frac{2KTSd}{q } + \mathrm{log}\mathrm{log}\frac{8KTG}{\widetilde \varepsilon'}}    
\end{align*}
Under $\color{blue}b \geq 1/(K^{1/2}\eta^2(L + \sqrt{K}L/\sqrt{b} + K\zeta + \sqrt{KT}L/\sqrt{Pb})^2\rho\varepsilon)$, from $\lambda \geq \sqrt{\rho\varepsilon}$, we have
\begin{align*}
   &\frac{2\sqrt{2}c\eta\lambda \sqrt{K}L}{\sqrt{b}} + \frac{2\sqrt{2}cL}{K^{1/4}\sqrt{b}} + \eta\lambda K\zeta + \frac{4cL}{\sqrt{PKb}} \\ \leq&\ \eta\left(\frac{2\sqrt{2}c\sqrt{K}L}{\sqrt{b}} + 2\sqrt{2}c\left(L + \frac{\sqrt{K}L}{\sqrt{b}} + \frac{\sqrt{KT}L}{\sqrt{Pb}}\right) + K\zeta + 4c\left(L + \frac{\sqrt{K}L}{\sqrt{b}} + \frac{\sqrt{KT}L}{\sqrt{Pb}}\right)\right) \lambda \\
   \leq&\ 12c\eta \left(L + \frac{\sqrt{K}L}{\sqrt{b}} + K\zeta +  \frac{\sqrt{KT}L}{\sqrt{Pb}}\right)\lambda.
\end{align*}

Also, under {\color{blue}$b \geq K$} and {\color{blue}$b \geq T/(PK)$}, we have
\begin{align*}
    \frac{4\sqrt{2}c\sqrt{K}}{\sqrt{b}} + 4 + \frac{4\sqrt{2}c\sqrt{T}}{\sqrt{PKb}} \leq 24c.
\end{align*}

We choose $\widetilde \varepsilon$ such that $\widetilde \varepsilon' \leq  \eta^4L^2\lambda^2r_0^2/(64(\mathrm{log}\frac{2KTSd}{q} + \mathrm{log}\mathrm{log}\frac{8KTG}{\widetilde \varepsilon'})c^2)$. Then, it holds that

\begin{align*}
    \|y_{I+1}\| 
    \leq&\  \left\{ 4c\left(\frac{\sqrt{K}L}{\sqrt{b}} + K\zeta + \frac{\sqrt{KTL}}{\sqrt{PKb}}\right)U_{\hat w}(I) +12c\eta\left(L + \frac{\sqrt{K}L}{\sqrt{b}} + K\zeta +  \frac{\sqrt{KT}L}{\sqrt{Pb}}\right)\eta\lambda r_0 \right. \\
    &+  24c\rho U_\Delta U_w(I+1)
    \Bigg\} \times \sqrt{ \mathrm{log}\frac{2KTSd}{q } + \mathrm{log}\mathrm{log}\frac{8KTG}{\widetilde \varepsilon'}} + 0.25U_y(I+1).
\end{align*}

From the definition of $U_{\hat w}(I)$: 
\begin{align*}
    U_{\hat w}(I) := 4e(1+\mathrm{log}\mathcal J)\eta\rho U_\Delta  U_w(I) + \left(\frac{1}{c_\mathrm{upper}^{(y)}\left(L+ \frac{ \sqrt{K}L}{\sqrt{b}} + K\zeta + \frac{\sqrt{KT}L}{\sqrt{Pb}}\right)} + 2e(1+\mathrm{log}\mathcal J)\eta(1+\eta \lambda \mathcal J)\right) U_y(I),
\end{align*}
we get
\begin{align*}
    \|y_{I+1}\| 
    \leq&\  \left\{ 4c\left(\frac{1}{c_\mathrm{upper}^{(y)}} + 2e(1+\mathrm{log}\mathcal J)\eta\left(\frac{\sqrt{K}L}{\sqrt{b}} + K\zeta + \frac{\sqrt{KTL}}{\sqrt{PKb}}\right)(1 + \eta\lambda\mathcal J) \right)U_{y}(I) \right. \\ &+12c\eta\left(L + \frac{\sqrt{K}L}{\sqrt{b}} + K\zeta +  \frac{\sqrt{KT}L}{\sqrt{Pb}}\right)\eta\lambda r_0  \\
    &+  \left(24c + 16ce(1+\mathrm{log}\mathcal J)\eta\left(\frac{\sqrt{K}L}{\sqrt{b}} + K\zeta + \frac{\sqrt{KTL}}{\sqrt{PKb}}\right) \right)\rho U_\Delta U_w(I+1)
    \Bigg\} \\
    &\times \sqrt{ \mathrm{log}\frac{2KTSd}{q } + \mathrm{log}\mathrm{log}\frac{8KTG}{\widetilde \varepsilon'}} + 0.25 U_y(I+1).
\end{align*}

From the definitions of $\mathcal J$ and $U_\Delta $ with $r = c_r \varepsilon$ with $c_r = \widetilde O(1)$, we have
\begin{align*}
    \rho U_\Delta  U_w(I) \leq&\  \rho U_\Delta  \frac{c_\mathrm{upper}^{(w)}}{c_\mathrm{upper}^{(y)} \eta\lambda \left(L + \frac{ \sqrt{K}L}{\sqrt{b}} + K\zeta + \frac{\sqrt{KT}L}{\sqrt{Pb}}\right)}U_y(I+1) \\
    \leq&\ \frac{4c_r\sqrt{c_\mathcal F+2} c_\mathcal{J}\eta \rho \varepsilon}{\eta\lambda^2}\frac{c_\mathrm{upper}^{(w)}}{c_\mathrm{upper}^{(y)}\eta\left(L +  \frac{ \sqrt{K}L}{\sqrt{b}} + K\zeta + \frac{\sqrt{KT}L}{\sqrt{Pb}}\right) }U_y(I+1) \\
    \leq& \frac{4c_r\sqrt{c_\mathcal F+2} c_\mathcal{J}}{\eta\left(L + \frac{ \sqrt{K}L}{\sqrt{b}} + K\zeta + \frac{\sqrt{KT}L}{\sqrt{Pb}}\right)}U_y(I+1).
\end{align*}
Here, for the last inequality, we used $\lambda \geq \sqrt{\rho \varepsilon}$ and $c_\mathrm{upper}^{(y)} = c_\mathrm{upper}^{(w)}$. 

Therefore, we arrive at
\begin{align*}
    \|y_{I+1}\| 
    \leq&\  \left\{ 4c\left(\frac{1}{c_\mathrm{upper}^{(y)}} + 2e(1+\mathrm{log}\mathcal J)\eta\left(\frac{\sqrt{K}L}{\sqrt{b}} + K\zeta + \frac{\sqrt{KTL}}{\sqrt{PKb}}\right)(1 + \eta\lambda\mathcal J) \right)U_{y}(I) \right. \\ &+\frac{12c}{c_\mathrm{upper}^{(y)}}U_y(I+1) \\
    &+  \left(\frac{96cc_r\sqrt{c_\mathcal F+2} c_\mathcal{J}}{\eta\left(L + \frac{ \sqrt{K}L}{\sqrt{b}} + K\zeta + \frac{\sqrt{KT}L}{\sqrt{Pb}}\right)}
    + 64ce(1+\mathrm{log}\mathcal J)c_r\sqrt{c_\mathcal F+2} c_\mathcal{J}\right)U_y(I+1)
    \Bigg\} \\
    &\times \sqrt{ \mathrm{log}\frac{2KTSd}{q } + \mathrm{log}\mathrm{log}\frac{8KTG}{\widetilde \varepsilon'}} + 0.25U_y(I+1).
\end{align*}

We set {\color{blue} $c_\mathrm{upper}^{(y)} = c_\mathrm{upper}^{(w)} = \mathrm{max}\{3, 48c\sqrt{ \mathrm{log}\frac{2KTSd}{q } + \mathrm{log}\mathrm{log}\frac{8KTG}{\widetilde \varepsilon'}}\} $}.
Then, since $\eta \lambda \mathcal J \leq c_\mathcal J$, if we choose $\eta$ such that {\color{blue}$\eta (L + \sqrt{K}L/\sqrt{b} + K\zeta + \sqrt{KT}L/\sqrt{Pb})) \leq 1/(48c e(1+c_\mathcal J)(1+\mathrm{log}\mathcal J)\sqrt{ \mathrm{log}\frac{2KTSd}{q } + \mathrm{log}\mathrm{log}\frac{8KTG}{\widetilde \varepsilon'}})$}, the first term can be bounded by $0.25 U_y(I+1)$. 

Next, from the definition of $c_\mathrm{upper}^{(y)}$, we can see that the second term is bounded by $0.25 U_y(I+1)$.

Finally, we can choose $\eta$ such that {\color{blue}$\eta (L + \sqrt{K}L/\sqrt{b} + K\zeta + \sqrt{KT}L/\sqrt{Pb})) \geq \widetilde \Theta(\sqrt{c_\eta} + 1/(c_\mathcal J))$}. Then if we appropriately choose {\color{blue}$c_r \leq \widetilde O((\sqrt{c_\eta} + 1/c_\mathcal J)/(\sqrt{c_\mathcal F+2}c_\mathcal J))$}, the third term can be bounded by $0.25 U_y(I+1)$.  Therefore, we conclude that  
\begin{align*}
    \|y_{I+1}\| \leq U_y(I+1).
\end{align*}

This finishes the proof of the mathematical induction. \par

Let $\widetilde {\mathcal J} := \mathcal J - (\widetilde I - I(k_0, t_0, s_0))$.
From (\ref{ineq: cumulative_noise}), (\ref{ineq: delta_w_sum_coef}) and (\ref{ineq: y_sum_coef}), we have
\begin{align*}
    \|w_{I(k_0, t_0, s_0)+\mathcal J}\| =&\   \left\|\eta(1-\eta \mathcal H)^{\widetilde {\mathcal J}} \hat \xi_{\widetilde I}- \eta \sum_{i=\widetilde I}^{\widetilde I + \widetilde {\mathcal J}} (1 - \eta \mathcal H)^{\widetilde I + {\widetilde {\mathcal J}} - i} ( \Delta_i w_i + y_i)\right\| \\
    \geq&\ \|\eta(1-\eta \mathcal H)^{\widetilde{ \mathcal J}} \hat \xi_{\widetilde I}\| \\
    &-  \left\|\eta \sum_{i=\widetilde I}^{\widetilde I + \widetilde {\mathcal J}} (1 - \eta \mathcal H)^{\widetilde I + \widetilde {\mathcal J} - i} \Delta_i w_i\right\| \\
    &- \left\|\eta \sum_{i=\widetilde I}^{\widetilde I + \widetilde {\mathcal J}} (1 - \eta \mathcal H)^{\widetilde I + \widetilde {\mathcal J} - i} y_i\right\| \\
    \geq&\ \eta(1+\eta\lambda)^{\widetilde { \mathcal J}}r_0- \frac{1}{3c_\mathrm{upper}^{(w)}}U_w(I(k_0, t_0, s_0) + \mathcal J) \\
    =&\ \frac{2\eta(1+\eta\lambda)^{\widetilde {\mathcal J}}r_0}{3}.
\end{align*}

Now, we define $c_\mathcal J$ as the minimum positive number that satisfies
{\color{blue}$$c_\mathcal{J} \geq 1 +  2\mathrm{log}(48\sqrt{c_\mathcal F + 2}\mathcal J \sqrt{d}/q). $$}
From (\ref{ineq: unstack_prop}), we can see that
\begin{align*}
    \frac{2\eta(1+\eta\lambda)^{\widetilde {\mathcal{J}}}r_0}{3} \geq 4 U_\Delta .
\end{align*}

This is because we have
\begin{align*}
    \mathrm{log}\left((1+\eta\lambda)^{\widetilde {\mathcal{J}}}\right) =&\ \widetilde {\mathcal J} \mathrm{log}(1+\eta\lambda) \\
    \geq&\ \widetilde {\mathcal J} \left(1 - \frac{1}{1+\eta\lambda}\right) \\
    \geq&\ \frac{\eta \lambda \widetilde {\mathcal J}}{2} \\
    \geq&\ \frac{\eta \lambda(\mathcal J - 1/(\eta\lambda))}{2} \\
    =&\ \frac{c_\mathcal{J}-1}{2} \\
    \geq&\ \mathrm{log}(48\sqrt{c_\mathcal F + 2}\mathcal J\sqrt{d}/q ) 
\end{align*}
and thus  
\begin{align*}
    \frac{2\eta(1+\eta\lambda)^{\widetilde {\mathcal{J}}}r_0}{3} \geq&\ \frac{\eta(1+\eta\lambda)^{\widetilde {\mathcal{J}}} q  r}{3\sqrt{d}} \\
    \geq&\ 4 \times 4\sqrt{c_\mathcal F + 2}\eta \mathcal J r = 4 U_\Delta . 
\end{align*}

Here, the first inequality holds from (\ref{ineq: unstack_prop}). This contradicts with $\|w_{I(k_0, t_0, s_0)+\mathcal J}\| \leq 2U_\Delta $. 

\end{proof}

\subsection*{Proof of Proposition \ref{prop: f_diff_decrease}}

Now, we prove Proposition \ref{prop: f_diff_decrease}. Combining Proposition \ref{prop: improve_or_localize} with Proposition \ref{prop: unstack}, we have
\begin{align*}
    &\mathrm{min}\{f(x_{I(k_0, t_0, s_0)+\mathcal J_{I(k_0, t_0, s_0)}}) - f(x_{I(k_0, t_0 s_0)}), f(x_{I(k_0, t_0, s_0)+\mathcal J_{I(k_0, t_0, s_0)}}') - f(x_{I(k_0, t_0 s_0)}')\} \\
    \leq&\ -  \mathcal F_{I(k_0, t_0, s_0)} \\
    &+ \frac{2c_\eta}{\eta}\left(\frac{\mathcal J_{I(k_0, t_0, s_0)}\wedge K }{K} \sum_{i=I(0, t_0, s_0)}^{I(k_0, t_0, s_0)-1} \|x_{i+1} - x_i\|^2 + \frac{\mathcal J_{I(k_0, t_0, s_0)}\wedge KT}{KT}\sum_{i=I(0, 0, s_0)}^{I( 0, t_0, s_0)-1}\|x_{i+1} - x_{i}\|^2\right).
\end{align*}
with probability at least $1 - 9q $. \par

Finally, since $\{x_i\}_{i=0}^{KTS}$ has the same marginal distribution as $\{x_i'\}_{i=0}^{KTS}$, we conclude that
\begin{align}
    &f(x_{I(k_0, t_0, s_0)+\mathcal J_{I(k_0, t_0, s_0)}}) - f(x_{I(k_0, t_0 s_0)}) \notag\\
    \leq&\ -  \mathcal F_{I(k_0, t_0, s_0)} \notag\\
    &+ \frac{2c_\eta}{\eta}\left(\frac{\mathcal J_{I(k_0, t_0, s_0)}\wedge K }{K} \sum_{i=I(0, t_0, s_0)}^{I(k_0, t_0, s_0)-1} \|x_{i+1} - x_i\|^2 + \frac{\mathcal J_{I(k_0, t_0, s_0)}\wedge KT}{KT}\sum_{i=I(0, 0, s_0)}^{I( 0, t_0, s_0)-1}\|x_{i+1} - x_{i}\|^2\right)\label{ineq: f_diff_escape_saddle}.
\end{align}
with probability at least $1/2 - 9 q /2$. This finishes the proof of Proposition \ref{prop: f_diff_decrease}. \qed

\subsection{Finding Second Order Stationary Points}
\label{app: subsec: main_thm}
Let $\mathcal R_1 := \{x \in \mathbb{R}^d | \|\nabla f(x)\| > \varepsilon\}$, $\mathcal R_2 := \{x \in \mathbb{R}^d | \|\nabla f(x)\| \leq \varepsilon \wedge  \lambda_\mathrm{min}(\nabla^2 f(x)) < - \sqrt{\rho \varepsilon}\}$ and $\mathcal R_3 := \mathbb{R}^d \setminus (\mathcal R_1 \cup \mathcal R_2) = \{x \in \mathbb{R}^d | \|\nabla f(x)\| \leq \varepsilon \wedge  \lambda_\mathrm{min}(\nabla^2 f(x)) \geq - \sqrt{\rho \varepsilon}\}$. \par

We define 
\begin{align*}
    \iota_{m+1} = \begin{cases}
        \iota_m + 1 & (\widetilde x_{\iota_m} \in \mathcal R_1 \cup \mathcal R_3) \\
        \iota_m + \mathcal J_{\iota_m} & (\widetilde x_{\iota_m} \in \mathcal R_2)
    \end{cases} 
\end{align*}
with $\iota_1 := 0$. Note that $\mathcal J_{\iota_m} \leq c_{\mathcal J}/(\eta\sqrt{\rho \varepsilon})$. Let $M := \mathrm{min} \{m \in \mathbb{N} | \mathbb{E}[\iota_m] \geq KTS/8\}$. Observe that $\iota_M \leq M \times c_\mathcal J/(\eta\sqrt{\rho\varepsilon}) \leq (KTS/8) \times  c_\mathcal J/(\eta\sqrt{\rho\varepsilon})$ because $\iota_{KTS/8} \geq KTS/8$ always holds. We define $\check S$ as the minimum number that satisfies $\check S \geq (S / 8) \times  c_\mathcal J/(\eta\sqrt{\rho\varepsilon})) \vee S$ with $S = \Theta(1 + (f(\widetilde x_0) - f(x_*))/(\eta KT \varepsilon^2))$, where in the definition of $\eta$ we set $S \leftarrow \check S$. Then $\iota_M \leq KT\check S$ always holds. We will use Propositions \ref{prop: first_order_optimality} and \ref{prop: f_diff_decrease} with $S \leftarrow \check S$. $s(\iota_m)$ denotes the maximum natural number $s'$ satisfying $\iota_m \geq I(0, 0, s')$ and $t(\iota_m)$ denotes the maximum natural number $t'$ satisfying $\iota_m \geq I(0, t', s(\iota_m))$. We will show that $\widetilde x_i \in \mathcal R_3$ for some $i \in [KT S]\cup\{0\}$ with probability at least $1/2$. Let $E_i$ is the event that $\widetilde x_{i'} \notin \mathcal R_3$ for all $i' \leq i$ for $i \in [KTS]\cup\{0\}$. Note that $E_{i+1} \subset E_i$ for every $i$. We can say that the objective of this section is to show $\mathbb{P}(E_{KT S}) \leq 1/2$. \par

\begin{proposition}
\label{prop: second_order_stationary_probability}
Suppose that Assumptions \ref{assump: heterogeneous}, \ref{assump: local_loss_grad_lipschitzness}, \ref{assump: optimal_sol}, \ref{assump: local_loss_hessian_lipschitzness} and \ref{assump: bounded_loss_gradient} hold. Under $K = O(L/\zeta \wedge b \wedge Pb/T)$, if we appropriately choose $\eta = \widetilde \Theta(1/L \wedge 1/(K\zeta) \wedge \sqrt{b/K}/L \wedge \sqrt{Pb}/(\sqrt{KT}L))$ and $r = \widetilde \Theta(\varepsilon)$, then it holds that
\begin{align*}
    \frac{7\eta}{512}\mathbb{E}[\iota_{M}]\varepsilon^2  \leq f(\widetilde x_0) - f(x_*) +   \frac{\eta}{64}\sum_{m=1}^{M-1}\mathbb{P}(\widetilde x_{\iota_m} \in \mathcal R_3)\varepsilon^2.
\end{align*}
\end{proposition}

\subsection*{Proof of Proposition \ref{prop: second_order_stationary_probability}}
First, we consider the difference $\mathbb{E}[f(x_{\iota_{m+1}}) - f(x_{\iota_{m}})]$.
\subsection*{Bounding $\mathbb{E}[f(x_{\iota_{m+1}}) - f(x_{\iota_{m}})|\widetilde x_{\iota_m} \in \mathcal R_1]$}
Let $H_1$ be the event where (\ref{ineq: f_diff}) with $I(k_0, t_0, s_0)\leftarrow \iota_m$ and $I(k, t, s) \leftarrow \iota_{m+1}$ holds. Note that $\mathbb P(H_1|\widetilde x_{\iota_m} \in \mathcal R_1) \geq 1 - 3 q $. From Proposition \ref{prop: f_diff_decrease} and (\ref{ineq: almost_sure_object_bound}), we have for every $ q  \in (0, 1/6)$, 
\begin{align*}
     &\mathbb{E}[f(x_{\iota_{m+1}}) - f(x_{\iota_m}) | \widetilde x_{\iota_m} \in \mathcal R_1] \\
     =&\ \mathbb{E}[f(x_{\iota_{m+1}}) - f(x_{\iota_m}) | \widetilde x_{\iota_m} \in \mathcal R_1, H_1] \mathbb P(H_1|\widetilde x_{\iota_m} \in \mathcal R_1) \\ 
     &+ \mathbb{E}[f(x_{\iota_{m+1}}) - f(x_{\iota_m}) | \widetilde x_{\iota_m} \in \mathcal R_1, H_1^\complement] \mathbb P(H^\complement|\widetilde x_{\iota_m} \in \mathcal R_1) \\
     \leq&\  - (1-3 q )\frac{\eta }{2}\mathbb{E}\|\nabla f(x_{\iota_m})\|^2|\widetilde x_{I(k, t, s)} \in \mathcal R_1, H_1]+ \eta r^2 \\
     &+ \frac{c_\eta}{\eta}\mathbb{E}\left[\frac{(\iota_{m+1}-\iota_m) \wedge K }{K} \sum_{i=I(0, t(\iota_m), s(\iota_m)}^{\iota_m-1} \|x_{i+1} - x_i\|^2 \right. \\
     &+ \left. \frac{(\iota_{m+1}-\iota_m) \wedge KT}{KT}\sum_{i=I(0, 0, s(\iota_m))}^{I(0, t(\iota_m), s(\iota_m))-1}\|x_{i+1} - x_{i}\|^2|\widetilde x_{\iota_m} \in \mathcal R_1, H_1\right]\mathbb P(H_1|\widetilde x_{\iota_m} \in \mathcal R_1) \\
     &+ 3 q  \times 36 \eta K^2T^2S(G^2 + r^2) \\
     \leq&\ - \frac{\eta}{8}\varepsilon^2 + 2\eta r^2 \\
     &+ \frac{c_\eta}{\eta}\mathbb{E}\left[\frac{(\iota_{m+1}-\iota_m) \wedge K }{K} \sum_{i=I(0, t(\iota_m), s(\iota_m)}^{\iota_m-1} \|x_{i+1} - x_i\|^2 \right. \\
     &+ \left. \frac{(\iota_{m+1}-\iota_m) \wedge KT}{KT}\sum_{i=I(0, 0, s(\iota_m))}^{I(0, t(\iota_m), s(\iota_m))-1}\|x_{i+1} - x_{i}\|^2|\widetilde x_{\iota_m} \in \mathcal R_1, H_1\right]\mathbb P(H_1|\widetilde x_{\iota_m} \in \mathcal R_1) \\
     &+ 3 q \times (36 \eta K^2T^2S(G^2 + r^2) \\
     \leq&\ - \frac{\eta}{8}\varepsilon^2 + 2\eta r^2 \\
     &+ \frac{c_\eta}{\eta}\mathbb{E}\left[\frac{(\iota_{m+1}-\iota_m) \wedge K }{K} \sum_{i=I(0, t(\iota_m), s(\iota_m)}^{\iota_m-1} \|x_{i+1} - x_i\|^2 \right. \\
     &+ \left. \frac{(\iota_{m+1}-\iota_m) \wedge KT}{KT}\sum_{i=I(0, 0, s(\iota_m))}^{I(0, t(\iota_m), s(\iota_m))-1}\|x_{i+1} - x_{i}\|^2|\widetilde x_{\iota_m} \in \mathcal R_1\right] \\
     &+ 3 q  \times 36 \eta K^2T^2S(G^2 + r^2) .
\end{align*}
For the second inequality, we used $1/(1 - 3 q ) \leq 1/2 $ and $-\|\nabla f(x_{I(k, t, s)})\|^2 \leq -(1/2)\|\nabla f(\widetilde x_{I(k, t, s)})\|^2 + \|\nabla f(x_{I(k, t, s)}) - f(\widetilde x_{I(k, t, s)})\|^2 \leq -(1/2)\|\nabla f(\widetilde x_{I(k, t, s)})\|^2 + \eta^2L^2r^2 \leq -(1/2)\|\nabla f(\widetilde x_{I(k, t, s)})\|^2 + r^2$ since $\eta \leq 1/L$.  \par
Thus, setting $ q  := (\eta \varepsilon^2/16)/(96 K^2T^2S(G^2+\eta r^2))$ and  {\color{blue}$c_r \leq 1/\sqrt{96}$}, we get
\begin{align}
     &\mathbb{E}[f(x_{\iota_{m+1}}) - f(x_{\iota_m}) | \widetilde x_{\iota_m} \in \mathcal R_1]  \notag \\
     \leq&\  - \frac{\eta}{16}\varepsilon^2 + 3\eta r^2 \notag\\
     &+ \frac{c_\eta}{\eta}\mathbb{E}\left[\frac{(\iota_{m+1}-\iota_m) \wedge K }{K} \sum_{i=I(0, t(\iota_m), s(\iota_m)}^{\iota_m-1} \|x_{i+1} - x_i\|^2 \right. \notag \\
     &+ \left. \frac{(\iota_{m+1}-\iota_m) \wedge KT}{KT}\sum_{i=I(0, 0, s(\iota_m))}^{I(0, t(\iota_m), s(\iota_m))-1}\|x_{i+1} - x_{i}\|^2|\widetilde x_{\iota_m} \in \mathcal R_1\right] \notag \\
     \leq&\ - \frac{\eta}{32} \mathbb{E}[\iota_{m+1} - \iota_m| \widetilde x_{i_m} \in \mathcal{R}_1]\varepsilon^2 \notag \\
     &+  \frac{c_\eta}{\eta}\mathbb{E}\left[\frac{(\iota_{m+1}-\iota_m) \wedge K }{K} \sum_{i=I(0, t(\iota_m), s(\iota_m)}^{\iota_m-1} \|x_{i+1} - x_i\|^2 \right. \notag \\
     &+ \left. \frac{(\iota_{m+1}-\iota_m) \wedge KT}{KT}\sum_{i=I(0, 0, s(\iota_m))}^{I(0, t(\iota_m), s(\iota_m))-1}\|x_{i+1} - x_{i}\|^2|\widetilde x_{\iota_m} \in \mathcal R_1\right] \label{ineq: f_diff_x_in_r1}.
\end{align}
Here, we used $\mathbb{E}[\iota_{m+1} - \iota_m|\widetilde x_{i_m} \in \mathcal R_1] = 1$.

\subsection*{Bounding $\mathbb{E}[f(x_{\iota_{m+1}}) - f(x_{\iota_{m}})|\widetilde x_{\iota_m} \in \mathcal R_2]$}
$H_2$ denotes the event where (\ref{ineq: f_diff_escape_saddle}) with $I(k_0, t_0, s_0) \leftarrow \iota_m$ holds. Note that $\mathbb{P}(H_2 | \widetilde x_{\iota_m} \in \mathcal R_2) \geq 1/2 - 7 q /2$ by Proposition \ref{prop: f_diff_decrease}. Let $ q  \in (0, 1/14)$ and {\color{blue}with $c_\mathcal F \geq 16$}. We will use Proposition \ref{prop: f_diff_decrease}, (\ref{ineq: f_diff_always_bound}) and (\ref{ineq: almost_sure_object_bound}).
\begin{align*}
    &\mathbb{E}[f(x_{\iota_{m+1}}) - f(x_{\iota_m})| \widetilde x_{\iota_m} \in \mathcal R_2] \\
    =&\ \mathbb{E}[f(x_{\iota_{m+1}}) - f(x_{\iota_m})| \widetilde x_{\iota_m} \in \mathcal R_2, H_2]\mathbb{P}(H_2 |\widetilde x_{\iota_m} \in \mathcal R_2) \\
    &+ \mathbb{E}[f(x_{\iota_{m+1}}) - f(x_{\iota_m})| \widetilde x_{\iota_m} \in \mathcal R_2, H_1, H_2^\complement]P(H_1, H_2^\complement|\widetilde x_{\iota_m} \in \mathcal R_2) \\
    &+ \mathbb{E}[f(x_{\iota_{m+1}}) - f(x_{\iota_m})| \widetilde x_{\iota_m} \in \mathcal R_2, H_1^\complement, H_2^\complement]P(H_1^\complement, H_2^\complement|\widetilde x_{\iota_m} \in \mathcal R_2). 
\end{align*}
The first term can be bouded as 
\begin{align*}
    & \mathbb{E}[f(x_{\iota_{m+1}}) - f(x_{\iota_m})| \widetilde x_{\iota_m} \in \mathcal R_2, H_2]\mathbb{P}(H_2 |\widetilde x_{\iota_m} \in \mathcal R_2) \\
    \leq&\ \Biggl\{ -\mathbb{E}[
    \mathcal F_{\iota_m}|\widetilde x_{\iota_m} \in \mathcal R_2, H_2]  \\
    &+  \frac{2c_\eta}{\eta}\mathbb{E}\left[\frac{(\iota_{m+1}-\iota_m) \wedge K }{K} \sum_{i=I(0, t(\iota_m), s(\iota_m)}^{\iota_m-1} \|x_{i+1} - x_i\|^2  \right. \\
     &+ \left. \frac{(\iota_{m+1}-\iota_m) \wedge KT}{KT}\sum_{i=I(0, 0, s(\iota_m))}^{I(0, t(\iota_m), s(\iota_m))-1}\|x_{i+1} - x_{i}\|^2|\widetilde x_{\iota_m} \in \mathcal R_2, H_2\right]\Biggr\}\mathbb{P}(H_2 | x_{\iota_m} \in \mathcal R_2).    
\end{align*}
Here, the inequality holds from Proposition \ref{prop: f_diff_decrease}. 
The second term can be bounded as
\begin{align*}
    &  \mathbb{E}[f(x_{\iota_{m+1}}) - f(x_{\iota_m})| \widetilde x_{\iota_m} \in \mathcal R_2, H_1, H_2^\complement]P(H_1, H_2^\complement|\widetilde x_{\iota_m} \in \mathcal R_2) \\
    \leq&\ \Biggl\{\frac{2}{c_\mathcal F} \mathbb{E}[\mathcal F_{\iota_m}| \widetilde x_{\iota_m} \in \mathcal R_2, H_1, H_2^\complement] \\
    & + \frac{c_\eta}{\eta}\mathbb{E}\left[\frac{(\iota_{m+1}-\iota_m) \wedge K }{K} \sum_{i=I(0, t(\iota_m), s(\iota_m)}^{\iota_m-1} \|x_{i+1} - x_i\|^2  \right. \\
     &+ \left. \frac{(\iota_{m+1}-\iota_m) \wedge KT}{KT}\sum_{i=I(0, 0, s(\iota_m))}^{I(0, t(\iota_m), s(\iota_m))-1}\|x_{i+1} - x_{i}\|^2|\widetilde x_{\iota_m} \in \mathcal R_2, H_1,  H_2^\complement\right]\Biggr\}\mathbb{P}(H_1, H_2^\complement | x_{\iota_m} \in \mathcal R_2).
\end{align*}
Here, we used (\ref{ineq: f_diff_always_bound}).

Finally, the last term can be bounded as
\begin{align*}
    &\mathbb{E}[f(x_{\iota_{m+1}}) - f(x_{\iota_m})| \widetilde x_{\iota_m} \in \mathcal R_2, H_1^\complement, H_2^\complement]P(H_1^\complement, H_2^\complement|\widetilde x_{\iota_m} \in \mathcal R_2) \\
    \leq&\ 36 \eta K^2T^2S(G^2 + r^2)\mathbb{P}(H_1^\complement, H_2^\complement | \widetilde x_{\iota_m} \in \mathcal R_2).
\end{align*}
From these bounds, we have
\begin{align*}
    &\mathbb{E}[f(x_{\iota_{m+1}}) - f(x_{\iota_m})| \widetilde x_{\iota_m} \in \mathcal R_2] \\
    \leq&\ - \left(\frac{1}{2} - \frac{7 q }{2}  - \frac{2}{c_\mathcal F}\right)\mathbb{E}[
    \mathcal F_{\iota_m}|\widetilde x_{\iota_m} \in \mathcal R_2] \\
    &+\frac{3c_\eta}{\eta}\mathbb{E}\left[\frac{(\iota_{m+1}-\iota_m) \wedge K }{K} \sum_{i=I(0, t(\iota_m), s(\iota_m)}^{\iota_m-1} \|x_{i+1} - x_i\|^2  \right. \\
     &+ \left. \frac{(\iota_{m+1}-\iota_m) \wedge KT}{KT}\sum_{i=I(0, 0, s(\iota_m))}^{I(0, t(\iota_m), s(\iota_m))-1}\|x_{i+1} - x_{i}\|^2|\widetilde x_{\iota_m} \in \mathcal R_2\right] \\
    &+ 3 q  \times 36 \eta K^2T^2S(G^2 + r^2) \\
     \leq&\ - \frac{c_\mathcal F c_r^2\eta}{8}\mathbb{E}[
    \mathcal J_{\iota_m}|\widetilde x_{\iota_m} \in \mathcal R_2]\varepsilon^2 \\
    &+ \frac{3c_\eta}{\eta}\mathbb{E}\left[\frac{(\iota_{m+1}-\iota_m) \wedge K }{K} \sum_{i=I(0, t(\iota_m), s(\iota_m)}^{\iota_m-1} \|x_{i+1} - x_i\|^2  \right. \\
     &+ \left. \frac{(\iota_{m+1}-\iota_m) \wedge KT}{KT}\sum_{i=I(0, 0, s(\iota_m))}^{I(0, t(\iota_m), s(\iota_m))-1}\|x_{i+1} - x_{i}\|^2|\widetilde x_{\iota_m} \in \mathcal R_2\right] \\
    &+ 3 q  \times 36 \eta K^2T^2S(G^2 + r^2)   
\end{align*}
Here, for the first inequality, we used the facts that $\mathcal F_{\iota_m}$ only depends on the start point $\widetilde x_{\iota_m} \in \mathcal R_2$ and does not depend on $H_2$, which only captures the randomness after iteration index $\iota_m$, and $\mathbb{P}(H_2 | \widetilde x_{\iota_m} \in \mathcal R_2, E_{\iota_m}) \geq 1/2 - 7 q /2$. 
For the last inequality, we used  $\mathcal F_{I(k, t, s)} = c_\mathcal F\eta \mathcal J_{I(k, t, s)}r^2$ with $c_\mathcal F \geq 16$ and $r = c_r \varepsilon^2$. \par
Thus, setting $ q  := (c_\mathcal F c_r^2 \eta \varepsilon^2/16)/(96 K^2T^2S(G^2+\eta r^2))$, we get
\begin{align}
    &\mathbb{E}[f(x_{\iota_{m+1}}) - f(x_{\iota_m})| \widetilde x_{\iota_m} \in \mathcal R_2] \notag \\
    \leq&\ -\frac{c_\mathcal F c_r^2 \eta}{8}\mathbb{E}[\mathcal J_{\iota_m}| \widetilde x_{\iota_m} \in \mathcal R_2]\varepsilon^2 \notag \\
    &+ \frac{3c_\eta}{\eta}\mathbb{E}\left[\frac{(\iota_{m+1}-\iota_m) \wedge K }{K} \sum_{i=I(0, t(\iota_m), s(\iota_m)}^{\iota_m-1} \|x_{i+1} - x_i\|^2  \right. \notag\\
     &+ \left. \frac{(\iota_{m+1}-\iota_m) \wedge KT}{KT}\sum_{i=I(0, 0, s(\iota_m))}^{I(0, t(\iota_m), s(\iota_m))-1}\|x_{i+1} - x_{i}\|^2|\widetilde x_{\iota_m} \in \mathcal R_2\right] \notag \\
     &+ \frac{c_\mathcal Fc_r^2 \eta \varepsilon^2}{16} \notag \\
    =&\ -\frac{c_\mathcal F c_r^2 \eta}{16}\mathbb{E}[\iota_{m+1} - \iota_m| \widetilde x_{\iota_m} \in \mathcal R_2]\varepsilon^2 \notag \\
    &+ \frac{3c_\eta}{\eta}\mathbb{E}\left[\frac{(\iota_{m+1}-\iota_m) \wedge K }{K} \sum_{i=I(0, t(\iota_m), s(\iota_m)}^{\iota_m-1} \|x_{i+1} - x_i\|^2  \right. \notag \\
     &+ \left. \frac{(\iota_{m+1}-\iota_m) \wedge KT}{KT}\sum_{i=I(0, 0, s(\iota_m))}^{I(0, t(\iota_m), s(\iota_m))-1}\|x_{i+1} - x_{i}\|^2|\widetilde x_{\iota_m} \in \mathcal R_2\right] \label{ineq: f_diff_x_in_r2}
\end{align}

\subsection*{Bounding $\mathbb{E}[f(x_{\iota_{m+1}}) - f(x_{\iota_{m}})|\widetilde x_{\iota_m} \in \mathcal R_3]$}
Similar to the arguments for bounding $\mathbb{E}[f(x_{\iota_{m+1}}) - f(x_{\iota_{m}})|\widetilde x_{\iota_m} \in \mathcal R_1]$, we have

\begin{align}
     &\mathbb{E}[f(x_{\iota_{m+1}}) - f(x_{\iota_m}) | \widetilde x_{\iota_m} \in \mathcal R_3]  \notag \\
     \leq&\  3\eta r^2 + \frac{c_\eta}{\eta}\mathbb{E}\left[\frac{(\iota_{m+1}-\iota_m) \wedge K }{K} \sum_{i=I(0, t(\iota_m), s(\iota_m)}^{\iota_m-1} \|x_{i+1} - x_i\|^2  \right. \notag\\
     &+ \left. \frac{(\iota_{m+1}-\iota_m) \wedge KT}{KT}\sum_{i=I(0, 0, s(\iota_m))}^{I(0, t(\iota_m), s(\iota_m))-1}\|x_{i+1} - x_{i}\|^2|\widetilde x_{\iota_m} \in \mathcal R_3\right] \notag \\
     =&\  -\left(\frac{\eta}{32} \wedge \frac{c_\mathcal F c_r^2\eta}{16}\right)\mathbb{E}[\iota_{m+1} - \iota_m|\widetilde x_{\iota_m} \in \mathcal R_3] + \frac{\eta}{32} \wedge \frac{c_\mathcal F c_r^2\eta}{16}+ 3\eta r^2 \notag \\
     &+ \frac{c_\eta}{\eta}\mathbb{E}\left[\frac{(\iota_{m+1}-\iota_m) \wedge K }{K} \sum_{i=I(0, t(\iota_m), s(\iota_m)}^{\iota_m-1} \|x_{i+1} - x_i\|^2  \right. \notag \\
     &+  \left. \frac{(\iota_{m+1}-\iota_m) \wedge KT}{KT}\sum_{i=I(0, 0, s(\iota_m))}^{I(0, t(\iota_m), s(\iota_m))-1}\|x_{i+1} - x_{i}\|^2|\widetilde x_{\iota_m} \in \mathcal R_3\right]     
     \label{ineq: f_diff_x_in_r3}.
\end{align}
Here, we used the fact that $\mathbb{E}[\iota_{m+1} - \iota_m | \widetilde x_{i_m}\in \mathcal R_3]\mathbb{P}(\widetilde x_{\iota_m}\in\mathcal R_3) = \mathbb{P}(\widetilde x_{\iota_m}\in\mathcal R_3)$. 

Hence, combining (\ref{ineq: f_diff_x_in_r1}), (\ref{ineq: f_diff_x_in_r2}) and (\ref{ineq: f_diff_x_in_r3}) yields 
\begin{align*}
    &\mathbb{E}[f(x_{\iota_{m+1}}) - f(x_{\iota_m})] \\
    \leq&\ -\left(\frac{\eta}{32}\wedge \frac{c_\mathcal Fc_r^2\eta}{16}\right)\mathbb{E}[\iota_{m+1} - \iota_m]\varepsilon^2 +  \left(\frac{\eta}{32}\wedge \frac{c_\mathcal Fc_r^2\eta}{16} + 3\eta c_r^2\right)\mathbb{P}(\widetilde x_{\iota_m} \in \mathcal R_3)\varepsilon^2 \\
    &+   \frac{3c_\eta}{\eta}\mathbb{E}\left[\frac{(\iota_{m+1}-\iota_m) \wedge K }{K} \sum_{i=I(0, t(\iota_m), s(\iota_m)}^{\iota_m-1} \|x_{i+1} - x_i\|^2  \right. \\
     &+ \left. \frac{(\iota_{m+1}-\iota_m) \wedge KT}{KT}\sum_{i=I(0, 0, s(\iota_m))}^{I(0, t(\iota_m), s(\iota_m))-1}\|x_{i+1} - x_{i}\|^2\right] \\
    \leq&\ -\frac{c_\mathcal Fc_r^2\eta}{16}\mathbb{E}[\iota_{m+1} - \iota_m]\varepsilon^2 + \left(\frac{c_\mathcal Fc_r^2\eta}{16} + 3\eta c_r^2\right)\mathbb{P}(\widetilde x_{\iota_m} \in \mathcal R_3)\varepsilon^2 \\
    &+   \frac{3c_\eta}{\eta}\mathbb{E}\left[\frac{(\iota_{m+1}-\iota_m) \wedge K }{K} \sum_{i=I(0, t(\iota_m), s(\iota_m)}^{\iota_m-1} \|x_{i+1} - x_i\|^2  \right. \\
     &+ \left. \frac{(\iota_{m+1}-\iota_m) \wedge KT}{KT}\sum_{i=I(0, 0, s(\iota_m))}^{I(0, t(\iota_m), s(\iota_m))-1}\|x_{i+1} - x_{i}\|^2\right]  
\end{align*}
{\color{blue}under $c_r \leq 1/\sqrt{2c_\mathcal F}$}. 
Summing this inequality from $m=1$ to $M-1$ results in 
\begin{align}
    &\mathbb{E}[f(x_{\iota_{M}}) - f(x_0)] \notag \\
    \leq&\ -\frac{c_\mathcal Fc_r^2\eta}{16}\mathbb{E}[\iota_{M}]\varepsilon^2 + \left(\frac{c_\mathcal Fc_r^2\eta}{16} + 3\eta c_r^2\right)\sum_{m=1}^{M-1}\mathbb{P}(\widetilde x_{\iota_m} \in \mathcal R_3)\varepsilon^2 \notag \\
    &+  \frac{3c_\eta}{\eta} \mathbb{E}\left[\sum_{m=1}^{M-1}\frac{(\iota_{m+1}-\iota_m) \wedge K }{K} \sum_{i=I(0, t(\iota_m), s(\iota_m)}^{\iota_m-1} \|x_{i+1} - x_i\|^2  \right. \notag \\
     &+ \left. \frac{(\iota_{m+1}-\iota_m) \wedge KT}{KT}\sum_{i=I(0, 0, s(\iota_m))}^{I(0, t(\iota_m), s(\iota_m))-1}\|x_{i+1} - x_{i}\|^2\right]. \label{ineq: f_diff_final_bound1}
\end{align}
Here, we used the definition $\iota_1 = 0$. \par
By the way, from (\ref{ineq: f_diff}) and (\ref{ineq: almost_sure_object_bound}), we can also derive a different bound for $\mathbb{E}[f(x_{\iota_{m+1}}) - f(x_{\iota_m})]$. For every $ q  \in (0, 1/6)$, we have
\begin{align*}
     &\mathbb{E}[f(x_{\iota_{m+1}}) - f(x_{\iota_m})] \\
     =&\ \mathbb{E}[f(x_{\iota_{m+1}}) - f(x_{\iota_m}) |H_1] \mathbb P(H_1) \\ 
     &+ \mathbb{E}[f(x_{\iota_{m+1}}) - f(x_{\iota_m}) |H_1^\complement] \mathbb P(H^\complement) \\
     \leq&\  - (1-3 q )\frac{1}{8\eta}\mathbb{E}\left[\sum_{i=\iota_m}^{\iota_{m+1}-1}\|x_{i+1} - x_{i}\|^2| H_1\right]\mathbb P(H_1)+ 2\eta r^2 \mathbb{E}[\iota_{m+1}-\iota_m| H_1]\mathbb{P}(H_1) \\
     &+ \frac{c_\eta}{\eta}\mathbb{E}\left[\frac{(\iota_{m+1}-\iota_m) \wedge K }{K} \sum_{i=I(0, t(\iota_m), s(\iota_m)}^{\iota_m-1} \|x_{i+1} - x_i\|^2  \right. \\
     &+ \left. \frac{(\iota_{m+1}-\iota_m) \wedge KT}{KT}\sum_{i=I(0, 0, s(\iota_m))}^{I(0, t(\iota_m), s(\iota_m))-1}\|x_{i+1} - x_{i}\|^2|H_1\right]\mathbb P(H_1) \\
     &+ 3 q \times 36 \eta K^2T^2S(G + r^2).
\end{align*}

Observe that
\begin{align*}
    &-\mathbb{E}\left[\sum_{i=\iota_m}^{\iota_{m+1}-1}\|x_{i+1} - x_{i}\|^2| H_1\right]\mathbb P(H_1) \\
    =&\ -\mathbb{E}\left[\sum_{i=\iota_m}^{\iota_{m+1}-1}\|x_{i+1} - x_{i}\|^2\right] + \mathbb{E}\left[\sum_{i=\iota_m}^{\iota_{m+1}-1}\|x_{i+1} - x_{i}\|^2| H_1^\complement\right]P(H_1^\complement) \\
    \leq&\ -\mathbb{E}\left[\sum_{i=\iota_m}^{\iota_{m+1}-1}\|x_{i+1} - x_{i}\|^2\right] + 3 q \times192 \eta^2(KTG^2 + r^2) \\
\end{align*}
Here, for the inequality, we used 
\begin{align*}
    \|x_{i+1} - x_{i}\|^2 \leq&\  3\eta^2 \|v_i - \nabla f(x_i)\|^2 + 3\eta^2\|\nabla f(x_I)\|^2 + 3\eta^2 r^2 \\
    \leq&\ 96\eta^2KTG^2 + 3\eta^2G^2 + 3 \eta^2 r^2 \\
    \leq&\ 192 \eta^2(KTG^2 + r^2)
\end{align*}

Hence, with $ q  := \eta r^2/\{(96 K^2T^2S(G + \eta r^2) + 72\eta (KTG^2 + r^2)(c_\mathcal J/(\eta \sqrt{\rho \varepsilon}))\}$ we get 
\begin{align*}
     &\mathbb{E}[f(x_{\iota_{m+1}}) - f(x_{\iota_m})] \\
     \leq&\  - \frac{1}{16\eta}\mathbb{E}\left[\sum_{i=\iota_m}^{\iota_{m+1}-1}\|x_{i+1} - x_{i}\|^2\right]+ 2\eta r^2\mathbb{E}[\iota_{m+1}-\iota_m] \\
     &+ \frac{c_\eta}{\eta}\mathbb{E}\left[\frac{(\iota_{m+1}-\iota_m) \wedge K }{K} \sum_{i=I(0, t(\iota_m), s(\iota_m)}^{\iota_m-1} \|x_{i+1} - x_i\|^2  \right. \\
     &+ \left. \frac{(\iota_{m+1}-\iota_m) \wedge KT}{KT}\sum_{i=I(0, 0, s(\iota_m))}^{I(0, t(\iota_m), s(\iota_m))-1}\|x_{i+1} - x_{i}\|^2\right] \\
     &+  \eta r^2 \\
     \leq&\ - \frac{1}{16\eta}\mathbb{E}\left[\sum_{i=\iota_m}^{\iota_{m+1}-1}\|x_{i+1} - x_{i}\|^2\right]+  3\eta r^2\mathbb{E}[\iota_{m+1}-\iota_m] \\
     &+ \frac{2c_\eta}{\eta}\mathbb{E}\left[\frac{(\iota_{m+1}-\iota_m) \wedge K }{K} \sum_{i=I(0, t(\iota_m), s(\iota_m)}^{\iota_m-1} \|x_{i+1} - x_i\|^2  \right. \\
     &+ \left. \frac{(\iota_{m+1}-\iota_m) \wedge KT}{KT}\sum_{i=I(0, 0, s(\iota_m))}^{I(0, t(\iota_m), s(\iota_m))-1}\|x_{i+1} - x_{i}\|^2\right].
\end{align*}

Summing this inequality from $m=1$ to $M-1$ gives 
\begin{align*}
     &\mathbb{E}[f(x_{\iota_{M}}) - f(x_{0})] \\
     \leq&\  - \frac{1}{16\eta}\sum_{m=1}^{M-1}\mathbb{E}\left[\sum_{i=\iota_m}^{\iota_{m+1}-1}\|x_{i+1} - x_{i}\|^2\right]+ 3\eta r^2\mathbb{E}[\iota_{M}] \\
     &+ \frac{c_\eta}{\eta}\mathbb{E}\left[\sum_{m=1}^{M-1}\frac{(\iota_{m+1}-\iota_m) \wedge K }{K} \sum_{i=I(0, t(\iota_m), s(\iota_m)}^{\iota_m-1} \|x_{i+1} - x_i\|^2  \right. \\
     &+ \left. \frac{(\iota_{m+1}-\iota_m) \wedge KT}{KT}\sum_{i=I(0, 0, s(\iota_m))}^{I(0, t(\iota_m), s(\iota_m))-1}\|x_{i+1} - x_{i}\|^2\right]
\end{align*}
Combining this inequality with (\ref{ineq: f_diff_final_bound1}), with we obtain

\begin{align*}
    &\mathbb{E}[f(x_{\iota_{M}}) - f(x_0)]  \\
    \leq&\ -\frac{1}{2}\left(\frac{c_\mathcal F c_r^2\eta }{16} - 3 \eta c_r^2\right)\mathbb{E}[\iota_{M}]\varepsilon^2  +  \frac{1}{2}\left(\frac{c_\mathcal F c_r^2\eta }{16} + 3\eta c_r^2\right)\sum_{m=1}^{M-1}\mathbb{P}(\widetilde x_{\iota_m} \in \mathcal R_3)\varepsilon^2  \\
    &- \frac{1}{16\eta} \mathbb{E}\left[\sum_{i=0}^{\iota_{M}-1}\|x_{i+1} - x_{i}\|^2\right]\\
    &+  \frac{2c_\eta}{\eta}\mathbb{E}\left[\sum_{m=1}^{M-1}\left(\frac{(\iota_{m+1}-\iota_m) \wedge K }{K} \sum_{i=I(0, t(\iota_m), s(\iota_m)}^{\iota_m-1} \|x_{i+1} - x_i\|^2  \right. \right.\\
    &+ \left. \left. \frac{(\iota_{m+1}-\iota_m) \wedge KT}{KT}\sum_{i=I(0, 0, s(\iota_m))}^{I(0, t(\iota_m), s(\iota_m))-1}\|x_{i+1} - x_{i}\|^2\right)\right].
\end{align*}

We want to show that
\begin{align*}
     &\sum_{i=0}^{\iota_{M}-1}\|x_{i+1} - x_{i}\|^2 \\
     \geq&\  \frac{1}{4}\sum_{m=1}^{M-1}\left(\frac{(\iota_{m+1}-\iota_m) \wedge K }{K} \sum_{i=I(0, t(\iota_m), s(\iota_m))}^{\iota_m-1} \|x_{i+1} - x_i\|^2  \right. \\
     &+ \left. \frac{(\iota_{m+1}-\iota_m) \wedge KT}{KT}\sum_{i=I(0, 0, s(\iota_m))}^{I(0, t(\iota_m), s(\iota_m))-1}\|x_{i+1} - x_{i}\|^2\right).
\end{align*}
To prove this inequality, we fix $i' \in [\iota_M-1]\cup\{0\}$ and show that the coefficient of $\|x_{i'+1} - x_{i'}\|^2$ of the left hand side is greater than or equal to the one of the right hand side. At first, the coefficient of $\|x_{i'+1} - x_{i'}\|^2$ of the left hand side is trivially $1$. Next we consider the right hand side. Let $s'$ be the natural number that satisfies $I(0, 0, s') \leq i' < I(0, 0, s'+1)$. Also, $t'$ be the natural number that satisfies $I(0, t', s') \leq i' < I(0, t'+1, s')$. We define $\bm{m}_1 := \{m \in \mathbb{N}| I(0, t', s') \leq \iota_m < I(0, t'+1, s')\}$ and $\bm{m}_2 := \{m \in \mathbb{N}| I(0, 0, s') \leq \iota_m < I(0, 0, s'+1)\}$. We can see that the coefficient of $\|x_{i'+1} - x_{i'}\|^2$ in the right hand side is 
\begin{align*}
    &\frac{1}{4}\left( \sum_{m=1}^{M-1}\frac{(\iota_{m+1}-\iota_m) \wedge K }{K} \mathds{1}_{I(0, t(\iota_m), s(\iota_m)) \leq i' \leq \iota_m-1} \right.\\
    &+ \left. \sum_{m=1}^{M-1}\frac{(\iota_{m+1}-\iota_m) \wedge KT }{KT} \mathds{1}_{I(0, 0, s(\iota_m)) \leq i' \leq I(0, t(\iota_m), s(\iota_m))-1} \right) \\
    \leq&\ \frac{1}{4}\left(\sum_{m \in \bm{m}_1}\frac{(\iota_{m+1}-\iota_m) \wedge K }{K} + \sum_{m \in \bm{m}_2}\frac{(\iota_{m+1}-\iota_m) \wedge KT }{KT} \right) \\
    =&\ \frac{1}{4}\left(1 + \sum_{m \in \bm{m}_1 \setminus\{ \mathrm{max}\{\bm{m}_1\}\}}\frac{(\iota_{m+1}-\iota_m) \wedge K }{K} + 1 + \sum_{m \in \bm{m}_2\setminus\{\mathrm{max}\{\bm{m}_2\}\}}\frac{(\iota_{m+1}-\iota_m) \wedge KT }{KT} \right) \\
    \leq&\ 1. 
\end{align*}
Here, for the first inequality we used the facts that (i) $I(0, t(\iota_m), s(\iota_m))\leq i' \leq \iota_m -1)$ implies $m \in \bm{m}_1$ and (ii) $I(0, 0, s(\iota_m) \leq i' \leq I(0, t(\iota_m), s(\iota_m))-1$ implies $m \in \bm{m}_2$. To show (i), note that $\iota_m < I(0, t', s')$ implies $\iota_m - 1< i'$ and $\iota_m \geq I(0, t'+1, s')$ implies $I(0, t(\iota_m), s(\iota_m)) \geq I(0, t'+1, s') > i'$. Similarly, to show (ii), observe that $\iota_m < I(0, 0, s')$ implies $i' > \iota_m > I(0, t(\iota_m), s(\iota_m)) - 1$ and $\iota_m \geq I(0, 0, s'+1)$ implies $i' < I(0, 0, s'+1) \leq I(0, 0, s(\iota_m))$.  For the last inequality we used $\sum_{m \in \bm{m}_1 \setminus\{ \mathrm{max}\{\bm{m}_1\}\}}(\iota_{m+1}-\iota_m) \leq K$ and $\sum_{m \in \bm{m}_2 \setminus\{ \mathrm{max}\{\bm{m}_2\}\}}(\iota_{m+1}-\iota_m) \leq KT$. \par 
We choose {\color{blue}$c_\eta \leq 1/128$}. Then, we obtain
\begin{align*}
    f(x_*) - f(\widetilde x_0) \leq&\ \mathbb{E}[f(x_{\iota_M}) - f(x_0)] + \eta r^2  \\
    \leq&\ -\left(\frac{c_\mathcal Fc_r^2\eta}{32} - 3  c_r^2\eta\right)\mathbb{E}[\iota_{M}]\varepsilon^2  +  (\frac{c_\mathcal Fc_r^2\eta}{32} + 3 c_r^2\eta)\sum_{m=1}^{M-1}\mathbb{P}(\widetilde x_{\iota_m} \in \mathcal R_3)\varepsilon^2.
\end{align*}
Here, for the first inequality, we used $\mathbb{E}[f(x_{\iota_M})] \geq f(x_*)$ and $\mathbb{E}[f(x_0)] \leq f(\widetilde x_0) + \langle \nabla f(\widetilde x_0), \mathbb{E}[x_0 - \widetilde x_0]\rangle + (L/2)\|x_0 - \widetilde x_0\|^2 =  f(\widetilde x_0) +  \eta^2 L r^2/2 \leq f(\widetilde x_0) + \eta r^2$ by the smoothness of $f$. For the second inequality, we used the above bounds with the definition of $c_\eta$ for $\mathbb{E}[f(x_{\iota_M}) - f(x_0)]$.
%and $\mathbb{E}[\iota_M] \geq M$. For the last inequality, we used the definition of $c_r$. 
This finishes the proof. \qed

\begin{comment}
\begin{theorem}[Final Theorem]\label{thm: main}
Suppose that Assumptions \ref{assump: heterogeneous}, \ref{assump: local_loss_grad_lipschitzness}, \ref{assump: optimal_sol}, \ref{assump: local_loss_hessian_lipschitzness} and \ref{assump: bounded_loss_gradient} hold. If we appropriately choose $\eta = \widetilde \Theta(1/L \wedge 1/(K\zeta) \wedge \sqrt{b/K}/L \wedge \sqrt{Pb}/(\sqrt{KT}L))$, $r = \Theta(\varepsilon)$ and $S = \Theta(1 + (f(\widetilde x_0) - f(x_*))/(\eta KT \varepsilon^2))$, with probability at least $1/2$, there exists $i \in [KTS]\cup\{0\}$ such that $\|\nabla f(\widetilde x_i)\| \leq \varepsilon$ and $\nabla^2 f(\widetilde x_i) \succeq - \sqrt{\rho\varepsilon}$. 
\end{theorem}
\end{comment}

\subsection*{Proof of Theorem \ref{thm: main}}

Now, we choose {\color{blue}$S \geq 48(f(x_0) - f(x_*))/(c_r^2\eta   KT \varepsilon^2) = \widetilde \Theta((f(x_0) - f(x_*))/(\eta KT \varepsilon^2)$}. Note that $\mathbb{E}[\iota_M] \geq KTS/8 \geq 6(f(x_0)-f(x_*))/(c_r^2 \eta\varepsilon^2)$. \par
Suppose that $\mathbb{P}(\widetilde x_{\iota_m} \in \mathcal R_3) \leq 3/4$ for every $m \in [M-1]$. Then, since $c_\mathcal F c_r^2 \eta/32 - 3c_r^2\eta - (3/4) \times (c_\mathcal F c_r^2 \eta/32 + 3c_r^2\eta) \geq 1/4 (c_\mathcal F/32 - 21)c_r^2\eta \geq c_r^2\eta/4$ under {\color{blue}$c_\mathcal F \geq 32 \times 22$}, we have
\begin{align*}
    f(x_*) - f(x_0) \leq -\frac{c_r^2 \eta}{4}\mathbb{E}[\iota_{M}]\varepsilon^2 
\end{align*}
and thus 
\begin{align*}
    \mathbb{E}[\iota_{M}] \leq \frac{4(f(x_0) - f(x_*))}{c_r^2\eta \varepsilon^2}
\end{align*}
from Proposition \ref{prop: second_order_stationary_probability}.
This contradicts the previous lower bound of $\mathbb{E}[\iota_M]$. Therefore, we conclude that there exists $m \in [M-1]$ such that $\mathbb{P}(\widetilde x_{i_m} \in \mathcal R_3) > 3/4$. Remember that $E_i$ is the event that $\widetilde x_{i'} \notin \mathcal R_3$ for all $i' \leq i$. This implies $\mathbb{P}(E_{\iota_{M-1}}^\complement) > 3/4 $, and thus $\mathbb{P}(E_{\iota_{M-1}}) \leq 1/4$. 

Finally, we bound $\mathbb{P}(E_{KTS})$. From the definition of $M$, we have $\mathbb{E}[\iota_{M-1}] < KTS/8$. Thus, from Markov's inequality, it holds that $\mathbb{P}(\iota_{M-1} \geq KTS) \leq 1/8$. 

This yields 
\begin{align*}
    \mathbb{P}(E_{KTS}) =&\  \mathbb{P}(E_{KTS}|\iota_{M-1} \geq KTS)\mathbb{P}( \iota_{M-1} \geq KTS) + \mathbb{P}(E_{KTS}|\iota_{M-1} < KTS)\mathbb{P}( \iota_{M-1} < KTS) \\
    \leq&\ 1 \times \frac{1}{8} + \mathbb{P}(E_{\iota_{M-1}}) \\
    \leq&\  1/2. 
\end{align*}

This finishes the proof. \qed
\end{document}